\documentclass{article} % For LaTeX2e
\usepackage{etoolbox}
\newtoggle{colt}
\togglefalse{colt}
\newcommand{\arxiv}[1]{\iftoggle{colt}{}{#1}}
\newcommand{\colt}[1]{\iftoggle{colt}{#1}{}}
\newcommand{\loose}{\colt{\looseness=-1}}

\colt{\usepackage{times}}
\usepackage[utf8]{inputenc} 
\usepackage[T1]{fontenc}   
\usepackage{xcolor}
\usepackage[colorlinks,citecolor=blue!70!black,linkcolor=blue!70!black,urlcolor=blue!70!black,breaklinks=true]{hyperref}    
\usepackage[suppress]{color-edits}
\addauthor{df}{ForestGreen}
\addauthor{YE}{blue}
\addauthor{ak}{purple}
\addauthor{dm}{red}
\usepackage{enumitem}
\usepackage{comment}

\newcommand{\epsnot}{\eps_0}

\arxiv{
\usepackage[letterpaper, left=1in, right=1in, top=1in, bottom=1in]{geometry}
\usepackage{natbib}
\bibliographystyle{plainnat}
\bibpunct{(}{)}{;}{a}{,}{,}

\usepackage{breakcites}
}

\usepackage{algorithm,algorithmicx}
\usepackage[noend]{algpseudocode}
\usepackage{varwidth}

\makeatletter
\let\OldStatex\Statex
\renewcommand{\Statex}[1][3]{%
  \setlength\@tempdima{\algorithmicindent}%
  \OldStatex\hskip\dimexpr#1\@tempdima\relax}
\makeatother

\usepackage{mkolar_definitions}
\usepackage{bbm}

\newtheorem{claim}{Claim}

\usepackage{thmtools}
\usepackage{thm-restate}

\usepackage{hyperref}       % hyperlinks
\usepackage{url}            % simple URL typesetting
\usepackage{booktabs}       % professional-quality tables
\usepackage{nicefrac}       % compact symbols for 1/2, etc.
\usepackage{microtype}      % microtypography
\usepackage{setspace}
\usepackage{amssymb}
\usepackage{amsmath}
\usepackage{mathtools}
\usepackage{graphicx}
\usepackage{scrextend}

\usepackage[normalem]{ulem}	% For strikethrough

\usepackage[scaled=.9]{helvet}
\usepackage{xspace}

\usepackage{mathrsfs}  

\usepackage{inconsolata} % Make tt not ugly.

\usepackage{comment}

\newcommand{\algcommentul}[1]{\textbf{\underline{#1}}}
\newcommand{\algcomment}[1]{\textcolor{blue!70!black}{\footnotesize{//\hspace{2pt}#1}}}

\newcommand{\algcommentbig}[1]{\textcolor{blue!70!black}{\footnotesize{/*
          #1~*/}}}

\newcommand{\algspace}{\hspace{\algorithmicindent}}

\DeclarePairedDelimiter{\abs}{\lvert}{\rvert} %
\DeclarePairedDelimiter{\brk}{[}{]}
\DeclarePairedDelimiter{\crl}{\{}{\}}
\DeclarePairedDelimiter{\prn}{(}{)}
\DeclarePairedDelimiter{\nrm}{\|}{\|}

\DeclareMathOperator{\En}{\mathbb{E}}

\newcommand{\wt}[1]{\widetilde{#1}}
\newcommand{\wh}[1]{\widehat{#1}}

\def\ddefloop#1{\ifx\ddefloop#1\else\ddef{#1}\expandafter\ddefloop\fi}
\def\ddef#1{\expandafter\def\csname bb#1\endcsname{\ensuremath{\mathbb{#1}}}}
\ddefloop ABCDEFGHIJKLMNOPQRSTUVWXYZ\ddefloop
\def\ddefloop#1{\ifx\ddefloop#1\else\ddef{#1}\expandafter\ddefloop\fi}
\def\ddef#1{\expandafter\def\csname b#1\endcsname{\ensuremath{\mathbf{#1}}}}
\ddefloop ABCDEFGHIJKLMNOPQRSTUVWXYZ\ddefloop
\def\ddef#1{\expandafter\def\csname sf#1\endcsname{\ensuremath{\mathsf{#1}}}}
\ddefloop ABCDEFGHIJKLMNOPQRSTUVWXYZ\ddefloop
\def\ddef#1{\expandafter\def\csname c#1\endcsname{\ensuremath{\mathcal{#1}}}}
\ddefloop ABCDEFGHIJKLMNOPQRSTUVWXYZ\ddefloop
\def\ddef#1{\expandafter\def\csname h#1\endcsname{\ensuremath{\widehat{#1}}}}
\ddefloop ABCDEFGHIJKLMNOPQRSTUVWXYZ\ddefloop
\def\ddef#1{\expandafter\def\csname hc#1\endcsname{\ensuremath{\widehat{\mathcal{#1}}}}}
\ddefloop ABCDEFGHIJKLMNOPQRSTUVWXYZ\ddefloop
\def\ddef#1{\expandafter\def\csname t#1\endcsname{\ensuremath{\widetilde{#1}}}}
\ddefloop ABCDEFGHIJKLMNOPQRSTUVWXYZ\ddefloop
\def\ddef#1{\expandafter\def\csname tc#1\endcsname{\ensuremath{\widetilde{\mathcal{#1}}}}}
\ddefloop ABCDEFGHIJKLMNOPQRSTUVWXYZ\ddefloop
\def\ddefloop#1{\ifx\ddefloop#1\else\ddef{#1}\expandafter\ddefloop\fi}
\def\ddef#1{\expandafter\def\csname scr#1\endcsname{\ensuremath{\mathscr{#1}}}}
\ddefloop ABCDEFGHIJKLMNOPQRSTUVWXYZ\ddefloop
\def\ddefloop#1{\ifx\ddefloop#1\else\ddef{#1}\expandafter\ddefloop\fi}
\def\ddef#1{\expandafter\def\csname s#1\endcsname{\ensuremath{\mathscr{#1}}}}
\ddefloop ABCDEFGHIJKLMNOPQRSTUVWXYZ\ddefloop

\newcommand{\eps}{\epsilon}
\newcommand{\veps}{\epsilon}

\newcommand{\ldef}{\vcentcolon=}
\newcommand{\rdef}{=\vcentcolon}

\usepackage{xpatch}
\xpatchcmd{\proof}{\itshape}{\normalfont\proofnameformat}{}{}
\newcommand{\proofnameformat}{\bfseries}

\usepackage{crossreftools}
\newcommand{\preftitle}[1]{\crtcref{#1}}

\arxiv{\usepackage{parskip}}
\allowdisplaybreaks

\makeatletter
  \renewenvironment{proof}[1][Proof]%
  {%
   \par\noindent{\bfseries\upshape {#1.}\ }%
  }%
  {\qed\newline}
  \makeatother

\makeatletter\newcommand\paragraphp{\@startsection{paragraph}{4}{\z@}{3.25ex \@plus1ex \@minus.2ex}{-1em}{\normalfont\normalsize\itshape}}\makeatother

\arxiv{\usepackage{authblk}}
\newcommand{\lhs}{left-hand side\xspace}
\newcommand{\rhs}{right-hand side\xspace}

\newcommand{\epm}{$\EndoPolicyMaximizer^\eps_{t,h}$\xspace}
\newcommand{\efs}{$\FactorDetectionAlg_{t,h}^{\eps}$\xspace}

\newcommand{\en}{\mathrm{en}}

\newcommand{\ind}[1]{^{\scriptscriptstyle(#1)}}

\newcommand{\ttoh}{$t\to{}h$\xspace}

\newcommand{\pihat}{\wh{\pi}}

\newcommand{\Enpi}[1][\pi]{\En_{#1}}
\newcommand{\Prpi}[1][\pi]{\bbP_{#1}}

\newcommand{\bigoh}{O}
\newcommand{\bigom}{\Omega}

\renewcommand{\S}{S}

\newcommand{\J}{J}
\newcommand{\Jpi}[1][\pi]{\J(#1)}

\renewcommand{\th}{\mathrm{th}}

\newcommand{\fullcleq}[1][k]{\mathscr{I}_{\leq#1}}

\newcommand{\unf}{{\mathrm{Unf}}} % DF: Just FYI you should always use
\newcommand{\unif}{{\mathrm{Unf}}}

\newcommand{\fullc}{\mathscr{I}}
\newcommand{\FactorDetectionAlg}{\textsf{EndoFactorSelection}\xspace}

\newcommand{\EndoPolicyMaximizer}{\textsf{EndoPolicyOptimization}\xspace}
\newcommand{\AlgName}{\textsf{OSSR}\xspace}
\newcommand{\RLwithExo}{\textsf{ExoRL}\xspace}
\newcommand{\mainalg}{\RLwithExo}
\newcommand{\pcalg}{\AlgName}
\newcommand{\pcalgfull}{$\AlgName_{h}^{\veps,\delta}$\xspace}
\newcommand{\pcalglong}{Optimization-Selection State
  Refinement\xspace}
\newcommand{\pcalgonestep}{\ExactOneStep}
\newcommand{\pcalgexact}{\ExactAlgName}

\newcommand{\AMFD}{\textsf{AbstractFactorSearch}\xspace}

\newcommand{\ExactAlgName}{\textsf{OSSR.Exact}\xspace}
\newcommand{\ExactOneStep}{\textsf{OSSR.OneStep}\xspace}
\newcommand{\PSDP}{\textsf{PSDP}\xspace}

\newcommand{\PSDPE}{\textsf{ExoPSDP}\xspace}
\newcommand{\kEMDP}{ExoMDP\xspace}
\newcommand{\kEMDPs}{ExoMDPs\xspace}
\newcommand{\framework}{\kEMDP}
\newcommand{\frameworks}{\kEMDPs}

\newcommand{\Condition}{\mathrm{Condition}}

\newcommand{\iscover}{\textsf{is\_cover}\xspace}
\newcommand{\sufficientcover}{\textsf{sufficient\_cover}\xspace}
\newcommand{\poly}{\mathrm{poly}}

\newcommand{\GS}{\mathcal{\Gcal}}

\newcommand{\True}{\textsf{true}}
\newcommand{\False}{\textsf{false}}
\newcommand{\fail}{\textsf{fail}}

\newcommand{\Ic}{\cI_{\star}}
\newcommand{\Inc}{\Ic^\mathrm{c}}
\newcommand{\cItil}{\wt{\cI}}
\newcommand{\cIhat}{\wh{\cI}}

\newcommand{\cIen}{\cI_{\mathrm{en}}}

\newcommand{\cJen}{\cJ_{\mathrm{en}}}
\newcommand{\cJex}{\cJ_{\mathrm{ex}}}
\newcommand{\cIhaten}{\cIhat_{\mathrm{en}}}
\newcommand{\cIhatex}{\cIhat_{\mathrm{ex}}}

\newcommand{\circt}[1][t]{\circ_{#1}}

\newcommand{\Tc}{T_{\mathrm{en}}}
\newcommand{\Tnc}{T_{\mathrm{ex}}}
\newcommand{\rc}{R_{\mathrm{en}}}
\newcommand{\rct}{R_{\mathrm{en},t}}
\newcommand{\dc}{d_{1,\mathrm{en}}}
\newcommand{\dnc}{d_{1,\mathrm{ex}}}

\renewcommand{\emptyset}{\varnothing}

\newcommand{\GenericCollection}{\scrI}
\newcommand{\AllAbstractions}{\scrI_{\leq{}k}}

\newcommand{\EmpOccupancy}{\wh{\cD}}
\newcommand{\dhat}{\wh{d}}

\newcommand{\midsem}{\;;\;}
\newcommand{\PiInduced}{\Pi}
\newcommand{\PiInd}{\PiInduced}
\newcommand{\PiIndNS}{\PiInduced_{\mathrm{NS}}}
\newcommand{\PiNS}{\PiIndNS}
\newcommand{\Pimix}{\Pi_{\mathrm{mix}}}

\newcommand{\mut}[1][t]{\mu\ind{#1}}
\newcommand{\Iendo}[1]{#1_{\mathrm{en}}}
\newcommand{\Iexo}[1]{#1_{\mathrm{ex}}}
\newcommand{\setComp}[1]{#1^{\mathrm{c}}}

\newcommand{\multiline}[1]{\parbox[t]{\dimexpr\linewidth-\algorithmicindent}{#1}}

\newcommand{\cIbar}{\cJ}

\arxiv{\title{Sample-Efficient Reinforcement Learning 
              in the Presence of Exogenous Information}}
\colt{\title[Reinforcement Learning in the Presence of Exogenous Information]{Sample-Efficient Reinforcement Learning\\~in the Presence of Exogenous Information}}
\addtocontents{toc}{\protect\setcounter{tocdepth}{0}}

 \arxiv{ \author[1]{Yonathan Efroni}
          \author[1]{Dylan J. Foster} 
          \author[1]{Dipendra Misra}
          \author[1]{Akshay Krishnamurthy}
          \author[1]{John Langford}
\affil[1]{Microsoft Research NYC}
\date{}
}

\begin{document}

\maketitle
\begin{abstract}

In real-world reinforcement learning applications
the learner's observation space is ubiquitously high-dimensional with both relevant and irrelevant information about the task at
hand. Learning from high-dimensional observations has
been the subject of extensive investigation in supervised learning and statistics (e.g., via sparsity), but
analogous issues in reinforcement learning are not well understood,
even in finite state/action (tabular) domains. We introduce a new problem setting for reinforcement learning, the Exogenous Markov Decision Process
(\framework), in which the state
space admits an (unknown) factorization into a small controllable (or,
\emph{endogenous}) component and a large irrelevant (or,
\emph{exogenous}) component; the exogenous component is independent of
the learner's actions, but evolves in an arbitrary, temporally
correlated fashion. We provide a new algorithm, \mainalg, which learns
a near-optimal policy with sample complexity polynomial in the size of
the endogenous component and nearly independent of the size
of the exogenous component, thereby offering a \emph{doubly-exponential} improvement over off-the-shelf algorithms. Our results
highlight for the first time that sample-efficient reinforcement learning is possible in the
presence of exogenous information, and provide a simple, user-friendly benchmark for investigation going
forward.

% \dipendra{Should we call the setting new given the PPE paper? We
% already introduced it. How about ExoRL instead of ExRL which is
% easier to pronounce?}
% DF: I think it's fine; we clarify later. Also ExRL is easier to
% pronounce IMO.

% Nearly-independent of the number of components, and constitutes a
% \emph{doubly exponential} improvement over naive tabular RL approaches.

% Shows for the first time...

% Beyond our technical results, ... 
% We hope that...serves as a simple testbed. (advantage of our
% setting is the simplicity...)

% Policy cover

\begin{comment}
  \begin{itemize}
  \item High-level motivation
  \item More context, possibly slightly more technical (e.g., so far
               existing approaches fail because...
   \item We introduce a new problem setting, ....
   \item More context on the problem setting
   \item We prove an algorithm which...
  \item Features of the result: Doubly-exponential improvement over
               naive tabular RL, computationally efficient,...
  \item Shows for the first-time that sample-efficient rl in the
               presence of exogenous information is possible (stochastic)
  \end{itemize}
\end{comment}

%%% Local Variables:
%%% mode: latex
%%% TeX-master: "paper"
%%% End:
    
\end{abstract}

% \begin{keywords}
%   Reinforcement learning
% \end{keywords}

% \tableofcontents

% TODOS:
% \begin{enumerate}
%     \item Change $\Pi\rightarrow \Gamma$. Change $\PiInd$ to $\Pi.$
%     \item Dynamics relevant.
%     \item Give a motivating story. Emphasize the naturalness of the problem.
%     \item Decision: $\cI_\star$ the endogenous set, $\cI_\star^{\mathrm{c}}$ the exogenous set, $\cI_{\mathrm{en}} = \cI\cap \cI_\star, \cI_{\mathrm{ex}} = \cI\cap \cI_\star^{\mathrm{c}}$.
%     \item We need $\eps/12$ close model - then $4\eps/12=\eps/3$. Changed that already but need to check there are no leftovers.
%\item Collection or set of sets?
% \end{enumerate}

\section{Introduction}
\label{sec:intro}
% \akedit{
Most applications of machine learning and statistics involve complex inputs such as images or text, which may contain spurious information for the task at hand. A traditional approach to this problem is to use feature engineering to identify relevant information, but this requires significant domain expertise, and can lead to poor performance if relevant information is missed.
% \akedit{
As an alternative, representation learning and feature selection methodologies developed over the last several decades address these issues, and enable practitioners to directly operate on complex, high-dimensional inputs with minimal domain knowledge. In the context of supervised learning and statistical estimation, these methods are particularly well-understood \citep{hastie2015statistical,wainwright2019high} and---in some cases---can be shown to provably identify relevant information for the task at hand in the presence of a vast amount of irrelevant or spurious features. As such, these approaches have emerged as the methods of choice for many practitioners.
% }

Complex, high-dimensional inputs are also ubiquitous in Reinforcement Learning (RL) applications. However, due to the interactive, multi-step nature of the RL problem, naive extensions of representation learning techniques from supervised learning do not seem adequate. Empirically, this can be seen in the brittleness of deep RL algorithms and, the large body of work on stabilizing these methods \citep{gelada2019deepmdp,zhang2020learning}. Theoretically, this can be seen by the prevalence of strong function approximation assumptions that preclude introducing spurious features \citep{wang2020statistical,weisz2021exponential}. As a result, developing representation learning methodology for RL is a central topic of investigation.
% \akedit{
% }

Recently, a line of theoretical works have developed structural conditions under which RL with complex inputs is statistically tractable~\citep{jiang2017contextual,jin2021bellman,du2021bilinear,foster2021statistical}, along with a complementary set of algorithms for addressing these problems via representation learning \citep{du2019provably,misra2020kinematic,agarwal2020flambe,misra2021provable,uehara2021representation}. While these works provide some clarity into the challenges of high-dimensionality in RL, the models considered do not allow for spurious, temporally correlated information (e.g., exogenous information that evolves over time through a complex dynamical system). On the other hand, this structure is common in applications; for example, when a human is navigating a forest trail, the flight of birds in the sky is temporally correlated, but irrelevant for the human's decision making. Motivated by the success of high-dimensional statistics in developing and understanding feature selection methods for supervised learning, we ask:
% \akedit{
% }
\begin{center}
\emph{Can we develop provably efficient algorithms for RL in the presence of a large number of dynamic, yet irrelevant features?}
\end{center}
% \akedit
\cite{efroni2021provable} initiated the study of this question in a rich-observation setting with function approximation. However, their results require deterministic dynamics, and their approach crucially uses determinism to sidestep many challenges that arise in the presence of exogenous information.
% }

%% \akcomment{Straight to "our contributions" from here. We will mention the related works like~\cite{efroni2021provable} in the related work section. I'd also be open to mentioning this one right here, saying something like "This question was initiated by~\citet{efroni2021provable}, but their results rely crucially on the assumption of deterministic relevant dynamics. In this paper, we take a step back and [something]" }

% \akedit{
\paragraph{Our contributions.}
In this paper, we take a step back from the function approximation setting considered by~\citet{efroni2021provable}, and introduce a simplified problem setting in which to study representation learning and exploration with high-dimensional, exogenous information. Our model, the \emph{Exogenous Markov Decision Process} or \kEMDP, involves a discrete $d$-dimensional state space (with each dimension taking values in $\crl*{1,\ldots,S}$) in which an \emph{unknown} subset of $k \ll d$ dimensions of the state can be controlled by the agent's actions. The remaining $d-k$ state variables are irrelevant for the agent's task, but may exhibit complex temporal structure.
%We make no other assumptions on the dynamics.

% \akedit{
Our main result is a new algorithm, \mainalg, that learns a policy which is (i) near-optimal and (ii) does not depend on the exogenous and irrelevant factors, while requiring only $\poly(S^k,\log(d))$ trajectories.  Here, the dominant $S^{k}$ term represents the size of the controllable (or, endogenous) state space, and the $\log(d)$ term represents the price incurred for feature selection (analogous to guarantees for sparse regression \citep{hastie2015statistical,wainwright2019high}). Our result represents a \emph{doubly-exponential} improvement over naive application of existing tabular RL methods to the \framework setting, which results in $\poly(S^d)$ sample complexity.  Our algorithm and analysis involve many new ideas for addressing exogenous noise, and we believe our work may serve as a building block for addressing these issues in more practical settings.

\section{Overview of Results}\label{sec: preliminaries}

In this section we introduce the \framework setting and give an overview of our algorithmic results, highlighting the key challenges they overcome. Before proceeding, we formally describe the basic RL setup we consider.

\paragraph{Markov decision processes.} We consider a finite-horizon Markov decision process (MDP) defined by the tuple $\Mcal =
\rbr{\Scal,\Acal, T, R, H, d_1}$, in which $\mathcal{S}$ is the state space, $\mathcal{A}$ is the action space, $T:\cS\times\cA\to\Delta(\cS)$ is the transition operator $R:\cS\times\cA\to\brk{0,1}$ is the reward function, $H\in\bbN$ is the horizon, and $d_1\in\Delta(\cS)$ is the initial state distribution. Given a non-stationary policy $\pi=(\pi_1,\ldots,\pi_H)$, where $\pi_h:\cS\to\cA$, an episode in the MDP $\cM$ proceeds as follows, beginning from $s_1\sim{}d_1$: For $h=1,\ldots,H$: $a_h=\pi_h(s_h)$, $r_h=R(s_h,a_h)$, and $s_{h+1}\sim{}T(\cdot\mid{}s_h,a_h)$. We let $\En_{\pi}\brk{\cdot}$ and $\bbP_{\pi}(\cdot)$ denote the expectation and probability for the trajectory $(s_1,a_1,r_1),\ldots,(s_H,a_H,r_H)$ when $\pi$ is executed, respectively, and define
$\Jpi = \Enpi\brk*{\sum_{h=1}^{H}r_h}$
as the average reward.

The objective of the learner is to learn an $\eps$-optimal policy online: Given $N$ episodes to execute a policy and observe the resulting trajectory, find a policy $\pihat$ such that $\J(\pihat) \geq{} \max_{\pi\in\PiNS}\J(\pi) - \veps$, where $\PiNS$ denotes the set of all non-stationary policies $\pi=(\pi_1,\ldots,\pi_H)$.

% Our objective is to find a near-optimal policy for a general \kEMDP\ formally defined as follows.
% \begin{definition}[$\eps$-optimal policy]
%   A policy $\widehat{\pi}$ is said to be $\eps$-optimal if $\J(\pihat) \geq{} \max_{\pi}\J(\pi) - \veps$.
%   %$\max_\pi V_1^\pi \leq V^{\widehat{\pi}}_1 +\eps$.
% \end{definition}

% the transition operator which describes the probability transitioning to state $s'$ upon taking action $a$ at state $s$ and $H\in \mathbb{N}$ is the horizon.

% Furthermore, $d_1(s)$ is the state distribution at the first timestep, $d_1(s) \ldef \PP(s_1 = s)$ where $s_1$ is a random variable which represents the first state within an episode. \dfcomment{Need to explain how $s_{h+1}$ evolves from $s_h$, and that $r_h=R(s_h,a_h)$.}

% \dfcomment{Do we need to explain what stochastic policies are here?}\YEcomment{Nope- only use deterministic policies.}

% . Lastly, we define $\J(\pi)\ldef{} \En_{\pi}\brk*{\sum_{h=1}^{H}r_h}$ as the average reward function. 

% Lastly, $V_1^\pi \ldef \EE_{s\sim d_1}\sbr{V_1^\pi\rbr{s}}$ is the value of a policy averaged over the initial state distribution. \dfcomment{Can we call this $J(\pi)$ instead?}

\subsection{The Exogenous MDP (\kEMDP) Setting}
The \framework is a Markov decision process in which the state space factorizes into an endogenous component that is (potentially) affected by the learner's actions, and an exogenous component that is independent of the learner's actions, but evolves in an arbitrary, temporally correlated fashion. Formally, given a parameter $d\in\bbN$ (the number of factors), the state space $\cS$ takes the form $\Scal = \otimes_{i=1}^d\cS_i$, so that each state $s\in\cS$ has the form $s=(s_1,\ldots,s_d)$, with $s_i\in\cS_i$; we refer to $\cS_i$ (equivalently, $i$) as the \emph{$i^{\th}$ factor}. We take $\Ic\subset\brk{d}$ to represent the \emph{endogenous factors} and $\Inc\ldef{}\brk{d}\setminus\Ic$ to represent the \emph{exogenous factors}, which are unknown to the learner. Letting $s\brk{\cI}\ldef{}(s_i)_{i\in\cI}$, we assume the dynamics and rewards factorize across the endogenous and exogenous components as follows:
\begin{align}
  \begin{aligned}
     &T\rbr{s' \mid s, a} =   \Tc(s'\sbr{\Ic} \mid s\sbr{\Ic},a)\cdot\Tnc(s'\sbr{\Inc} \mid s\sbr{\Inc}), \\
     &R(s,a) = \rc(s\sbr{\Ic},a), \\
    &d_1(s) = \dc\rbr{s\sbr{\Ic}}\cdot\dnc\rbr{s\sbr{\Inc}},
  \end{aligned}
      \label{eq: model assumption d}
\end{align}
% \dfcomment{It's weird that $d_1$ doesn't have the en/ex subscript but rewards/transitions do.}\YEcomment{since its a probability distribution it is implicitly defined isn't it?}
for all $s,s'\in \cS$ and $a\in \Acal.$
That is, the endogenous factors $\Ic$  are (potentially) affected by the agent's actions and are sufficient to model the reward, while the exogenous factors $\Inc$ evolve independently of the learner's actions and do not influence the reward.  

% Extension to the case $\eta$ is unknown is standard; by running our algorithm in an exponentially decreasing scale, $\eta=1/2,1/2^2,\ldots$, we can recover the same results while suffering from additional $\log\rbr{1/\eta}$ multiplicative factor in our bounds (see \YEcomment{TODO} for further details). \dipendra{this is slightly weird since you claim knowledge of $\eta$ not a knowledge of lower bound. The halving trick works only if you can assume a lower bound since the halving may not eventually meet $\eta$ but only give you a lower value eventually.}
% Note that in general, $\Ic$ may not be unique or identifiable.

% \dfcomment{Tabular}

In this paper, we focus on a finite-state/action (tabular) variant of the \framework setting in which $\cS_i=\brk{S}$ and $\cA=\brk{A}$, with $S\in\bbN$ representing the number of states per factor and $A\in\bbN$ representing the number of actions. We assume that $\abs{\Ic}\leq{}k$, where $k\ll{}d$ is a known upper bound on the number of endogenous factors.\footnote{Extending our results to settings in which different factors have different sizes (i.e., $\cS_i=\brk{S_i}$) is straightforward.} In the absence of the structure in \eqref{eq: model assumption d}, this is a generic tabular RL problem with $\abs{\cS}=S^{d}$, and the optimal sample complexity scales as $\poly(\S^d,A,H,\eps^{-1})$ \citep{azar2017minimax}, which has exponential dependence on the number of factors $d$. On the other hand, if $\Ic$ were known a-priori, applying off-the-shelf algorithms for tabular RL to the endogenous subset of the state space would lead to sample complexity $\poly(\S^k,A,H,\eps^{-1})$ \citep{azar2017minimax,jin2018q,zanette2019tighter,kaufmann2021adaptive}, which is independent of $d$ and offers significant improvement when $k\ll{}d$. This motivates us to ask:   \emph{With no prior knowledge, can we learn an $\eps$-optimal policy for the \framework with sample complexity polynomial in $S^{k}$ and sublinear in $d$?}
% \begin{center}
% \dipendra{you may want to formally define what this nearly independent of $d$ means. I guess you mean sublinear, or even logarithmic?}

% \end{center}
% This formulation mirrors the canonical problem of sparse regression in high-dimensional statistics \citep{hastie2015statistical,wainwright2019high}, where the aim is to learn the parameters of a high-dimensional linear model with sample complexity polynomial in the number of non-zero components and nearly independent of the ambient dimension. In addition, it provides a simplified testbed in which to study analogous issues of exogeneity that arise in the context of RL with rich observations \citep{efroni2021provable}.
% }

% The number of endogenous factors is known to be upper bounded by $k$, i.e.,  $\abr{\Ic} \leq k$, $\abr{\Inc}\geq d-k$. Hence, it holds that $\abr{\Scal\sbr{\Ic}}\leq \S^k \ll \abr{\Scal} = \S^d$ when $k\ll d,$.

% \dfcomment{Explain why naive stuff doesn't work (analogous to RichLQR)}

% \dfcomment{Highlight simplicity of the setting/testbed}

% We study a class of MDPs which we refer as Exogenous Factor MDPs (\kEMDPs). An \kEMDP\ is an MDP with a factorized structure. A state is encoded as a set of $d\in \mathbb{N}$ factors $s = \rbr{s_1,\ldots,s_d}$ and any factor $s_i\in [\S]$\footnote{Generalization to a setting in which $s_i\in [\S_i]$ is straightforward. For brevity we focus on the simpler model.}. Thus, the state space is given by. 

% \subsection{Challenges, Key Idea, and New Techniques}
\subsection{Challenges of RL in the Presence of Exogenous Information}

% \dfcomment{We can explain what an (endogeneous) policy cover is here if it helps with the explanation}
% Non trivial: Bellman rank is large. 

% Given prior knowledge of the endogenous factors $\Ic$, one can ignore the exogenous component of the state space and learn a near-optimal policy using $\poly(\S^k,A,H,\eps^{-1})$ episodes. 
Sample-efficient learning in the absence of prior knowledge poses significant algorithmic challenges.
% \item
% \begin{enumerate}
  \begin{enumerate}[label = $(\mathrm{C}\arabic*)$]
  \item \emph{Hardness of identifying endogenous factors.} In general, the endogenous factors may not be identifiable (that is, multiple choices for $\Ic$ may obey the structure in \eqref{eq: model assumption d}). Even when $\Ic$ is identifiable, \emph{certifying} whether a particular factor $i\in\brk{d}$ is exogenous can be statistically intractable (e.g., if the effect of the agent's action on the state component $s_i$ is small relative to $\eps$).\loose
    % \label{challenge:1}    
  \item \emph{Necessity of exploration.} The agent's action might have a large effect on an endogenous factor $i\in\Ic$, but only in a particular state $s\in\cS$ that requires deliberate planning to reach. As such, any approach that attempts to recover the endogenous factors must be interleaved with exploration, resulting in a chicken-and-egg problem. %Simple
    ``Test-then-explore'' approaches do not suffice.\loose
    % it may statistically impossible whether a particular factor $i$ exogenous
    % An endogenous factor $i\in \brk{d}$ might be  unidentifiable in the statistical limit of $\eps$. This arises in practice when a feature is very weakly affected by the agent's actions. In such a case, it may statistically impossible to certify such a factor as endogenous or exogenous. 
  \item \emph{Entanglement of endogenous and exogenous factors.}
    The factorized dynamics in \pref{eq: model assumption d} lead to a number of useful structural properties for \kEMDPs, such as factorization of state occupancy measures (cf. \pref{app:structural}). However, these properties generally only hold for policies that act on the endogenous portion of the state.  When an agent executes a policy whose actions depend on the exogenous state factors, the evolution of the endogenous and exogenous components becomes entangled. This entanglement makes it difficult to apply supervised learning or estimation methods to extract information from trajectories gathered from such policies, and can lead to error amplification. As a result, significant care is required in gathering data.

    % as well, the components becomes correlated; the endogenous and exogenous state factors are no longer independent in this scenario, although the transition operator has a factorized structure.
          \end{enumerate}
          \paragraph{Failure of existing algorithms.} Existing RL techniques do not appear to be sufficient to address the challenges above and generally have sample complexity requirements scaling with $\bigom(d)$ or worse. For example, tabular methods do not exploit factored structure, resulting in $\bigom(S^d)$ sample complexity, and we can show that complexity measures like the Bellman rank~\citep{jiang2017contextual} and its variants scale as $\bigom(d)$, so they do not lead to sample-efficient learning guarantees. Moreover, algorithms for factored MDPs (e.g.,~\citet{rosenberg2020oracle}) obtain guarantees that depend on sparsity in the transition operator, but this operator is dense in the \framework setting, leading to sample complexity that is exponential in $d$. See further discussion in \pref{sec:related,app:bellman_rank}.

    \subsection{Main Result} 
    % Our main contribution is the formulation and analysis of a learning algorithm that returns an $\eps$-optimal policy with high probability for a general \kEMDP.

We present a new algorithm, \mainalg, which learns a near-optimal policy for the \framework with sample complexity polynomial in the number of endogenous states and \emph{logarithmic} in the number of exogenous components. Following previous approaches to representation learning in RL \citep{du2019provably,misra2020kinematic,agarwal2020flambe}, our results depend on a \emph{reachability parameter}.%, defined as follows.
\begin{definition}%[Reachability]
The endogenous state space is $\eta$-reachable if for all $h\in\brk{H}$ and $s\brk{\Ic}\in \cS\brk{\Ic}$, either\loose
\begin{align*}
    \max_{\pi\in \PiIndNS} \Prpi\rbr{s_h\brk{\Ic} = s\brk{\Ic}}\geq \eta, \quad  \mathrm{or} \quad \max_{\pi\in \PiIndNS} \Prpi\rbr{s_h\brk{\Ic} = s\brk{\Ic}} = 0.
\end{align*}
\end{definition}
Crucially, this notation of reachability considers only the endogenous portion of the state space, not the full state space. We assume access to a lower bound $\eta$ on the optimal reachability parameter.

Our main result is as follows.

% \YEcomment{TODO: change with the reachability parameter.}
\newtheorem*{thm:informal1}{Theorem \ref*{thm:main} (informal)}
\begin{thm:informal1}
  % \begin{theorem}[informal]
  
  \label{thm: informal main theorem}
  With high probability, $\mainalg$ learns an $\eps$-optimal policy for the \framework using $\poly(\S^k,A,H,\log(d))\cdot{}\rbr{\eps^{-2}+\eta^{-2}}$ trajectories.
\end{thm:informal1}
This constitutes a \emph{doubly-exponential} improvement over the $S^{d}$ sample complexity for naive tabular RL in terms of dependence on the number of factors $d$, and it provides a RL analogue of sparsity-dependent guarantees in high-dimensional statistics \citep{hastie2015statistical,wainwright2019high}. Importantly, the result does not require any statistical assumptions beyond the factored structure in \eqref{eq: model assumption d} and reachability (for example, we do not require deterministic dynamics). Beyond polynomial factors, the dependence on the size of the state space cannot be improved further.

% In particular, our algorithm requires only logarithmic number of trajectories in the number of factors $d$, and polynomial number of trajectories in the number of relevant factors $k$.

% This is analogous to results on learning sparse models in statistics and machine learning.

% To the best of our knowledge,~\pref{thm: informal main theorem} is the first result in the RL literature that provides such guarantee, while tackling the exploration problem under no assumptions. We also comment that polynomial dependence in $\S^k,A$ and $H$ is necessary.

% \dfcomment{We should emphasize here that this gives a doubly-exponential improvement over naive tabular RL.}

% \subsection{Approach and Key Techniques}
% \paragraph{Approach and key techniques.}
% \paragraph{Approach and key techniques.}
\subsection{Our Approach: Exploration with a Certifiably Endogenous Policy Cover}

% Our approach utilizes a special structure of an \kEMDP. In such a model, the endogenous state space $\cS\brk*{\Ic}$ can be efficiently explored with a set of policies that depend only on the endogenous factors~$\Ic$, which we refer as an \emph{endogenous policy}.

\mainalg is built upon the notion of an \emph{endogenous policy cover}. Define an endogenous policy as follows.\colt{\vspace{-10pt}}
\begin{definition}[Endogenous policy]\label{def: endogenous policy}
  A policy $\pi=(\pi_1,\ldots,\pi_H)$ is \emph{endogenous} if it acts only on the endogenous component of the state space: For all $h\in\brk{H}$ and $s\in\cS$, we have $\pi_h(s)=\pi_h(s\brk{\Ic})$.\loose
  % A policy or a non-stationary policy $\pi$
  % is an \emph{endogenous policy} if $\pi\in \PiInd\sbr{\Ic}$ or $\pi\in \PiIndNS\sbr{\Ic}$, respectively.
\end{definition}
% The main primitive upon which \mainalg is built is that of
An endogenous policy cover is a  (small) collection of endogenous policies that ensure each state is reached with near-maximal probability.
\begin{definition}[Endogenous policy cover]\label{def: approximate endognous policy cover}
  A set of non-stationary policies $\Psi$ is an endogenous ($\eps$-approximate) policy cover for timestep $h$ if:
\begin{enumerate}
\item For all $s\in\cS$, $\max_{\psi\in\Psi}\Prpi[\psi]\rbr{ s_h\brk{\Ic}=s\brk{\Ic} } \geq{} \max_{\pi\in\PiNS} \Prpi\rbr{ s_h\brk{\Ic}=s\brk{\Ic}}- \eps$.% \dipendra{define what second max is over}
    \item The set $\Psi$ contains only endogenous policies.% and $\abr{\Psi\ind{h}}\leq \abr{\Scal\sbr{\Ic}}$.
    % \item $\abr{\Psi\ind{h}}\leq \abr{\Scal\sbr{\Ic}}$.
\end{enumerate}
\end{definition}
While the coverage property of \pref{def: approximate endognous policy cover} is stated in terms of occupancy measures for the endogenous portion of the state space, the factored structure of the \framework implies that this yields a cover for the entire state space (cf. \pref{app:structural_occupancy}):
\[
  \max_{\psi\in\Psi}\Prpi[\psi]\rbr{ s_h=s } \geq{} \max_{\pi} \Prpi\rbr{ s_h=s}- \eps,\quad\forall{}s\in\cS.
\]
In particular, even though $\abs{\cS}=S^{d}$, this guarantees that for each timestep $h$, there exists a \emph{small} endogenous policy cover with $\abs{\Psi}\leq{}S^{k}$. \mainalg constructs such a policy cover and uses it for sample-efficient exploration in two phases. First, in Phase I (\AlgName), the algorithm builds the policy cover in a manner guaranteeing endogeneity; this accounts for the majority of the algorithm design and analysis effort. Then, in Phase II (\PSDPE), the algorithm uses the policy cover to optimize rewards.\loose

% The first requirement formalizes the notion of near-optimal policy cover over the endogenous state space. The second asserts that $\Psi\ind{h}$ contains only endogenous policies that do not depend on the exogenous factors, and has small cardinality. 

\paragraph{Finding a certifiably endogenous policy cover: \AlgName.}

The main component of \mainalg is an algorithm, \pcalg, which iteratively learns a sequence of endogenous policy covers $\Psi\ind{1},\ldots,\Psi\ind{H}$ with \[\max_{\psi\in\Psi\ind{h}}\Prpi[\psi]\rbr{ s_h\brk{\Ic}=s\brk{\Ic} } \geq{} \max_{\pi} \Prpi\rbr{ s_h\brk{\Ic}=s\brk{\Ic}}- \eps\] for all $s\brk{\Ic}\in\cS\brk{\Ic}$. For each $h\in\brk{H}$, given the policy covers $\Psi\ind{1},\ldots,\Psi\ind{h-1}$ for preceding timesteps, \pcalg builds the policy cover $\Psi\ind{h}$ using a novel statistical test. The test constructs a factor set $\cI\subset\brk{d}$ which is (i) endogenous, in the sense that $\cI\subset\Ic$, yet (ii) ensures sufficient coverage, in the sense that there exists a near-optimal policy cover operating only on $s\brk{\cI}$. The analysis of this test relies on a unique structural property of the \framework setting called the \emph{restriction lemma} (\pref{lem: invariance of reduced policy cover}), which provides a mechanism to ``regularize'' the factor set under consideration toward endogeneity in a data-driven fashion.

This approach circumvents challenges $(\mathrm{C}1)$ and $(\mathrm{C}2)$: It does not rely on explicit identification of the endogenous factors and instead iteratively builds a \emph{subset} of factors that is certifiably endogenous, but nonetheless sufficient to explore. Endogeneity of the resulting policy cover $\Psi\ind{h}$ ensures the success of subsequent tests at rounds $h+1,\ldots,H$, and circumvents the issue of entanglement raised in challenge $(\mathrm{C}3)$. To summarize, the following guarantee constitutes our main technical result.
  \newtheorem*{thm:informal2}{Theorem \ref*{thm: sample complexity of state refinment} (informal)}
\begin{thm:informal2}
% \begin{theorem}[informal]
  \label{thm: informal app endo policy cover theorem}
With high probability, $\AlgName$ finds an endogenous $\frac{\eta}{2}$-approximate policy cover using $\poly\prn*{\S^k,A,H,\log(d)}\cdot\eta^{-2}$ trajectories.
  % Given $\eps$-endogenous policy covers over $t\in [h-1]$, $\AlgName_h^{\eps,\delta}$ returns an $\eps$-endogenous policy cover over the $h^{\th}$ timestep for any \kEMDP\ with probability greater than $1-\delta$ given $\poly(\S^k,A,H,\log(d/\delta))/\eps^2$ trajectories.
\end{thm:informal2}

\subsection{Organization}

The remainder of the paper is organized as follows. In \pref{sec:warmup}, we introduce the \AlgName algorithm, highlight the key algorithm design techniques and analysis ideas, and state its formal guarantee (\pref{thm: sample complexity of state refinment}) for finding a policy cover. Building on this result, in \pref{sec:main} we introduce the \mainalg algorithm, and provide the main sample complexity guarantee for RL in \frameworks (\pref{thm:main}). We close with discussion of additional related work (\pref{sec:related}) and open problems (\pref{sec:conclusion}).

\subsection{Preliminaries}

% \paragraph{Additional reinforcement learning notation.}}
We let $\PiInd$ denote the set of all one-step policies $\pi:\cS\to\cA$. We use the term \emph{$t\to{}h$ policy} to refer to a non-stationary policy $\pi=(\pi_{t},\ldots,\pi_h)$ defined over a subset of timesteps $t\leq h$.%, which take the form. %We refer to such policies as \emph{$t\to{}h$ policies}.

For a non-stationary policy $\pi\in\PiNS$, we define the state-action and state value functions:
 $Q_h^\pi(s,a) \ldef \EE_\pi\sbr{\sum_{h'=h}^H r_{h'}\mid{}s_h=s,a_h=a}$, and $V_h^\pi(s) \ldef{} Q_h^\pi(s,\pi_h(s)).$ We denote the expected value of a policy $\pi$ from time step $t$ to $h$ by $ V_{t,h}\rbr{\pi } \ldef \EE_{\pi }\brk*{ \sum_{t'=t}^h r_{t'}}.$
% \begin{align*}
%   % Q_h^\pi(s,a) \ldef \EE_\pi\sbr{\sum_{h'=h}^H r\rbr{s_{h'},\pi\ind{h}(s_{h'})} \mid s_h=s,a_h=a}, \quad V_h^\pi(s) = \EE_\pi\sbr{Q_h^\pi(s,a)\mid s_h=s},
%   Q_h^\pi(s,a) \ldef \EE_\pi\sbr{\sum_{h'=h}^H r_{h'}\mid{}s_h=s,a_h=a}, \quad\text{and}\quad V_h^\pi(s) \ldef{} Q_h^\pi(s,\pi_h(s)).
%  \end{align*}
We adopt the shorthand $d_h(s \midsem \pi) \ldef \PP_\pi(s_h = s)$ for the induced state occupancy measure. Likewise, for $\cI\subseteq\brk{d}$, we define  $d_h(s\brk{\cI} \midsem \pi) \ldef \PP_\pi(s_h\brk{\cI} = s\brk{\cI})$.

 % An $H$ step non-stationary policy is a sequence of $H-1$ policies $\rbr{\pi\ind{1},\ldots,\pi\ind{H-1}}$ where  $\pi\ind{t}\in \Pi$ is  executed at the $t^{\th}$ timestep. We denote the set of $H$ step non-stationary policies as $\PiIndNS$. We refer to a non-stationary policy that is executed from $t_1$ to $t_2$ timesteps as a $\rbr{t_1\rightarrow t_2}$ policy.
 % \dfcomment{Add notation for mixture policies like $\mu\ind{t}$. For example, we could say $\mu\ind{t}\in\Pi_{\mathrm{mix}}\ldef\Delta(\Pi_{\mathrm{NS}})$.} \dfcomment{$\En_{\mu}$}

For algorithm design purposes, we consider \emph{mixture policies} of the form $\mu\in \Pimix\ldef{}\Delta(\PiNS)$. To run a mixture policy $\mu\in\Pimix$, we sample $\pi\sim\mu$, then execute $\pi$ for an entire episode. We further denote $\Pimix\brk{\cI} \ldef \Delta(\PiNS\brk{\cI})$ as the set of mixture policies over the policy set $\PiNS\brk{\cI}$, where $\PiNS\brk{\cI}$ denotes the set of policies that act on the factor set $\cI$.  We let $\En_{\mu}\brk{\cdot}$ and $\Prpi[\mu]\prn{\cdot}$ denote the expectation and probability under this process, and we define $\J(\mu)=\En_{\pi\sim\mu}\brk{\J(\pi)}= \En_{\mu}\brk*{\sum_{h=1}^{H}r_h}$ and $d_h(s \midsem \mu) \ldef \PP_{\mu}(s_h = s)$ analogously. We say that $\mu\in\Pimix$ is endogenous if it is supported over endogenous policies in $\PiNS$. Finally, for $\mu\in \Pimix$ and $\pi\in \PiInd$ we let  $\mu \circt \pi$ be the policy that follows $\mu$ for the first $t-1$ timesteps, and at the $t^{\mathrm{th}}$ timestep it switches to $\pi$. For sets of policies $\Psi_1$ and $\Psi_2$ we let $\Psi_1 \circt \Psi_2 \ldef \cbr{\psi_1\circt \psi_2 \mid \psi_1\in \Psi_1, \psi_2\in \Psi_2}$.

% \paragraph{Additional \framework notation.}
 \paragraph{\framework notation.}

 Recall that for a factor set $\cI\subseteq\brk{d}$, we define $\cI^\mathrm{c} \ldef [d]\setminus \cI$ as the complement, and define $s\sbr{\cI}\ldef \prn{s_i}_{i\in\cI}$ and $\Scal\sbr{\cI} \ldef \otimes_{i\in\cI}\cS_i$ as the corresponding components of the state and state space. We make frequent use of the fact that for any pair of factors $\cI_1$ and $\cI_2$ with $\cI=\cI_1\cup \cI_2$ and $\cI_1\cap \cI_2=\emptyset$, any state $s\brk*{\cI}\in \cS\brk*{\cI}$ can be uniquely split as $s\brk*{\cI} = \rbr{s\brk*{\cI_1}, s\brk*{\cI_2}}$, with $s\brk{\cI_1}\in\cS\brk{\cI_1}$ and $s\brk{\cI_2}\in\cS\brk*{\cI_2}$. We use a canonical ordering when indexing with factor sets.

Any factor set $\cI\subseteq\brk{d}$ can be written as $\cI = \rbr{\cI\cap \Ic} \cup \rbr{\cI \cap \Inc}$. We denote these intersections by $\Iendo{\cI} \ldef \cI\cap \Ic$ and $\Iexo{\cI} \ldef \cI\cap \Inc$, which represent the endogenous and exogenous components of $\cI$.\loose 

 We say that a policy $\pi$ \emph{acts on a factor set $\cI$} if it  selects actions as a measurable function of $\cS\brk{\cI}$. We let $\PiInd\brk{\cI}$ denote the set of all one-step policies $\pi:\cS\brk{\cI}\to\cA$ that act on $\cI$, and let $\PiIndNS\brk{\cI}$ denote the set of all non-stationary policies that act on $\cI$.
 
 Lastly, if $\cI\subseteq \Ic^{\mathrm{c}}$, i.e., the factor $\cI$ is a subset of the exogenous factors, we omit the dependence in the policy $\pi$ from its occupancy measure, $d_h(s\brk{\cI} \midsem \pi)= d_h(s\brk{\cI})$.  Indeed, for any $\pi,\pi'\in \PiNS$ it holds that $d_h(s\brk{\cI} \midsem \pi) = d_h(s\brk{\cI} \midsem \pi')$, and hence the occupancy measure of $s\brk{\cI}$ is independent of the policy.
 % We denote the set of policies that depend on the set of factor $\cI$ as $\PiInd\sbr{\cI}$, and the set of $H$-step non-stationary policies that depends on $\cI$ as $\PiIndNS\sbr{\cI}$. 

\paragraph{Collections of factor sets.}
For a factor set $\cI\subseteq\brk{d}$, we let $\fullc_{\leq k}(\cI) \ldef \cbr{\cI' \subseteq[d] \mid \cI \subseteq \cI',\ \abr{\cI'}\leq k}$ denote a collection of all factor sets of size at most $k$ that contain $\cI$, and analogously define $\fullc_{k}(\cI) \ldef \cbr{\cI' \subseteq[d] \mid \cI \subseteq \cI',\ \abr{\cI'} = k}$. We adopt the shorthand $\fullc_{\leq k} \ldef \fullc_{\leq k}\rbr{\emptyset}$ and $\fullc_{k} \ldef \fullc_{k}\rbr{\emptyset}$. With some abuse of notation, for a given collection of factor sets $\fullc$, we define $\PiInd\brk*{\fullc} \ldef \cup_{\cI\in \fullc} \PiInd\brk*{\cI}$ as the set of all possible policies induced by factors in $\fullc$.
% , respectively.
%

% See that by the \kEMDP\ model assumptions it holds that $\Ic\in \fullc_{\leq k}$, and if $\cI\subseteq \Ic$ then $\Ic\in \fullc_{\leq k}\rbr{\cI}$. 

% Furthermore, it can be shown that for any $\cI$ with $\cJ\leq k$ the set $\fullc_{\leq k}(\cI)$ is a $\pi$ system (see~\pref{app: structural results on class}).
% \akcomment{This paragraph needs to move up, so that endo-policy cover can be well-defined.}
% \paragraph{

% Lastly, we define an similar notion of state frequency for an induced state factor $\cS\brk*{\cI}$ by some $\cI$. Given a policy $\pi\in \PiIndNS$, we let the state frequency for the induced state $s\brk*{\cI}\in \cS\brk*{\cI}$ at timestep $h$ be $d_h\rbr{s\sbr{\cI} \midsem \pi} \ldef \PP_\pi\rbr{s_h\sbr{\cI} = s\sbr{\cI}}$.

We define $[N]\ldef \{1, 2, \cdots, N\}$. $\unf(\cX)$ denotes the uniform distribution over a finite set $\cX$.\loose

\section{Learning a Near-Optimal Endogenous Policy Cover: $\AlgName$}
  % Learning an Endogenous Policy Cover with Exact Queries}
% \label{sec: EXACT CASE reward free exploration of endo state}
\label{sec:warmup}
% !TEX root = paper.tex

% \begin{enumerate}
%     \item Rolling with an endogenous policy for 1-step MDP.
%     \item Extend to general MDP via state refinement.
%     \end{enumerate}

%     \dfcomment{Remind people why we care about finding a policy cover here.}

In this section, we present the first of our main algorithms, \pcalg (\pref{alg: dp-sr paper version}), which performs reward-free exploration to construct an endogenous policy cover for the \framework. \pcalg constitutes the main algorithmic component of \mainalg, and we believe it is of independent interest. % but we present it separately because

\pcalg is a \emph{forward-backward} algorithm. For each layer $h\in\brk{H}$, given previous policy covers $\Psi\ind{1},\ldots,\Psi\ind{h-1}$, the algorithm constructs an endogenous policy cover $\Psi\ind{h}$ in a backwards fashion. Backward steps proceed from $t=h-1,\ldots,1$, with each step consisting of (i) an \emph{optimization} phase, in which we find a (potentially large) collection of policies for choosing actions at step $t$ that lead to good coverage for all possible target factors sets $\cI$ at layer $h$, and (ii) a \emph{selection} phase, in which we narrow the collection of policies from the first phase down to a small set of policies that act on a single (endogenous) factor set $\cI$, yet still ensure coverage for all states at step $h$.

Instead of directly diving into \pcalg, we build up to the algorithm through two warm-up exercises:\loose
\begin{itemize}
\item In \pref{sec: endogenous policy cover 1 step exact}, we consider a simplified version of \pcalg (\ExactOneStep, or  \pref{alg: 1step endo policy cover}) which computes an endogenous policy cover under the assumption that (i) $H=2$, and (ii) certain occupancy measures for the underlying \framework can be computed exactly.
\item Building on this result, in \pref{sec: DP-SR exact} we provide another simplified algorithm (\pcalgexact, or \pref{alg: exact dynamics state refinement}) which computes an endogenous policy cover for general $H$, but still requires exact access to certain occupancy measures for the \framework.
% \item 
\end{itemize}
Finally, in \pref{sec: learning endogenous policy cover} we present the full \pcalg algorithm and its main sample complexity guarantee.

% To clarify the intuition of $\AlgName^{\eps,\delta}_{t,h}$, we start by considering the problem of finding an exact ($\eps=0$) endogenous policy cover when queries to exact values of $d_h\rbr{s\brk*{\cI}  \midsem \pi}$ are available. Our solution for finding an exact endogenous policy cover relies on two key technologies we introduce: an ability to \emph{find an endogenous policy cover} for a 1-step \kEMDP\ (see~\pref{sec: endogenous policy cover 1 step exact}) and \emph{state refinement} process by which we extend this ability to a general $H$-step \kEMDP\ (see~\pref{sec: DP-SR exact}).

% \subsection{Challenge: Endogenous Factors Should be Learned Adaptively and Non-identifiability \YE{should we have it here and above?}}

\subsection{Warm-Up I: Finding an Endogenous Policy Cover with Exact Queries ($H=2$)}\label{sec: endogenous policy cover 1 step exact}

\begin{algorithm}[htp]
  \caption{$\ExactOneStep$: \pcalglong for \kEMDPs with $H=2$}
  \label{alg: 1step endo policy cover}
\begin{algorithmic}[1]
  % \item[]{\algcommentbig{\textbf{Optimization Phase}}}
\item[]{\algcommentul{Phase I: Optimization}}\vspace{2pt}
    \State \label{line: exact 1 endo optimization} Find factor set $\cItil\in \fullc_{\leq k}$ with minimal cardinality such that for all $ \cJ\in
      \fullc_{\leq k}$ and $s\brk{\cJ}\in \cS\brk{\cJ}$,
\[
        \max_{\pi\in \PiInd\brk{\fullc_{\leq k}}} d_2\rbr{s\brk{\cJ} \midsem \pi} = \max_{\pi\in \PiInd\brk{\cItil}}d_2\rbr{s\brk{\cJ} \midsem \pi}.
      \]
      \State\label{line:onestep_line2} For all $\cJ\in\fullc_{\leq{}k}$, define $\pi_{s\brk{\cJ}}= \arg\max_{\pi\in \PiInd\brk{\cItil}} d_2\rbr{s\brk{\cJ} \midsem \pi}$ for each $s\brk{\cJ}\in\cS\brk{\cJ}$, then set
      \[
        \Gamma\sbr{\cJ} \ldef \cbr{\pi_{s\brk{\cJ}} : s\brk{\cJ}\in \cS\brk{\cJ}}.
      \]
    % \vspace{0.05cm}
  % \item[]{\algcommentbig{\textbf{ Selection Phase}}}
  \item[]{\algcommentul{Phase II: Selection}}\vspace{2pt}
    \State  \label{line: exact 1 endo selection} Find factor set $\widehat{\cI}\in \fullc_{\leq k}$ with minimal cardinality such that for all $\cJ\in\fullc_{\leq k}$ and $s\sbr{\cJ} \in \cS\sbr{\cJ}$,
    % for $\pi_{s\brk{\cJ \cap \widehat{\cI}}}\in \Gamma\brk*{\cJ \cap \widehat{\cI}}$ it holds that 
    \[
      \max_{\pi\in \PiInd\brk{\fullc_{\leq k}} } d_2\rbr{s\brk{\cJ} \midsem \pi} = d_2\rbr{s\brk{\cJ} \midsem \pi_{s\brk{\cJ \cap \widehat{\cI}}}}.
      \]
    %\label{eq: main paper stage 2 exact}
    \State \textbf{return} $\prn[\big]{\widehat{\cI}, \Gamma\brk{\widehat{\cI}}}$
\end{algorithmic}
\end{algorithm}

% Consider a 1-step \kEMDP, an \kEMDP\ with $H=2$. For such an \kEMDP\ we formulate a procedure---Optimization-Selection for 1-Step \kEMDP\ ($\ExactOneStep$)---depicted in~\pref{alg: 1step endo policy cover}, that finds an exact ($\eps=0$) endogenous policy cover, without knowledge of $\Ic.$  

\pref{alg: 1step endo policy cover} presents \pcalgonestep, a simplified version of \pcalg that computes a (small) endogenous policy cover for horizon two, assuming exact access to the state occupancies $d_2(s\midsem\pi)$. This algorithm highlights the mechanism through which \pcalg is able to simultaneously ensure both endogeneity and coverage.

\pcalgonestep learns an endogenous policy cover in two phases. In the \emph{optimization phase} (\savehyperref{line: exact 1 endo optimization}{Lines \ref*{line: exact 1 endo optimization} and \ref*{line:onestep_line2}}) the algorithm computes a \emph{partial policy cover} $\Gamma\brk{\cJ}$ for each factor set $\cJ\in\fullc_{\leq{}k}$, which ensures that for all state factor values $s\brk{\cJ}\in \cS\brk{\cJ}$ there exists a policy $\pi_{s\brk{\cJ}}\in \Gamma\brk{\cJ}$ which maximizes the probability to reach the state factor value $s\brk{\cJ}$ at the $2^{\mathrm{nd}}$ timestep.
% $s\brk{\cJ}\in\cS\brk{\cJ}$ the probability of reaching each state fac $\max_{\pi\in\Gamma\brk{\cJ}}\Prpi[\pi]\rbr{ s_h\brk{\cJ}=s\brk{\cJ} } \geq{} \max_{\pi} \Prpi\rbr{ s_h\brk{\cJ}=s\brk{\cJ}}$.
% \begin{definition}[Partial policy cover]\label{def:partial_cover}
  % A set of non-stationary policies $\Psi$ is a ($\eps$-approximate) \emph{partial policy cover} for a factor set $\cJ$ at timestep $h$ if for all.
% \end{definition}

% \begin{definition}[Partial policy cover]\label{def:partial_cover}
%   A set of non-stationary policies $\Psi$ is a ($\eps$-approximate) \emph{partial policy cover} for a factor set $\cJ$ at timestep $h$ if for all $s\brk{\cJ}\in\cS\brk{\cJ}$, $\max_{\psi\in\Psi}\Prpi[\psi]\rbr{ s_h\brk{\cJ}=s\brk{\cJ} } \geq{} \max_{\pi} \Prpi\rbr{ s_h\brk{\cJ}=s\brk{\cJ}}- \eps$.
% \end{definition}}
All of the partial policy covers are induced by a single factor set $\cItil$; existence of such a factor set is guaranteed by \pref{property: endo maximizer is sufficient}. We show that by regularizing by cardinality, $\cItil$ is guaranteed to be endogenous, and so the policy covers $(\Gamma\brk{\cJ})_{\cJ\in\fullcleq}$ are endogenous as well.

At this point, the only issue is size: The set $\bigcup_{\cJ\in\fullcleq}\Gamma\brk{\cJ}$ is an exact policy cover for $h=2$ (in the sense of \pref{def: approximate endognous policy cover}), but its size scales as $\bigom(d^{k})$,\footnote{The set $\Pi\brk{\cItil}$ also gives a policy cover, but it is even larger.} which makes it unsuitable for exploration. To address this issue, the \emph{selection phase} (\pref{line: exact 1 endo selection}) identifies a single endogenous factor $\cIhat$ such that $\Gamma\brk{\cIhat}$ is an endogenous policy cover (note that choosing $\Gamma\brk{\Ic}$ would suffice, but $\Ic$ is not known to the learner). Since $\abs{\Gamma\brk{\cIhat}}\leq{}S^{k}$ by construction, this yields a small policy cover as desired.\loose

\begin{comment}
\begin{proposition}[Recovery of 1-Step Endogenous Policy Cover]\label{prop: exact case recovery of 1-step endo policy cover}
  $\ExactOneStep$ returns $(\widehat{\cI}, \Gamma\brk{\widehat{\cI}})$ where
  \begin{enumerate}[label = $(\roman*)$,align=left]
      \item $\widehat{\cI}$ is a set of endogenous factors, $\widehat{\cI}\subseteq\Ic$.
      \item $\Gamma\brk{\widehat{\cI}} =\cbr{\pi_{s\brk{\widehat{\cI}}}: s\brk{\widehat{\cI}} \in \cS\brk{\widehat{\cI}}}$, is a set of endogenous policies that reaches~$\cS\brk{\widehat{\cI}}$. 
      \item $\Gamma\brk{\widehat{\cI}}$ is an endogenous policy cover. For all $s\brk*{\Ic} = (s\brk{\widehat{\cI}} , s\brk*{\Ic \setminus{} \widehat{\cI}})\in \cS\brk*{\Ic}:$
          $$
          \max_{\pi\in \PiInd} d_2\rbr{s\brk*{\Ic} \midsem \pi} = d_2\rbr{s\brk*{\Ic} \midsem \pi_{s\brk{\widehat{\cI}}}},\ \text{for}\ \pi_{s\brk{\widehat{\cI}}}\in \Gamma\sbr{\widehat{\cI}}.
          $$ 
        \end{enumerate}
        % \dfcomment{Fix (i), (ii), etc. above so that they are left-aligned rather than right aligned.}
\end{proposition}
\end{comment}
\begin{proposition}%[Correctness of $\ExactOneStep$]%[Recovery of 1-Step Endogenous Policy Cover]
  \label{prop: exact case recovery of 1-step endo policy cover}
   The pair $\prn[\big]{\widehat{\cI}, \Gamma\brk{\widehat{\cI}}}$ returned by \pcalgonestep has the property that (i) $\widehat{\cI}$ is endogenous (i.e., $\widehat{\cI}\subseteq\Ic$), and (ii) $\Gamma\brk{\widehat{\cI}}$ is an endogenous policy cover for $h=2$: For all $s\in\cS$,\loose
        \[
                  \max_{\pi\in \PiInd} d_2\prn[\big]{s\brk{\Ic} \midsem \pi} = d_2\prn[\big]{s\brk{\Ic} \midsem \pi_{s\brk{\widehat{\cI}}}},\quad\text{where $\pi_{s\brk{\widehat{\cI}}}\in \Gamma\brk{\widehat{\cI}}$.}
        \]
      \end{proposition}

%   \label{prop: exact case recovery of 1-step endo policy cover}
%   The pair $\prn[\big]{\widehat{\cI}, \Gamma\brk{\widehat{\cI}}}$ returned by \pcalgonestep has the following properties:
%   \begin{enumerate}[label = $(\roman*)$,align=left]
%       \item $\widehat{\cI}$ is an endogenous factor set (i.e., $\widehat{\cI}\subseteq\Ic$).
%       % \item $\Gamma\brk{\widehat{\cI}} =\cbr{\pi_{s\brk{\widehat{\cI}}}: s\brk{\widehat{\cI}} \in \cS\brk{\widehat{\cI}}}$, is a set of endogenous policies that reaches~$\cS\brk{\widehat{\cI}}$. 
%       \item $\Gamma\brk{\widehat{\cI}}$ is an endogenous policy cover for $h=2$. In particular, for all $s\in\cS$,
%         % \[
%         %   \max_{\pi\in \PiInd} d_2\rbr{s \midsem \pi} = d_2\rbr{s \midsem \pi_{s\brk{\widehat{\cI}}}},\quad\text{where $\pi_{s\brk{\widehat{\cI}}}\in \Gamma\brk{\widehat{\cI}}$.}
%         % \]
%                 \[
%                   \max_{\pi\in \PiInd} d_2\prn[\big]{s\brk{\Ic} \midsem \pi} = d_2\prn[\big]{s\brk{\Ic} \midsem \pi_{s\brk{\widehat{\cI}}}},\quad\text{where $\pi_{s\brk{\widehat{\cI}}}\in \Gamma\brk{\widehat{\cI}}$.}
%         \]
%       % .
%         \end{enumerate}
%         % \dfcomment{Fix (i), (ii), etc. above so that they are left-aligned rather than right aligned.}
%       \end{proposition}}
      The \framework transition structure further implies that $\max_{\pi\in \PiInd} d_2\prn[\big]{s \midsem \pi} = d_2\prn[\big]{s \midsem \pi_{s\brk{\widehat{\cI}}}}\;\forall{}s\in\cS$.\looseness=-1
% \dfcomment{Explain again that this yields a policy cover for the full state space.}
% $\ExactOneStep$ returns $(\widehat{\cI},\Gamma\brk{\widehat{\cI}}$) such that $\Gamma\brk{\widehat{\cI}}$ optimally reaches to $\cS\brk{\widehat{\cI}}$, and by doing so, it also optimally reaches to $\cS\brk*{\Ic}.$ Next, we elaborate on the proof of~\pref{prop: exact case recovery of 1-step endo policy cover}. 
% \subsubsection{Proof of \pref{prop: exact case recovery of 1-step endo policy cover}: Correctness of $\ExactOneStep$}
% \subsubsection{Proof of \pref{prop: exact case recovery of 1-step endo policy cover}}
\begin{proof}[\pfref{prop: exact case recovery of 1-step endo policy cover}]
We begin by highlighting two useful structural properties of the \kEMDP; both properties are specializations of more general results, \pref{lem: decoupling of future state dis,lem: invariance of reduced policy cover} (\pref{app:structural}).
\begin{property}[Decoupling for endogenous policies]\label{property: decoupling}
 For any endogenous policy $\pi$, we have $d_2\rbr{s\brk*{\cI} \midsem \pi} =d_2\rbr{s\brk*{\Iendo{\cI}} \midsem \pi}\cdot{}d_2\rbr{s\brk*{\Iexo{\cI}}}$, for all $\cI\subseteq\brk{d}$ and $s\in\cS$.
\end{property}
\begin{property}[Restriction lemma]\label{property: endo maximizer is sufficient}
For all factor sets $\cI$ and $\cJ$, we have % and any $$ it holds that
\begin{equation}
  \max_{\pi\in \PiInd\brk{\cI}} d_2\rbr{s\brk{\cJ} \midsem \pi} =\max_{\pi\in \PiInd\brk{\Iendo{\cI}}} d_2\rbr{s\brk{\cJ} \midsem \pi}\quad \forall{}s\brk*{\cJ}\in \cS\brk*{\cJ}.\label{eq:restriction1}
\end{equation}
% \begin{equation}
% \max_{\pi\in \Pi} d_2\rbr{s\brk{\cJ} \midsem \pi} = \max_{\pi\in \PiInd\brk{\Ic}} d_2\rbr{s\brk{\cJ} \midsem \pi}.\label{eq:restriction2}
% \end{equation}}
\end{property}
% \YEcomment{think it is better to show it; don't understand why it is by construction, we need to use the fact that $\widehat{\cI}$ is endogenous to show this}

  \pref{property: endo maximizer is sufficient} is perhaps the most critical structural result used by our algorithms. It implies that $\max_{\pi\in \Pi} d_2\rbr{s\brk{\cJ} \midsem \pi} = \max_{\pi\in \PiInd\brk{\Ic}} d_2\rbr{s\brk{\cJ} \midsem \pi}$, which in turn implies that the optimization and selection phases of \pref{alg: 1step endo policy cover} are feasible (since we can show that $\Ic$ is a valid choice). If $\cItil$ and $\cIhat$ are endogenous, then since $\cIhat \subset \Ic$ the selection rule ensures that $\Gamma\brk{\cIhat}$ is a policy cover for $\cS[\Ic]$ (by choosing $\cJ = \Ic$ in~\cref{line: exact 1 endo selection} and since $\cIhat \cap \Ic = \cIhat$). We next show that both $\cItil$ and $\cIhat$ are endogenous.
  
  % In addition, $\Gamma\brk{\cIhat}$ is a policy cover for $\cS\brk{\Ic}$ by construction (and by \pref{property: decoupling} a policy cover for all of $\cS$). It remains to show that $\cItil$ and $\cIhat$ are endogenous.
  
  \vspace{2pt}
  \noindent\emph{Claim 1: $\cItil$ is endogenous.} Observe that for any
    (potentially non-endogenous) factor set
    $\cItil=\Iendo{\cItil} \cup \Iexo{\cItil}$, \pref{property: endo
      maximizer is sufficient} implies that for all
    $\cJ\in \fullc_{\leq k}$ and $s\brk{\cJ}\in \cS\brk*{\cJ}$,
    \[
      \max_{\pi\in \PiInd\brk{\cItil}} d_2\rbr{s\brk{\cJ} \midsem \pi}
      = \max_{\pi\in \PiInd\brk{\Iendo{\cItil}}} d_2\rbr{s\brk{\cJ}
        \midsem \pi},
    \]
    For any factor set $\cItil$ that satisfies the
    constraints in \pref{line: exact 1 endo optimization} but has
    $\Iexo{\cItil}\neq\emptyset$, we can further reduce the
    cardinality without violating the constraints, so the
    minimum cardinality solution is endogenous.\loose

    \vspace{2pt}
    \noindent\emph{Claim 2: $\cIhat$ is endogenous.} Consider a
    (potentially non-endogenous) factor set
    $\widehat{\cI} = \Iendo{\widehat{\cI}} \cup
    \Iexo{\widehat{\cI}}$. If $\widehat{\cI}$ satisfies the constraint
    in~\pref{line: exact 1 endo selection}, then for all
    $\cJ\in \fullc_{\leq k}$ and $s\in\cS$, since
    $\Iendo{\cJ}=\cJ\cap\Ic\in\fullcleq$, %we have
    \begin{align}
      \max_{\pi\in \PiInd\brk{\fullcleq}} d_2\rbr{s\brk{\Iendo{\cJ}} \midsem \pi} = d_2\rbr{s\brk{\Iendo{\cJ}} \midsem \pi_{s\brk{\Iendo{\cJ} \cap \widehat{\cI}}}}= d_2\rbr{s\brk{\Iendo{\cJ}} \midsem \pi_{s\brk{\Iendo{\cJ} \cap \Iendo{\widehat{\cI}}}}}. \label{eq: proof sketch stage 2}
    \end{align}
     Next, using \pref{property:
       endo maximizer is sufficient} and \pref{property: decoupling}, we have
     \[
       \max_{\pi\in \PiInd\brk{\fullcleq}} d_2\rbr{s\brk{\cJ} \midsem \pi}
       =        \max_{\pi\in \PiInd\brk{\Ic}} d_2\rbr{s\brk{\cJ} \midsem \pi}
       = \max_{\pi\in \PiInd\brk{\Ic}} d_2\rbr{s\brk{\cJen} \midsem \pi}\cdot{}d_2\prn{s\brk{\cJex}}.
     \]
As a result, since
    $\pi_{s\brk{\Iendo{\cJ} \cap \Iendo{\widehat{\cI}}}}$ satisfies 
    \begin{align*}
        \max_{\pi\in \PiInd\brk{\Ic}} d_2\rbr{s\brk{\cJen} \midsem \pi} = d_2\rbr{s\brk{\cJen} \midsem \pi_{s\brk{\Iendo{\cJ} \cap \Iendo{\widehat{\cI}}}}}
    \end{align*}
    and it is an
    endogenous policy, we have 
    \begin{align*}
      \max_{\pi\in \PiInd\brk{\fullcleq}} d_2\rbr{s\brk{\cJ} \midsem \pi}  &= d_2\rbr{s\brk{\Iendo{\cJ}} \midsem \pi_{s\brk{\Iendo{\cJ} \cap \Iendo{\widehat{\cI}}}}}\cdot{}d_2\prn{s\brk{\cJex}}\\
      &= d_2\rbr{s\brk{\cJ} \midsem \pi_{s\brk{\Iendo{\cJ} \cap \Iendo{\widehat{\cI}}}}}= d_2\rbr{s\brk{\cJ} \midsem \pi_{s\brk{\cJ \cap \Iendo{\widehat{\cI}}}}},
    \end{align*}
    % \YEcomment{in the last relation we should have
    % $d_2\rbr{s\brk{\cJ} \midsem \pi_{s\brk{\cJ \cap
    % \Iendo{\widehat{\cI}}}}}$}
    % \YEcomment{the first relation holds only when optimzing over
    %   $\Ic$; only then $\pi$ is endogenous}
    where the second relation holds by~\pref{property: decoupling}, applicable since $\pi_{s\brk{\Iendo{\cJ} \cap \Iendo{\widehat{\cI}}}}$ is an endogenous policy, and the third relatin holds since $\Iendo{\cJ} \cap \widehat{\Iendo{\cI}} =\cJ \cap \widehat{\Iendo{\cI}}.$
    
    Thus, $\cIhaten$
    satisfies the constraint in~\pref{line: exact 1 endo
      selection}, and if
    $\cIhatex\neq\emptyset$, we can reduce the cardinality
    while keeping the constraints satisfied, so the minimum
    cardinality solution is endogenous.\loose
\end{proof}

\begin{algorithm}[t]
  \setstretch{1.1}
  \caption{$\ExactAlgName_h$: \pcalglong with Exact Occupancies}
  \label{alg: exact dynamics state refinement}
\begin{algorithmic}[1]
  \State \textbf{require:} Timestep $h\in\brk{H}$, policy covers $\cbr{\Psi\ind{t}}_{t=1}^{h-1}$ for steps $1,\ldots,h-1$.
    \State  \textbf{initialize:} $\cI\ind{h,h} \gets \emptyset$ and $\Psi\ind{h,h} \gets \emptyset$.
    % $\Psi\ind{h,h}=\emptyset,\ \cI\ind{h,h}= \emptyset$, for all $t\in [h-1], \mu\ind{t} = \unf\rbr{\Psi\ind{t}}$.
    \For{$t=h-1,\ldots, 1$}
  \item[]\algspace {\algcommentul{Phase I: Optimization}}\vspace{3pt}

    \State Let $\mu\ind{t}\ldef{}\unif(\Psi\ind{t})$.
      % \item[] \algspace {\algcomment{Find collection $\cbr{\pi\ind{t}_{s\brk*{\cJ}}: s\brk*{\cJ}\in \cS\brk*{\cJ}}$ such that %\\ [] \hspace{.9cm}
      %       $\pi\ind{t}_{s\brk*{\cJ}}\circt[t+1]\psi_{s\brk[\big]{\cI\ind{t+1,h}}}\ind{t+1,h}$ optimally reaches  $s\brk*{\cJ}$ from step $t$.}}
        \State\label{line: exact DP-SR state1}%
        % \begin{varwidth}[t]{\linewidth}  
        Find $\cItil\in \fullc_{\leq k}$ with minimal cardinality such that for all $\cJ\in \fullc_{\leq k}\rbr{\cI\ind{t+1,h}}$, $s\sbr{\cJ} \in \cS\sbr{\cJ}$,
        \[
          ~~~~~~~\max_{\pi\in \PiInd\brk{\fullc_{\leq k}}}  d_h\rbr{s\brk{\cJ} \midsem \mut\circt\pi\circt[t+1]\psi\ind{t+1,h}_{s\brk{\cI\ind{t+1,h}}} } = \max_{\pi\in \PiInd\brk{\cItil}}d_h\rbr{s\brk{\cJ} \midsem \mut\circt \pi \circt[t+1]\psi\ind{t+1,h}_{s\brk{\cI\ind{t+1,h}}} }.
      \]
    \item[] \algspace {\algcomment{Beginning from any state at layer $t$, %\\ [] \hspace{.9cm}
          $\pi\ind{t}_{s\brk*{\cJ}}\circt[t+1]\psi_{s\brk[\big]{\cI\ind{t+1,h}}}\ind{t+1,h}$ maximizes probability that $s_h\brk{\cJ}=s\brk*{\cJ}$.}}
      \State  \label{line: exact DP-SR state12} \multiline{For each factor set $\cJ\in \fullc_{\leq k}\rbr{\cI\ind{t+1,h}}$ and $s\brk{\cJ}\in\cS\brk{\cJ}$, let
        \[\pi_{s\brk{\cJ}}\in \argmax_{\pi\in \PiInd\brk{\cItil}}d_h\rbr{s\brk{\cJ} \midsem \mut\circt \pi \circt[t+1]\psi\ind{t+1,h}_{s\brk{\cI\ind{t+1,h}}}},\] and define
        $\Gamma\ind{t}\sbr{\cJ} \ldef \cbr{\pi_{s\brk{\cJ}} : s\brk{\cJ}\in \cS\brk{\cJ}}$.
      }
        % \end{varwidth}
        % \vspace{0.025cm}
        % \item[]\algspace {\algcommentbig{\textbf{Selection phase}}}
      \vspace{2pt}
      \item[]\algspace {\algcommentul{Phase II: Selection}}\vspace{3pt}
      % \item[] \algspace{\algcomment{The factor set $\cIhat$ is chosen such that $\Gamma\ind{t}\brk[\big]{\cIhat}$ has good coverage for all factors in $\fullc_{\leq k}\rbr{\cI\ind{t+1,h}}$.}}
        \State  \label{line: exact DP-SR state2}
        % \begin{varwidth}[t]{\linewidth}  
        \mbox{Find $\widehat{\cI}\in \fullc_{\leq k}\rbr{\cI\ind{t+1,h}}$ with minimal cardinality s.t. for all $\cJ\in \fullc_{\leq k}\rbr{\cI\ind{t+1,h}}$, $s\sbr{\cJ} \in \cS\sbr{\cJ}$,}
          % for $\pi\ind{t}_{s\brk{\cJ \cap \widehat{\cI}}}\in \Gamma\ind{t}\brk[\big]{\cJ \cap \widehat{\cI}}$ it holds that
          \[
            ~~~~~~~\max_{\pi\in \PiInd\brk{\fullc_{\leq k}}}  d_h\rbr{s\brk{\cJ} \midsem \mut\circt\pi\circt[t+1]\psi\ind{t+1,h}_{s\brk{\cI\ind{t+1,h}}} } = d_h\rbr{s\brk{\cJ} \midsem \mut\circt\pi\ind{t}_{s\brk{\cJ \cap \widehat{\cI}}}\circt[t+1]\psi\ind{t+1,h}_{s\brk{\cI\ind{t+1,h}}} }.
      \]
        % \end{varwidth}
        % \vspace{2pt}
    \item[]\algspace {\algcommentul{Policy composition}}\vspace{3pt}
      % \item[]
      \State  Let $\cI\ind{t,h} \gets \cIhat$, and for each $s\brk{\cI\ind{t,h}}\in\cS\brk{\cI\ind{t,h}}$ define
          \begin{align*}
          \psi\ind{t,h}_{s\brk{\cI\ind{t,h}}} \ldef \pi\ind{t}_{s\brk{\cI\ind{t,h}}} \circt \psi\ind{t+1,h}_{s\brk{\cI\ind{t+1,h}}}.
          \end{align*}
        \item[]\hfill\algcomment{Recall that $\pi\ind{t}_{s\brk{\cI\ind{t,h}}} \in \Gamma\ind{t}\brk*{\cI\ind{t,h}}$ and $\psi\ind{t+1,h}_{s\brk{\cI\ind{t+1,h}}}\in \Psi\ind{t+1,h}$.}
          \State Let $\Psi\ind{t,h} \gets \cbr{\psi\ind{t,h}_{s\brk{\cI\ind{t,h}}} : s\brk{\cI\ind{t,h}}\in \cS\brk{\cI\ind{t,h}}}$.
    \EndFor
    \State \textbf{return} $\Psi\ind{h}\ldef{}\Psi\ind{1,h}$
\end{algorithmic}
\end{algorithm}

\subsection{Warm-Up II: Finding an Endogenous Policy Cover with Exact Occupancies ($H\geq{}2$)}\label{sec: DP-SR exact}

% As our second warm-up,
\pref{alg: exact dynamics state refinement} describes \pcalgexact, which extends the $\ExactOneStep$ method to handle \kEMDPs with general horizon (rather than $H=2$), but still requires exact access to occupancy measures. When invoked with a layer $h$, $\pcalgexact_h$ takes as input a sequence of endogenous policy covers $\Psi\ind{1},\ldots,\Psi\ind{h-1}$ for layers $1,\ldots,h-1$ and uses them to compute an endogenous policy cover $\Psi\ind{h}$ for layer $h$. The algorithm constructs $\Psi\ind{h}$ in a \emph{backwards} fashion based on the dynamic programming principle. To describe the approach in detail, we use the notation of  \emph{$t\to{}h$ policy cover}.
\begin{definition}
  \label{def:ttoh_cover}
  For $h\in\brk{H}$ and $t<h$, a set of non-stationary policies $\Psi$ is said to be a ($\eps$-approximate) \emph{\ttoh policy cover} with respect to a roll-in policy $\mu\in\Pimix$ if for all $s\in\cS$,
  \[
    \max_{\psi\in\Psi}d_h\rbr{s\brk{\Ic}\midsem \mu\circt\psi } \geq{} \max_{\pi\in\PiNS} d_h\rbr{s\brk{\Ic}\midsem\mu\circt\pi}- \eps.
  \]
  If all policies in $\Psi$ are endogenous, we say that $\Psi$ is endogenous.
\end{definition}
$\pcalgexact_h$ performs a serious of ``backward'' steps $t=h-1,\ldots,1$. In each step $t$, the algorithm rolls in with the mixture policy $\mu\ind{t}\ldef\unif(\Psi\ind{t})$ and constructs a \ttoh policy cover $\Psi\ind{t,h}$ with respect to $\mu\ind{t}$.  $\Psi\ind{t,h}$ acts on an endogenous factor set $\cI\ind{t,h}$ (with $\cI\ind{t,h}\supseteq\cI\ind{t+1,h}\supseteq\cdots\supseteq \cI\ind{h,h}=\emptyset$), and is built from the next-step policy cover $\Psi\ind{t+1,h}$ via dynamic programming. In particular, the algorithm searches for a collection of endogenous ``one-step'' policies for choosing the action at time $t$ that---when carefully composed with the $(t+1)\to{}h$ policy cover $\Psi\ind{t+1,h}$---result in a $t\to{}h$ policy cover. The algorithm ensures that the factor set $\cI\ind{t,h}$ (upon which $\Psi\ind{t,h}$ acts) is endogenous using an optimization and selection phases analogous to those in $\pcalgonestep$.
% , which search for an \emph{endogenous} factor set  $\cI\ind{t,h}$
% (with $\cI\ind{t,h}\subseteq\cI\ind{t+1,h}\subseteq\cdots$)
% for a policy to select actions at layer $t$
%In particular, the algorithm 
% For each step $t=h-1,\ldots,1$, $\pcalgexact_h$ rolls in with $\mu\ind{t}\ldef\unif(\Psi\ind{t})$, and computes an endogenous factor set $\cI\ind{t,h}$ and 
% it inductively finds $\rbr{\cI\ind{t,h},\Psi\ind{t,h}}$ where $\cI\ind{t,h}$ is a set of factors and $\Psi\ind{t,h}$ is a set of $(t\rightarrow h)$ policies. Analogously to the guarantees of $\ExactOneStep$~(\pref{prop: exact case recovery of 1-step endo policy cover}),
% $\rbr{\cI\ind{t,h},\Psi\ind{t,h}}$ has the following guarantees:

In more detail, $\pcalgexact_h$ satisfies the following invariants for $1\leq{}t\leq{}h-1$.
\begin{enumerate}[label = $(\roman*)$, align=left, topsep=0pt]
    \item $\cI\ind{h,h}\subseteq\cdots\subseteq\cI\ind{t,h}\subseteq\cdots\subseteq \Ic$.\hfill (``state refinement'')
      % \item $\Psi\ind{t,h}$ is set of endogenous $(t\rightarrow h)$ policies that reaches $\cS\brk*{\cI\ind{t,h}}$ at time step $h$.
          \item The set $\Psi\ind{t,h}$ is an endogenous \ttoh  policy cover with respect to $\mu\ind{t}=\unif(\Psi\ind{t})$:
        % :with respect to $\mu\ind{t}=\unif(\Psi\ind{t})$, and acts on $\cI\ind{t,h}$. In particular, for all $s\brk{\Ic}\in\cS\brk{\Ic}$,
      \[
        d_h\prn[\big]{s\brk{\Ic}\midsem \mu\ind{t}\circt \psi\ind{t,h}_{s\brk{\cI\ind{t,h}} }} = \max_{\pi\in\PiNS}d_h\prn[\big]{s\brk{\Ic}\midsem\mu\ind{t}\circt\pi},\quad\forall{}s\brk{\Ic}\in\cS\brk{\Ic}.
      \]
    %   \[
    %     d_h\prn[\big]{s\brk{\Ic}\midsem \mu\ind{t}\circt \psi\ind{t,h}_{s\brk{\cI\ind{t,h}} }} = \max_{\pi\in\PiNS}d_h\prn[\big]{s\brk{\Ic}\midsem\mu\ind{t}\circt\pi}.
    %   \]}
      % the policy $\psi\ind{t,h}_{s\brk{\cI\ind{t,h}}}\in\Psi\ind{t,h}$ maximizes the probability that $s_h\brk{\Ic}=s\brk{\Ic}$.
    \end{enumerate}
    This implies that $\Psi\ind{h}\ldef{}\Psi\ind{1,h}$ is an endogenous policy cover for layer $h$ (\pref{def: approximate endognous policy cover}). In what follows we show how $\pcalgexact_h$ uses dynamic programming to satisfy these invariants.
% For any $t\in [h-1]$ $\ExactAlgName$ applies a similar optimization-selection approach as $\ExactOneStep$. Furthermore, since $\cI\ind{t,h}\supseteq \cI\ind{t+1,h}$ $\ExactAlgName$ always refines the state space it targets; $\cS\brk*{\cI\ind{t,h}}$ includes more state information relatively to  $\cI\ind{t+1,h}$.

% \dfcomment{mention state refinement.}

% \dfcomment{Add $t\to{}h$ policy cover definition.}

% $\ExactAlgName$ uses similar approach to $\ExactOneStep$ to keep the invariant valid at across time steps; it solves a 1-step problem and finds a pair $\rbr{\cI\ind{t,h},\Psi\ind{t,h}}$ which satisfies $(i),(ii)$ and $(iii)$. 
\paragraph{Dynamic programming.}
Consider step $t<h-1$, and suppose that $\rbr{\cI\ind{t+1,h},\Psi\ind{t+1,h}}$ satisfies invariants $(i)$ and $(ii)$. Because $\mu\ind{t+1}$ uniformly covers all states in layer $t+1$ (recall $\Psi\ind{1},\ldots,\Psi\ind{h-1}$ are policy covers),
% this implies that for any state $s\brk*{\Ic}\in\cS\brk{\Ic}$, 
the policy $\psi\ind{t+1,h}_{s\brk{\cI\ind{t+1,h}}}$ maximizes the probability that $s_h\brk{\Ic}=s\brk*{\Ic}$, starting from any state in layer $t+1$. Hence, the Bellman optimality principle implies that to find a $t\rightarrow h$ policy to maximize this probability, it suffices to use the policy $\pi\ind{t}\circt[t+1]\psi\ind{t+1,h}_{s\brk{\cI\ind{t+1,h}}}$, where $\pi\ind{t}$ solves the one-step problem:%solve the following one-step problem:
\begin{align}
    \pi\ind{t}\in \argmax_{\pi\in \PiInd\brk*{\Ic}} d_h\rbr{s\brk{\Ic} \midsem \mut\circt\pi\circt[t+1]\psi\ind{t+1,h}_{s\brk{\cI\ind{t+1,h}}}}. \label{eq: main paper dp sr relation 1}
\end{align}

% then take as our policy.}
% \begin{align}
%     \pi\ind{t}\in \argmax_{\pi\in \PiInd\brk*{\Ic}} d_h\rbr{s\brk{\Ic} \midsem \mut\circt\pi\circt[t+1]\psi\ind{t+1,h}_{s\brk{\cI\ind{t+1,h}}}}, \label{eq: main paper dp sr relation 1}
% \end{align}
% then take $\pi\ind{t}\circt[t+1]\psi\ind{t+1,h}_{s\brk{\cI\ind{t+1,h}}}$ as our policy.}

At first glance, it is not apparent whether this observation is useful, because the endogenous factor set $\Ic$ is not known to the learner, which prevents one from directly solving the optimization problem in \eqref{eq: main paper dp sr relation 1}. Fortunately, we can tackle this problem using a generalization of the optimization-selection approach of $\ExactOneStep$. First, in the optimization phase (\pref{line: exact DP-SR state1} and \pref{line: exact DP-SR state12}), we compute a collection of one-step policy covers $(\Gamma\ind{t}\brk{\cJ})_{\cJ\in\fullcleq(\cI\ind{t+1,h})}$, where $\Gamma\ind{t}\brk{\cJ}$ consists of the policies that solve \eqref{eq: main paper dp sr relation 1} with $\Ic$ replaced by $\cJ$, for all possible choices of state in $s\brk{\cJ}\in\cS\brk{\cJ}$. Then, in the selection phase (\pref{line: exact DP-SR state2}), we find a single factor $\cI\ind{t,h}\supseteq\cI\ind{t+1,h}$ such that $\Gamma\ind{t}\brk{\cI\ind{t,h}}$ provides good coverage (in the sense of \eqref{eq: main paper dp sr relation 1}) for all factor sets $\cJ\in\fullcleq(\cI\ind{t+1,h})$ simultaneously. Both steps ensure endogeneity by penalizing by cardinality in the same fashion as \pcalgonestep. The success of this  approach critically relies on the assumption that the preceding policy covers $\Psi\ind{1},\ldots,\Psi\ind{h-1}$ are endogenous, which ensures that the occupancy measures induced by $\mu\ind{1},\ldots,\mu\ind{h-1}$ factorize (due to independence of the endogenous and exogenous state factors). To summarize:
% .to tackle this problem. Indeed, at the $t^{\th}$ time step $\ExactAlgName$ performs similar optimization-selection steps.  The only modification it makes is that it optimizes over $\pi$ at the $t^{\th}$ time step, while sampling a state at the $t^{\th}$ time step by $\mut$ and following a policy $\psi\ind{t+1,h}_{s\brk{\cI\ind{t+1,h}}}\in \Psi\ind{t+1,h}$ from the $t+1$ time step.
% \dfcomment{Explain why $\mu\ind{t}$ cover property is important.}
% \dfcomment{Explain why endogeneity is important}

\begin{proposition}
  If $\Psi\ind{1},\ldots,\Psi\ind{h-1}$ are endogenous policy covers for layers $1,\ldots,h-1$, then the set $\Psi\ind{h}$ returned by $\pcalgonestep_h$ is an endogenous policy cover for layer $h$, and has $\abs{\Psi\ind{h}}\leq{}S^{k}$.\loose
\end{proposition}
We do not prove this result directly, and instead refer the reader to the proof of \pref{thm: sample complexity of state refinment}, which proves the sample-based version of the result using the same reasoning.
% We do not prove
% After applying the optimization and selection steps of $\ExactAlgName$ at the $t^{\th}$ iteration it can be shown that $\rbr{\cI\ind{t,h},\Psi\ind{t,h}}$ satisfies the guarantees of $(i),(ii)$ and $(iii)$ via similar arguments as in~\pref{sec: endogenous policy cover 1 step exact}; using the fact that $\cbr{\Psi\ind{t}}_{t=1}^{h-1}$ and $\Psi\ind{t+1,h}$ are endogenous sets of policies the key properties~\pref{property: decoupling} and~\pref{property: endo maximizer is sufficient} on which the proof of $\ExactOneStep$ relies upon still hold.
% \dfcomment{Refer forward to proof in appendix.}

% \subsection{Learning a Near-Optimal Endogenous Policy Cover: $\AlgName^{\eps,\delta}_h$}
\subsection{$\AlgName$: Overview and Main Result}
\label{sec: learning endogenous policy cover}

The full version of the \pcalg algorithm ($\pcalg_h^{\veps,\delta}$) is given in \pref{alg: dp-sr paper version} (deferred to \pref{app: learnin eps endogenous policy cover} due to space constraints). The algorithm follows the same template as \pcalgexact: For each $h\in\brk{H}$, given policy covers $\Psi\ind{1},\ldots,\Psi\ind{h-1}$, the algorithm builds a policy cover $\Psi\ind{h}$ for layer $h$ in a backwards fashion using dynamic programming. There are two differences from the exact algorithm. First, since the MDP is unknown, the algorithm estimates the relevant occupancy measures for each backwards step using Monte Carlo rollouts. Second, the optimization and selection phases from \pcalgexact are replaced by error-tolerant variants given by subroutines $\EndoPolicyMaximizer$ and $\FactorDetectionAlg$ (\pref{alg: approximate 1-step endogenous policy maximizer} in \pref{app: near optimal endogenous policy} and \pref{alg: sufficient exploration check} in \pref{app: no false detection of factors}, respectively).

Briefly, the $\EndoPolicyMaximizer$ and $\FactorDetectionAlg$ subroutines are based on approximate versions of the constraints used in the optimization and selection phase for \ExactAlgName (\pref{line: exact DP-SR state1} and~\pref{line: exact DP-SR state2} of \pref{alg: exact dynamics state refinement}), but ensuring endogeneity of the resulting factors is more challenging due to approximation errors, and it no longer suffices to simply search for the factor set with minimum cardinality. Instead, we search for factor sets that satisfy approximate versions of \pref{line: exact DP-SR state1} and~\pref{line: exact DP-SR state2} with an \emph{additive regularization term} based on cardinality. We show that as long as this penalty is carefully chosen as a function of the statistical error in the occupancy estimates, the resulting factor sets will be endogenous with high probability. 
% We provide a general template for designing error-tolerant algorithms that search for endogenous factors using this approach in \pref{app:abstract_search}.
  
The main guarantee for \pref{alg: dp-sr paper version} is as follows.
\begin{restatable}[Sample complexity of \AlgName]{theorem}{DPSRSampleComplexity}\label{thm: sample complexity of state refinment}
  Suppose that $\AlgName^{\eps,\delta}_h$ is invoked with $\cbr{\Psi\ind{t}}_{t=1}^{h-1}$, where each $\Psi\ind{t}$ is an endogenous, $\eta/2$-approximate policy cover for layer $t$. Then with probability at least $1-\delta$, the set $\Psi\ind{h}$ returned by $\AlgName^{\eps,\delta}_h$ is an endogenous $\eps$-approximate policy cover for layer $h$, and has $\abr{\Psi\ind{h}}\leq S^k$. The algorithm uses at most
  $
  \bigoh\rbr{A S^{4k} H^2 k^3 \log\rbr{\frac{dSAH}{\delta}}\cdot\eps^{-2}}
    $
    episodes.\loose
% trajectories $(1)$ $\Psi\ind{h}$ is an endogenous $\eps$-approximate policy cover over time step $h$, and $(2)$ $\abr{\Psi\ind{h}}\leq S^k$.
% \begin{enumerate}
%     \item $\Psi\ind{h}$ is an $\eps$-endogeneous policy cover over time step $h$.
%     \item $\cI\ind{h}$ contains only endogenous factors.
% \end{enumerate}
  \end{restatable}
By iterating the process $\Psi\ind{h}\gets\AlgName^{\eta/2,\delta}_h(\crl*{\Psi\ind{t}}_{t=1}^{h-1})$, we obtain a policy cover for every layer.

\section{Main Result: Sample-Efficient RL in the Presence of Exogenous Information}
% \label{sec: reward free exploration of endo state}
\label{sec:main}

% \subsection{RL in the Presence of Exogenous Information: $\RLwithExo$}\label{sec: full algorithm RL in the presence of exo info}

\begin{algorithm}[tp]
  \caption{$\RLwithExo$: RL in the Presence of Exogenous Information}
  \label{alg: main}
% \label{alg: controllability detection k factors }
\begin{algorithmic}
    \State \textbf{require:} precision parameter $\epsilon>0$, reachability parameter $\eta>0$, failure
    probability $\delta\in (0,1)$.
    \State \textbf{initialize:} $\Psi\ind{1}=\emptyset$.
    \For{$h=2,3,\cdots, H$}
        \State
        $\Psi\ind{h}\gets\AlgName_h^{\eta/2,\delta}\prn[\big]{\cbr{\Psi\ind{t}}_{t=1}^{h-1}}$. \hfill\algcomment{Learn policy cover via \pcalg (\pref{alg: dp-sr paper version} in \pref{app: learnin eps endogenous policy cover}).}
        %   \mycomment{Note that $\cI\ind{h-1}\subseteq\cI\ind{h}$.}
    \EndFor
    \State  $\widehat{\pi}  \gets  \PSDPE^{\eps,\delta}\prn[\big]{\cbr{\Psi\ind{h}}_{h=1}^H}$. \hfill \algcomment{Apply \PSDPE (\pref{alg: psdp with exogenous noise} in \pref{app: psdp in the presence of exogenous information}) to optimize rewards.}
    \State \textbf{return}  $\widehat{\pi}$
\end{algorithmic}
\end{algorithm}

In this section we provide our main algorithm, \RLwithExo (\pref{alg: main}).
$\RLwithExo$ first applies $\AlgName$ iteratively to learn an endogenous, $\eta/2$-approximate policy cover for each layer, then applies a novel variant of the classical Policy Search by Dynamic Programming method of \citep{bagnell2004policy} (\PSDPE), which uses the covers to optimize rewards; the original \PSDP method cannot be applied to the \framework setting as-is due to subtle statistical issues (cf. \pref{app: psdp in the presence of exogenous information} for background).
% , resulting in an $\eps$-optimal policy.
The main guarantee for \mainalg is as follows; see \pref{app: RL in the presence of exogenous information} for a proof and overview of analysis techniques.

%\begin{theorem}[Sample Complexity of $\RLwithExo$]
%
\begin{restatable}[Sample complexity of $\RLwithExo$]{theorem}{rlwithexo}
  \label{thm:main}
  \mainalg, when invoked with parameter, $\eps\in (0,1)$ and $\delta\in (0,1)$, returns an $\eps$-optimal policy with probability at least $1-\delta$, and does so using at most\arxiv{$$\bigoh\rbr{AS^{3k}H^2(S^k+H^2)k^3 \log\rbr{\frac{dSAH}{\delta}}\cdot{}\rbr{\eps^{-2}+\eta^{-2}}}$$}\colt{$\bigoh\rbr{AS^{3k}H^2(S^k+H^2)k^3 \log\rbr{\frac{dSAH}{\delta}}\cdot{}\rbr{\eps^{-2}+\eta^{-2}}}$}
episodes.
\end{restatable}
Recall that $S^{k}=\abs{\cS\brk{\Ic}}$ may thought of as the cardinality of the endogenous state space so---up to polynomial factors, logarithmic dependence on $d$, and dependence on the reachability parameter $\eta$, the sample complexity of \mainalg matches the optimal sample complexity when $\Ic$ is known in advance.\loose 

% The quantity $\S^k$ represents the size of endogenous state space, and any general MDP with $S=\S^k$ number of states can be encoded in the endogenous dynamics. For such a general MDP, it is known that $\poly(\abs{\cS},\abs{\cA},H,\log(1/\delta)\cdot\eps^{-2}$ trajectories are necessary to find an $\eps$-optimal policy with probability greater than $1-\delta$ (e.g.,~\citet{jin2018q}).

\begin{remark}[Computational Complexity of~\RLwithExo]
The runtime for \mainalg scales with $\sum_{k'=0}^k{d\choose k'} = \Theta(d^{k})$ due to brute force enumeration over factors sets of size at most $k$. While this improves over the $S^{d}$ runtime required to run a tabular RL algorithm over the full state space, an interesting question that remains is whether the runtime can be improved to $\bigoh(d^{c})$ for some constant $c$ independent of $k$.
\end{remark}

\section{Related Work}
\label{sec:related}
In this section we highlight additional related work not already
covered by our discussion.

\paragraph{Reinforcement learning with exogenous information.}
The \framework setting is a special case of the Exogenous Block MDP (EX-BMDP)
setting introduced by \cite{efroni2021provable}, who initiated the
study of sample-efficient reinforcement learning with temporally
correlated exogenous information. 
In particular, one can view the
\framework as an EX-BMDP with $\cS$ as the observation space and
$\cS\brk{\Ic}$ as the latent state space, and with the set $\Phi\ldef{}\crl*{s\mapsto{}s\brk*{\cI}\mid{}\abs{\cI}\leq{}k}$ as the
class of decoders. \cite{efroni2021provable} provide an EX-BMDP
algorithm whose sample
complexity scales with the size of the latent state space and with
$\log\abs{\Phi}$, which translates to $\poly(S^{k},\log(d))$ sample
complexity for the \framework setting, but the algorithm requires that
the endogenous state space has deterministic transitions and initial state. The motivation for the present work was to
take a step back and provide a simplified testbed in which to study
the problem of learning with stochastic transitions, as well as other
refined issues (e.g., minimax rates). Also related to this line of
research is \cite{efroni2021sparsity}, which considers a linear
control setting with exogenous observations. Unlike our work,
\cite{efroni2021sparsity} assumes that the inherent system noise
induces sufficient exploration, and hence does not address the
exploration problem.

Empirical works that aim to filter exogenous noise in deep RL include~\citet{pathak2017curiosity,zhang2020learning,gelada2019deepmdp}, but these methods do not come with theoretical guarantees. 

% study the problem of learning with stochastic
% transitions in a simplified setting, as well 
% The motivation for the present work

% \dfcomment{Somewhere we should explicitly spell out that this is a special case of the ExBMDP setting in Efroni et al.}

\paragraph{Tabular reinforcement learning.}
As discussed earlier, existing approaches to tabular reinforcement
learning
\citep{azar2017minimax,jin2018q,zanette2019tighter,kaufmann2021adaptive}
incur $\bigom(S^{d})$ sample complexity if applied to the \framework
setting naively. One can improve this sample complexity to
$\poly(S^{k}, d^{k} ,A, H)$ using a simple reduction. This falls short of the $\poly(S^{k},A, H, \log(d))$ sample complexity our algorithms obtain, we sketch the reduction for completeness.
\begin{itemize}
\item For each $\cI\subseteq\brk{d}$ with $\abs{\cI}\leq{}k$, run any
optimal tabular RL algorithm with precision parameter $\eps$ over the state space
$\cS\brk{\cI}$, and let $\pi_{\cI}$ be the resulting policy.
\item Evaluate each policy $\pi_{\cI}$ to precision $\eps$ using Monte-Carlo rollouts, and
  take the best one.
\end{itemize}
The first phase has $\poly(S^{k},A,H)$ sample complexity for
each set $\cI$, and there are at most ${d\choose{}k}=\bigoh(d^{k})$ subsets. The
algorithm that runs on $\cS\brk{\Ic}$ will succeed in finding an
$\eps$-optimal policy with high probability, so the policy returned in
the second phase will be at least $2\eps$-optimal.

% \dfcomment{Somewhere we should mention that getting $d^{k}$ sample complexity is easy.}

% % \dfcomment{Mention the Exo LQR paper?}

% \dfcomment{Discuss parallels to sample complexity bounds in
%   high-dimensional stats.}

\paragraph{Factored Markov decision processes.}%
\newcommand{\pt}{\mathrm{pt}}%
The \framework setting is related to the Factored MDP model~\citep{kearns1999efficient}. Factored MDPs assume a factored state space whose transition dynamics obey the following structure:
\begin{equation*}
    \forall s, s' \in \Scal^d, a \in \Acal, \qquad T(s' \mid s, a) = \prod_{i=1}^d T_i(s'[i] \mid s[\pt(i)], a),
\end{equation*}
where $\pt: [d] \rightarrow 2^d$ is a \emph{parent function} and $T_i: \Scal^{|\pt(i)|} \times \Acal \rightarrow \Delta(\Scal)$ is the transition distribution of the $i$th factor. Many algorithms have been proposed for Factored MDPs, including for the setting where the parent function is unknown~\citep{strehl2009reinforcement,diuk2009adaptive,hallak2015off,guo2017sample,rosenberg2020oracle,misra2021provable}. These algorithms assume that the parent factor size is bounded, i.e., $|\pt(i)| \le \kappa$ for all $i \in [d]$, and their sample complexity typically scales with $\bigoh(|\Scal|^{c\kappa})$ for a numerical constant $c$. The \framework setting cannot be solved using off-the-shelf factored MDP algorithms for two reasons. First, we do not assume that each factor evolves independently of other factors given the previous state and action. Second, the size of the parent set for an exogenous factor can be as large as $d-k$. Therefore, even if factors were evolving independently, applying off-the-shelf Factored MDPs algorithms would lead to exponential sample in $d$ sample complexity. \loose

\section{Conclusion}
\label{sec:conclusion}

We have introduced the \framework setting and provided \mainalg, the first algorithm for sample-efficient reinforcement learning in stochastic systems with high-dimensional, exogenous information. Going forward, we believe that the \framework setting will serve as a useful testbed to understand refined aspects of learning with exogenous information. Natural questions we hope to see addressed include:
\begin{itemize}
\item \emph{Minimax rates.} While our results provide polynomial sample complexity, it remains to understand the precise minimax rate for the \framework as a function on $S^{k}$, $H$, and so on. Additionally, either removing the dependence on the reachability parameter or establishing a lower bound remains  for its necessity is an issue which deserves further investigation. 
% is an issue demands further algorithmic improvements.
% {\YE: commented out, don't think Azar's technique can assist here.}, along the lines of \cite{azar2017minimax}.
\item \emph{Computation.} Both $\RLwithExo$ and $\AlgName$ rely on brute force enumeration over subsets, which results in $\bigom(d^k)$ runtime. While this provides an improvement over naive tabular RL, it remains to see whether it is possible to develop an algorithm with runtime $\bigoh(d^{c})$, where $c>0$ is a constant independent of $k$.
\item \emph{Regret.} Naively lifting our $\eps$-PAC results to regret results in $T^{2/3}$-type dependence on the time horizon $T$. Developing algorithms with $\sqrt{T}$-type regret will require new techniques.
\item \emph{Parameter-free algorithms.} The $\AlgName$ algorithm requires an upper bound on $\abr{\Ic}$ and a lower bound on $\eta$. It is relatively straightforward to remove access to these quantities when the value of the optimal policy ($\max_{\pi} J(\pi)$) is known, by an application of the doubling trick. However, developing truly parameter-free algorithms is an interesting direction.
\end{itemize}
Finally, the problem of learning in the \framework model is related to the notion of out-of-distribution generalization and learning in the presence of acausal features~\citep{peters2016causal,arjovsky2019invariant,kim2019learning,wald2021calibration}. It would be interesting to explore these connections in more detail. Beyond these questions, we hope that our techniques will find further use beyond the tabular setting.
% (e.g., in the rich observation setting of \cite{efroni2021provable}).
% }

% Our work introduces several interesting future research directions Additionally, in this work, we focused on learning an $\eps$-optimal policy. We believe that constructing an algorithm with $o(T^{2/3})$ regret guarantee for the \kEMDP\ setting is of interest and requires developing new techniques.
% Oracle...

% Lastly, oraclizing our algorithms to the rich-observation setting is of key importance to address a more general setting of RL with exogenous information,  prevalent in practice.
% \dfcomment{We can mention other technical challenges like getting the right dependence on $S$ and $H$.}

% \dfcomment{$d^{c}$ runtime question.}

%%% Local Variables:
%%% mode: latex
%%% TeX-master: "paper"
%%% End:

% \bibliographystyle{apalike}
\bibliography{paper}

\clearpage

\appendix

\renewcommand{\contentsname}{Contents of Appendix}
\tableofcontents
%\newpage
%\setcounter{tocdepth}{3}
\addtocontents{toc}{\protect\setcounter{tocdepth}{2}}

\clearpage

\section*{Organization and Notation}
\label{app:organization}
The appendix contains three parts,~\pref{part:preli},~\pref{part:ommited_parts}, and ~\pref{part:main_res}.

\paragraph{\preftitle{part:preli}: Preliminaries.} 
In~\pref{part:preli} we provide basic technical results used in our analysis. \pref{app:supporting} contains technical lemmas for reinforcement learning (\pref{app:rl_preli}), concentration inequalities (\pref{app: cocentration results}), and basic analysis tools (\pref{app:analysis_preli}). In~\pref{app: structural results on class}, we provide a simple, yet useful result which shows that the collection $\fullc_{\leq k}\rbr{\cI}$ is a $\pi$-system for any factor set $\cI$ with $\abr{\cI}\leq k$. 

In~\pref{app:structural} we present structural results for the \framework model. We begin by establishing a negative result (\pref{app:bellman_rank}) which shows that the Bellman rank~\citep{jiang2017contextual} of for the \framework model scales with the number of exogenous factors. In~\pref{app:structural_occupancy} and \pref{app:structural_reward}, we prove key structural results for the \framework model, including a decoupling property (\pref{lem: decoupling of future state dis})  and restriction lemma (\pref{lem: invariance of reduced policy cover}) for occupancy measures, a restriction lemma for endogenous rewards (\pref{lem:restriction_endo_rewards}), and a performance difference lemma for endogenous policies (\pref{lem:pd_for_endo_policies}). %All of these results are essential for our later analysis. 

In \pref{app:abstract_search}, we present an algorithmic template, \AMFD, which forms the basis for the subroutines in \AlgName.

Notation used throughout the main paper and appendix is collected in \pref{tab:notation}.

% Later, following the algorithmic scheme of \AMFD we develop algorithms that find near-optimal solutions that provably depend only on the endogenous factors. With this we conclude~\pref{part:preli}.

\paragraph{\preftitle{part:ommited_parts}: Omitted subroutines.} 
In~\pref{part:ommited_parts}, we describe and analyze subroutines used by $\AlgName$ and $\RLwithExo$.
% omitted algorithms from the main paper which are being used by~$\AlgName$ and $\RLwithExo$.
\pref{app: near optimal endogenous policy} presents and analyzes the $\EndoPolicyMaximizer$ subroutine used in $\AlgName$ and  $\PSDPE$. \pref{app: no false detection of factors} we presents and analyzes the $\FactorDetectionAlg$ subroutine used in $\AlgName$. Finally, \pref{app: psdp in the presence of exogenous information} presents and analyzes $\PSDPE$ algorithm, which is used by $\RLwithExo$.
% after it acquires an endogenous $\eta/2$-approximate policy cover.

\paragraph{\preftitle{part:main_res}: Additional details and proofs for main results.} 
In~\pref{part:main_res}, we present our main results and their proofs. In~\pref{app:OSSR_alg}, we present and analyze the full version of the $\AlgName$ algorithm, and in \pref{app: RL in the presence of exogenous information}, we combine the results for $\AlgName$ and $\PSDPE$ to establish the main sample complexity bound for $\RLwithExo$.

\begin{table*}[h]
    \centering
    \begin{tabular}{l|l}
    \hline
        \textbf{Notation} & \textbf{Meaning} \\
        \hline 
        $\cI$ & an ordered set of factors (a set of distinct elements from $[d]$).\\
        $\fullc_{\leq k}\rbr{\cI}$ & $\cbr{\cJ\subseteq [d] : \cJ\supseteq \cI, \abr{\cJ}\leq k }$.\\
        $\fullc_{k}\rbr{\cI}$ & $\cbr{\cJ\subseteq [d] : \cJ\supseteq \cI, \abr{\cJ}= k }$.\\
        $\fullc_{\leq k}$ & $\cbr{\cJ\subseteq [d] : \abr{\cJ}\leq k }$, or equivalently, $\fullc_{\leq k}=\fullc_{\leq k}\rbr{\emptyset}.$  \\
        $\fullc_k$ & $\cbr{\cJ\subseteq [d] : \abr{\cJ}= k }$, or equivalently, $\fullc_{k}=\fullc_{k}\rbr{\emptyset}$. \\
        $\PiInd\brk*{\cI}$ & the set of policies that depend only on the factors specified in $\cI$.\\
        $\PiInd\brk*{\fullc}$ & the union of the set of policies $\cup_{\cI\in \fullc} \Pi\brk*{\cI}$.\\
        $\Ic$ & the set of endogenous factors.\\
        $\Inc$ & the set of exogenous factors.\\
        $\Scal\sbr{\cI}$ & the set of states induced by the factors in $\cI$.\\
        $s\sbr{\cI}$ & the state $s$ restricted to the set of factors $\cI$.\\ 
        $V^\pi_1$ & value of a policy $\pi$ measured with respect to an initial distribution.\\
        $V^\pi_h(s)$ & value of a policy $\pi$ measured from state $s$ at timestep $h$\\
        $V_{t,h}$ & $V_{t,h}\rbr{\pi} \ldef \EE_\pi\sbr{\sum_{t'=t}^h r_t}$.\\
        $Q^\pi_h(s,a)$ & $Q$-function for a policy $\pi$ measured from state $s$ at timestep $h$.\\
        $d_h\rbr{s\brk{\Ical};\pi}$ & shorthand for $\PP^\pi(s_h\brk{\cI} = s\brk{\cI})$.\\
        $d_h(s\brk{\cI} \mid s_t\brk{\cI'} = s\brk{\cI'};\pi)$ & shorthand for $\PP^\pi(s_h\brk{\cI} = s\brk{\cI} \mid s_t\brk{\cI'} = s\brk{\cI'})$.\\
        $\pi_1 \circt[t] \pi_2$ & Policy that executes $\pi_1$ until step $t-1$ and executes $\pi_2$ from then on.\\
        % $\pi\ind{t,h}$ & A non-stationary policy that is applied from timestep $t$ to $h$.\\
        $\Iendo{\cI}$ & For a set of factors $\cI$, $\Iendo{\cI} \ldef \cI\cap \Ic$.\\
        $\Iexo{\cI}$ & For a set of factors $\cI$,  $\Iexo{\cI} \ldef \cI\cap \Inc$.\\
        \hline
    \end{tabular}
    \caption{Summary of notation.}
    \label{tab:notation}
\end{table*}

%%% Local Variables:
%%% mode: latex
%%% TeX-master: "paper"
%%% End:

\newpage

\clearpage

\part{Preliminaries}\label{part:preli}

\section{Supporting Lemmas}
\label{app:supporting}

\subsection{Reinforcement Learning}\label{app:rl_preli}

\begin{lemma}[Performance difference lemma (\cite{kakade2002approximately}, Lemma 6.1)] \label{lem: value difference}
Consider a fixed MDP $\cM = (\Scal, \Acal, T, R,H,\mu)$. For any pair of policies $\pi,\pi'\in \PiIndNS$,
\begin{align*}
  \Jpi- \Jpi[\pi']= \EE_{\pi}\sbr{\sum_{t=1}^H Q_{t}^{\pi'}(s_t,\pi_t(s_t)) - Q_t^{\pi'}(s_t,\pi'_t(s_t))}.
\end{align*}
\end{lemma}

\begin{lemma}[Density ratio bound for policy cover]\label{lem: policy cover bound importance sampling ratio}
Let $\Psi$ be an endogenous $\eps$-approximate policy cover for timestep $t$ and $\mu\ind{t} \ldef \unf\rbr{\Psi}$. Then, for any $s\brk{\Ic}\in \Scal\sbr{\Ic}$ such that $\max_{\pi\in \PiIndNS\brk{\Ic}}d_t( s\brk{\Ic} \midsem \pi)\geq 2\eps$, it holds that
\begin{align*}
    \max_{\pi\in \PiIndNS\brk{\Ic}}\frac{d_t( s\brk{\Ic} \midsem \pi)}{d_t( s\brk{\Ic} ; \mu\ind{t})}\leq 2\S^k.
\end{align*}
\end{lemma}
\begin{proof}[\pfref{lem: policy cover bound importance sampling ratio}]
Fix $s\brk{\Ic}\in \Scal\sbr{\Ic}$. Since $\Psi$ is an endogenous $\eps$-approximate policy cover, there exists $\psi_{s\brk{\Ic}}\in \Psi$ such that 
\begin{align}
    \max_{\pi\in \PiIndNS\brk{\Ic}} d_t( s\brk{\Ic} ; \pi) \leq d_t( s\brk{\Ic} ; \psi_{s\brk{\Ic}}) +\eps. \label{eq: policy cover definition proof lemma}
\end{align}
Thus, we have that 
\begin{align*}
    \max_{\pi\in \PiIndNS\brk{\Ic}} \frac{d_t( s\brk{\Ic} ; \pi)}{d_t( s\brk{\Ic} ; \mu\ind{t})} 
    &\stackrel{\mathrm{(a)}}{=} \S^k  \max_{\pi\in \PiIndNS\brk{\Ic}} \frac{d_t( s\brk{\Ic} ; \pi)}{\sum_{s'\brk{\Ic}\in \Scal\sbr{\Ic}}d_t( s\brk{\Ic} ; \psi_{s'\brk{\Ic}})} \\
    &\stackrel{\mathrm{(b)}}{\leq} \S^k \max_{\pi\in \PiIndNS\brk{\Ic}}\frac{d_t( s\brk{\Ic} ; \pi)}{d_t( s\brk{\Ic} ; \pi_{s\brk{\Ic}})}\\
    &\stackrel{\mathrm{(c)}}{\leq} \S^k\frac{\max_{\pi\in \PiIndNS\brk{\Ic}} d_t( s\brk{\Ic} ; \pi)}{\max_{\pi\in \PiIndNS\brk{\Ic}} d_t( s\brk{\Ic} ; \pi) -\eps}. 
\end{align*}
Here, $\mathrm{(a)}$ holds because $\mu\ind{t} = \unf\rbr{\Psi}$,
$\mathrm{(b)}$ holds because $d_t( s\brk{\Ic} ; \psi_{s'\brk{\Ic}})\geq
0$ for all $\psi_{s'\brk{\Ic}}\in \Psi$, and $\mathrm{(c)}$ holds by
\eqref{eq: policy cover definition proof lemma}. Finally, since
$x/(x-\eps)\leq 2$ for $x\geq 2\eps$, we conclude the proof.
\end{proof}

\subsection{Probability}\label{app: cocentration results}

% See also~\cite{sridharan2002gentle} for the following version of Bernstein's inequality.
\begin{lemma}[Bernstein’s Inequality~(e.g., \cite{boucheron2013concentration})]\label{lem: bernstein's inequality}
Let $X_1,..,X_N$ be a sequence of i.i.d. random variables with $\EE\sbr{X_i}=\mu$,
$\EE\brk[\big]{\rbr{X_i-\mu}^2}=\sigma^2$, and $\abr{X_i-\mu}\leq C$
almost surely. Then for all $\delta\in (0,1)$,
\begin{align*}
    \PP\rbr{\abr{\frac{1}{N}\sum_{i=1}^N \rbr{X_i - \mu}} \geq \sqrt{ \frac{2\sigma^2 \log\rbr{\frac{2}{\delta}}}{N}} +\frac{C\log\rbr{\frac{2}{\delta}}}{N}} \leq \delta.
\end{align*}
% \begin{align*}
%     \PP\rbr{\abr{\frac{1}{N}\sum_{i=1}^N \rbr{X_i - \mu}} \geq \eps } \leq 2\exp\rbr{-\frac{N\eps^2}{2\sigma^2 + 2C\eps/3}}.
% \end{align*}
\end{lemma}

\begin{lemma}[Union bound for sequences]\label{lem: sequentual good event}
Let $\cbr{\GS_t}_{t=1}^h$ be a sequence of events. If $\PP(\GS_t \mid
\cap_{t'=1}^{t-1} \GS_{t'})\geq 1-\delta$ for all $t\in [h]$, then
$
\PP(\cap_{t=1}^{h} \GS_t)\geq 1-h\delta.
$
\end{lemma}
\begin{proof}[\pfref{lem: sequentual good event}]
We prove the claim by induction. The base case $h=1$ holds by
assumption. Now, suppose the claim holds for some $h'\leq h$:
\begin{align*}
    \PP(\cap_{t=1}^{h'} \GS_t)\geq 1-h'\delta.
\end{align*}
% We prove the induction step and prove it also holds for $h'+1$.
By Bayes' rule, we have that
\begin{align*}
  &\PP(\cap_{t=1}^{h'+1} \GS_t) \\
        &=  \PP(G_{h'+1} \mid \cap_{t=1}^{h'} \GS_t)  \PP(\cap_{t=1}^{h'} \GS_t) \\
        &\stackrel{\mathrm{(a)}}{\geq}  \PP(G_{h'+1} \mid \cap_{t=1}^{h'} \GS_t) \rbr{1-h'\delta} \\
        &\stackrel{\mathrm{(b)}}{\geq}  (1-\delta) \rbr{1-h'\delta}\\
        &\geq 1-(h'+1)\delta,
\end{align*}
where $\mathrm{(a)}$ holds by the induction hypothesis and $\mathrm{(b)}$ holds by  assumption of the lemma.
This proves the induction step and concludes the proof.
\end{proof}

\subsubsection{Concentration for Occupancy Measures}
% \label{app: cocentration results}

\begin{definition}[$\epsilon$-approximate occupancy measure collection] \label{def: approximate set of occupancy measures}
  Let 
$
\EmpOccupancy = \crl*{\dhat_h\rbr{\cdot \midsem \pi} \mid{} \pi\in\Pi},
$
be a set of occupancy measures for timestep $h$. We say that $\EmpOccupancy$ is $\epsilon$-approximate with respect to $\rbr{\Pi, \fullc,h}$ if for all $\pi\in \Pi, \cI\in \fullc$ and $s\sbr{\cI}\in \Scal\sbr{\cI}$ it holds that
\begin{align*}
    \abr{\widehat{d}_{h}\rbr{ s_h\sbr{\cI} = s\sbr{\cI} \midsem\pi} - d_{h}\rbr{s_h\sbr{\cI} = s\sbr{\cI} \midsem \pi}} \leq \eps.
\end{align*}
\end{definition}

% \SampleComplexityEstimator*

In the following lemma, we bound the sample complexity required to compute a set of $\eps$-approximate occupancy measures with respect to $(\mu\circ\Pi\circ \Psi, \fullc, h)$, where $\mu$ is a fixed policy, $\Pi$ is a set of $1$-step policies, and $\Psi$ is a set of non-stationary policies. The proof follows from a simple application of Bernstein's inequality and a union bound.

\begin{restatable}[Sample complexity for $\eps$-approximate occupancy measures]{lemma}{SampleComplexityEstimator}\label{lem: sample complexity of close model}
Let $t,h\in \mathbb{N}$ with $t\leq h$ be given. Fix a mixture policy
$\mu\in \Pimix$, a collection $\Gamma \subseteq\Pi$ of 1-step
policies, a set $\Psi\subseteq \PiIndNS$, and a collection of factors
$\fullc$. Assume the following bounds hold:
\begin{enumerate}
    \item $\abr{\Psi}\leq \S^k$.
    \item $\abr{\Gamma}\leq O\rbr{d^k A^{\S^k}}$.
    \item $\abr{\GenericCollection}\leq O\rbr{d^k}$.
    \item For any $\cI\in \GenericCollection$ it holds that $\abr{\Scal\sbr{\cI}}\leq \S^k$.
\end{enumerate}
Consider the dataset  $\Zcal_{t,h}^N = \cbr{(s_{t,n}, a_{t,n},\psi_{n},s_{h,n})}_{n=1}^N$ generated by the following process:
    \begin{itemize}[itemindent=10pt]
    \item Execute $\mu\ind{t}\ldef\unif(\Psi\ind{t})$ up to layer $t$ (resulting in state $s_{t,n}$).
    \item Sample action $a_{t,n}\sim\unif(\cA)$ and play it, transitioning to $s_{t+1,n}$ in the process. %(resulting in state $s_{t+1,n}$)
    \item Sample $\psi\ind{t+1,h}_n\sim\unif(\Psi\ind{t+1,h})$ and execute it from layers $t+1$ to $h$ (resulting in $s_{h,n}$).
    \end{itemize}
Define a collection of empirical occupancies $$\EmpOccupancy = \crl*{\dhat_h\rbr{ \cdot \midsem \mu\circt\pi\circt[t+1]\psi\ind{t+1,h}} \mid \pi\in\Gamma,\psi\ind{t+1,h}\in\Psi},$$ where $\dhat_h\rbr{ \cdot \midsem \mu\circt\pi\circt[t+1]\psi\ind{t+1,h}}$ is given by (see also~\pref{line: importance sampling estimation} in~\pref{alg: dp-sr paper version})
\begin{align}
    \dhat_h(s \midsem \mu\circt\pi\circt[t+1]\psi\ind{t+1,h})=\frac{1}{N}\sum_{n=1}^{N}\frac{\indic\crl*{a_{t,n}=\pi(s_{t,n}), \psi\ind{t+1,h}_n=\psi\ind{t+1,h},s_{h,n}=s}}{(1/\abs{\cA})\cdot(1/\abs{\Psi})}. \label{eq: proof IS estimator }
\end{align}
Then, whenever
$
N = \bigom\rbr{\frac{A S^{2k} k \log\rbr{\frac{d SA}{\delta}}}{\epsilon^2}}
$
trajectories, with probability at least $1-\delta$ it holds that $\EmpOccupancy$ is $\eps$-approximate with respect to $\rbr{\mu\circt\Gamma\circt[t+1]\Psi, \fullc,h}$.
\end{restatable}

% \dfcomment{Update statement below to reflect that we aren't using
%   $\mut\circ_{t} \unf(\Acal)\circ_{t+1}  \unf\rbr{\Psi\ind{t+1,h}}$
%   notation anymore.}
% The claim follows by showing the estimator in~\eqref{eq: proof IS estimator } is unbiased and bounded. Then, using Bernstein concentration and a union bound we conclude the proof.
\begin{proof}[\pfref{lem: sample complexity of close model}]
Denote $\rho$ as the policy that generates the data $\Zcal_{t,h}^N $. Fix $\pi\in \Gamma,\psi\in \Psi,\cI \in \GenericCollection,s\sbr{\cI}\in \Scal\sbr{\Ical}$. It holds that
\begin{align*}
    &\dhat_h(s\sbr{\cI} \midsem \mu\circt\pi\circt[t+1]\psi) - d_h(s\sbr{\cI} \midsem \mu\circt \pi\circt[t+1]\psi)  \\
    &\stackrel{\mathrm{(a)}}{=} \sum_{s\sbr{\setComp{\cI}} \in \Scal\sbr{\setComp{\cI}}}\dhat_h(s \midsem \mu\circt\pi\circt[t+1]\psi) - d_h(s \midsem \mu\circt \pi\circt[t+1]\psi)\\
    & = \frac{1}{N}\sum_{n=1}^{N}\frac{\indic\crl*{a_{t,n}=\pi(s_{t,n}), \psi_n=\psi, s_{h,n}\sbr{\cI}=s\sbr{\cI}}}{(1/\abs{\cA})\cdot(1/\abs{\Psi})} - d_h(s\sbr{\cI} \midsem \mu\circt \pi\circt[t+1]\psi)\\
    &= \frac{1}{N} \sum_{n=1}^N \rbr{X_n\rbr{\pi,\psi,s_h\sbr{\cI}} - d_h(s\sbr{\cI} \midsem \mu\circt \pi\circt[t+1]\psi)} 
\end{align*}
where $$X_n\rbr{\pi,\psi,s_h\sbr{\cI}} \ldef
\frac{\indic\crl*{a_{t,n}=\pi(s_{t,n}), \psi_n=\psi,
    s_{h,n}\sbr{\cI}=s\sbr{\cI}}}{(1/\abs{\cA})\cdot(1/\abs{\Psi})}.$$
Note that $\mathrm{\mathrm{(a)}}$ holds by definition: both
$\dhat_h(s\sbr{\cI} \midsem \mu\circt\pi\circt[t+1]\psi)$ and
$d_h(s\sbr{\cI} \midsem \mu\circt \pi\circt[t+1]\psi)$ are given by
marginalizing all state factors in $\setComp{\cI}$. Observe that the
estimator $X_n$ is unbiased and bounded almost surely:
\begin{align}
    \EE_\rho[X_n\rbr{\pi,\psi,s\sbr{\cI}}] =  d_h(s\sbr{\cI} \midsem \mu\circt \pi\circt[t+1]\psi),\quad \text{ and } 0\leq X_n\rbr{\pi,\psi,s\sbr{\cI}}\leq A\abr{\Psi}. \label{eq: unbiased and bounded}
\end{align}
As a result, we can control the quality of approximation of
$\dhat_h(s\sbr{\cI} \midsem \mu\circt\pi\circt[t+1]\psi)$ using
Bernstein's inequality~(\pref{lem: bernstein's inequality}). First,
observe that the variance of each term in the sum can be bounded as follows:
\begin{align}
  \sigma^2 &\ldef \EE_\rho[\rbr{X_n\rbr{\pi,\psi,s\sbr{\cI}} - d_h(s\sbr{\cI} \midsem \mu\circt \pi\circt[t+1]\psi)}]\nonumber  \\
    &\stackrel{\mathrm{\mathrm{(a)}}}{\leq } \EE_\rho[X_n\rbr{\pi,\psi,s\sbr{\cI}}^2 ] \nonumber \\
    &\stackrel{\mathrm{(b)}}{\leq }  A\abr{\Psi}\EE_\rho[X_n\rbr{\pi,\psi,s\sbr{\cI}} ] \nonumber\\
    &\stackrel{\mathrm{(c)}}{= }  A\abr{\Psi}d_h(s\sbr{\cI} \midsem \mu\circt \pi\circt[t+1]\psi) \nonumber\\
    &\leq A\abr{\Psi}. \label{eq: bound on sigma square of estimator}
\end{align}
Here $\mathrm{\mathrm{(a)}}$ holds since $d_h(s\sbr{\cI} \midsem
\mu\circt \pi\circt[t+1]\psi)\geq 0$, $\mathrm{(b)}$ holds  since
$0\leq X_n\rbr{\pi,\psi,s\sbr{\cI}}\leq{}A\abs{\Psi}$, and
$\mathrm{(c)}$ holds by \eqref{eq: unbiased and bounded}. As a
result, using Bernstein's inequality, we have that for any fixed $\pi\in \Gamma,\psi\in \Psi,\cI \in \GenericCollection,s\sbr{\cI}\in \Scal\sbr{\cI}$, with probability at least $1-\delta$,
\begin{align*}
    &\abr{\dhat_h(s\sbr{\cI} \midsem \mu\circt \pi\circt[t+1]\psi)  - d_h(s\sbr{\cI} \midsem \mu\circt \pi\circt[t+1]\psi) } \\
    &\stackrel{\mathrm{(a)}}{\leq } O\rbr{\sqrt{ \frac{\sigma^2 \log\rbr{\frac{1}{\delta}}}{N}} + \frac{A\abr{\Psi}\log\rbr{\frac{1}{\delta}}}{N}}\\
    &\stackrel{\mathrm{(b)}}{\leq } O\rbr{\sqrt{\frac{A\abr{\Psi} \log\rbr{\frac{1}{\delta}}}{N}} + \frac{A\abr{\Psi}\log\rbr{\frac{1}{\delta}}}{N}},
\end{align*}
where $\mathrm{(a)}$ holds by\pref{lem: bernstein's inequality} and
$\mathrm{(b)}$ holds by \eqref{eq: bound on sigma square of
  estimator}. Setting $N = \Theta\rbr{\frac{A
    \abr{\Psi}\log{1\setminus{}\delta}}{\epsilon^2}}$ and using that $\eps^2\leq \eps$ for $\eps\in (0,1)$, we find that 
\begin{align*}
    \abr{\dhat_h(s\sbr{\cI} \midsem \mu\circt \pi\circt[t+1]\psi)  - d_h(s\sbr{\cI} \midsem \mu\circt \pi\circt[t+1]\psi) }\leq O\rbr{\eps + \eps^2}\leq \eps. 
\end{align*}
Finally, taking a union bound over all $\pi\in \Gamma,\psi\in \Psi,
\cI\in \GenericCollection,s\sbr{\cI}\in \Scal\sbr{\cI}$ and using
assumptions $(1)-(4)$, we conclude the proof.
% We find that given 
% \begin{align*}
%     N = \Theta\rbr{\frac{A S^{2k} k\log\rbr{\frac{dSA}{\delta}}}{\epsilon^2}}
% \end{align*}
% trajectories, with probability at least $1-\delta$ it  holds for all $\pi\in \Gamma, \psi\in \Psi, \cI\in \fullc$ and $s\sbr{\cI}\in \Scal\sbr{\cI}$ that
% $$
% \abr{\dhat_h(s\sbr{\cI} \midsem \mu\circt \pi\circt[t+1]\psi)  - d_h(s\sbr{\cI} \midsem \mu\circt \pi\circt[t+1]\psi) }  \leq \epsilon
% $$
% Hence, with probability at least $1-\delta$  the set $\EmpOccupancy$ is $\epsilon$-approximate with respect to $\rbr{\mu\circt\Gamma\circt[t+1]\Psi, \fullc,h}$.
\end{proof}

\subsection{Analysis}\label{app:analysis_preli}
The following elementary result shows that if two functions
$\widehat{f},f:\Xcal\rightarrow \mathbb{R}$ are point-wise close, any
approximate optimizer for $\wh{f}$ is an approximate optimizer for $f$.
% near-optimality guarantees on $\widehat{f}$ can be transformed to near-optimality guarantees on $f$. 
\begin{lemma}\label{lem: near optimal maximizers and minimzers help
    lemma}
  % \dfcomment{where is the norm below defined? can we write
  %   $\nrm{\cdot}_{\infty}$ explicitly? Also, we should write
  %   $\nrm{f-f'}$, not $\nrm{f(x)-f'(x)}$}
Let $\Xcal$ be a compact set, and let $f,\widehat{f}:\Xcal \rightarrow
\mathbb{R}$ be such that $$\norm{\widehat{f} -f}_{\infty}\ldef \max_{x\in
  \Xcal}\abs[\big]{\widehat{f}(x) -f(x)}\leq \epsilon.$$ Then, for any
$\veps'>0$, the following results hold:
\begin{enumerate}
    \item If $\max_{x\in \Xcal} \widehat{f}(x) > \min_{x\in \Xcal} \widehat{f}(x) +\epsilon'$, then $\max_{x\in \Xcal} f(x) > \min_{x\in \Xcal} f(x) +\epsilon' -2\epsilon$.
    \item If $\max_{x\in \Xcal} \widehat{f}(x) \leq  \min_{x\in \Xcal} \widehat{f}(x) +\epsilon'$, then  $\max_{x\in \Xcal} f(x) \leq  \min_{x\in \Xcal} f(x) +\epsilon' +2\epsilon$.
    \item For any $\widehat{x}\in \Xcal$, if $\max_{x\in \Xcal} \widehat{f}(x) > \widehat{f}(\widehat{x}) +\epsilon'$, then $\max_{x\in \Xcal} f(x) >  f(\widehat{x}) +\epsilon' -2\epsilon$.
    \item For any $\widehat{x}\in \Xcal$, if $\max_{x\in \Xcal} \widehat{f}(x) \leq  \widehat{f}(\widehat{x}) +\epsilon'$, then $\max_{x\in \Xcal} f(x) \leq  f(\widehat{x}) +\epsilon' +2\epsilon$.
\end{enumerate}

\end{lemma}
\begin{proof}[\pfref{lem: near optimal maximizers and minimzers help lemma}]
Denote the maximizer and minimizer of $f$ by
\begin{align*}
    x_{\min,f} \ldef \arg\min_{x\in \Xcal} f(x),\quad x_{\max,f} \ldef \arg\max_{x\in \Xcal} f(x),
\end{align*}
and denote the maximizer and minimizer of $\widehat{f}$ by
\begin{align*}
        x_{\min,\widehat{f}} \ldef\arg\min_{x\in \Xcal}\widehat{f}(x),\quad x_{\max,\widehat{f}} \ldef \arg\max_{x\in \Xcal} \widehat{f}(x).
\end{align*}
Note that these points exist by compactness of $\Xcal$.

Observe that the following relations hold by the assumption that $\norm{\widehat{f} -f}_{\infty}\leq \epsilon$:
\begingroup
\allowdisplaybreaks
\begin{align}
    &\max_{x\in \Xcal} \widehat{f}(x) = \widehat{f}(x_{\max,\widehat{f}}) \leq f(x_{\max,\widehat{f}}) +\epsilon \leq   \max_{x\in \Xcal}f(x) + \epsilon, \label{eq: relation 1 max app max bound}\\
    &\min_{x\in \Xcal} \widehat{f}(x) =   \widehat{f}(x_{\min,\widehat{f}}) \geq    f(x_{\min,\widehat{f}}) - \epsilon \geq \min_{x\in \Xcal} f(x) -\epsilon, \label{eq: relation 2 max app max bound}\\
    &\max_{x\in \Xcal} \widehat{f}(x) \geq \widehat{f}(x_{\max,f}) \geq  f(x_{\max,f}) -\epsilon = \max_{x\in \Xcal} f(x) -\epsilon, \label{eq: relation 3 max app max bound}\\
    &\min_{x\in \Xcal} \widehat{f}(x) \leq  \widehat{f}(x_{\min,f}) \leq   f(x_{\min,f}) + \epsilon = \min_{x\in \Xcal} f(x) +\epsilon. \label{eq: relation 4 max app max bound}
\end{align}
\endgroup

\paragraphp{Proof of the first claim.} Combining relations \eqref{eq: relation 1 max app max bound},~\eqref{eq: relation 2 max app max bound} and rearranging, we have 
\begin{align*}
    \max_{x\in \Xcal} \widehat{f}(x) > \min_{x\in \Xcal} \widehat{f}(x) +\epsilon' 
    \implies
    \max_{x\in \Xcal} f(x) > \min_{x\in \Xcal} f(x) +\epsilon' -2\epsilon.
\end{align*}

\paragraphp{Proof of the second claim.} Combining relations \eqref{eq: relation 3 max app max bound},  \eqref{eq: relation 4 max app max bound} and rearranging, we have 
\begin{align*}
    \max_{x\in \Xcal} \widehat{f}(x) \leq  \min_{x\in \Xcal} \widehat{f}(x) +\epsilon'
    \implies  
    \max_{x\in \Xcal} f(x) \leq  \min_{x\in \Xcal}f(x) +\epsilon' +2\epsilon.
\end{align*}

\paragraphp{Proof of the third claim.} By \eqref{eq: relation 1 max
  app max bound} and the assumption that $\norm{f - \widehat{f}}_{\infty}\leq \epsilon$, we have 
\begin{align*}
    \max_{x\in \Xcal} \widehat{f}(x) > \widehat{f}(\widehat{x}) +\epsilon' \implies  \max_{x\in \Xcal} f(x) >  f(\widehat{x}) +\epsilon' -2\epsilon.
\end{align*}

\paragraphp{Proof of the fourth claim.} By \eqref{eq: relation 3 max
  app max bound} and the assumption that $\nrm{f - \widehat{f}}_{\infty} \leq \epsilon$, we have 
\begin{align*}
    \max_{x\in \Xcal} \widehat{f}(x) \leq  \widehat{f}(\widehat{x}) +\epsilon' \implies  \max_{x\in \Xcal} f(x) \leq  f(\widehat{x}) +\epsilon' +2\epsilon.
\end{align*}

\end{proof}

\begin{lemma}[Equivalence of Maximizers for Scaled Positive Functions]\label{lem: maximizer are equivalent}
Let $\Xcal$, $\Ycal$, and $\Acal$ be finite sets. Let $f:\Xcal\times
\Acal \rightarrow \mathbb{R}$ and $g: \Ycal\rightarrow
\mathbb{R}_+$ and let $\PP$ be a probability measure over
$\Xcal\times \Ycal$. Let $\Pi_{\Xcal\times\Ycal}$ and $\Pi_{\Xcal}$ be
the sets of all mappings from $\Xcal\times \Ycal$ to $\Acal$ and $\Xcal$ to $\Acal$, respectively. Then,
\begin{align*}
    \max_{\pi\in \Pi_{\Xcal,\Ycal}} \EE_{x,y\sim \PP}\sbr{f(x,\pi(x,y))g(y)} =   \max_{\pi\in \Pi_{\Xcal}} \EE_{x,y\sim \PP}\sbr{f(x,\pi(x))g(y)}.
\end{align*}
\end{lemma}
\begin{proof}[\pfref{lem: maximizer are equivalent}]
By the skolemization lemma (\pref{lem:skolemization}), we can exchange
maximization and expectation by writing
\begin{align}
     &\max_{\pi\in \Pi_{\Xcal,\Ycal}} \EE_{x,y\sim \PP}\sbr{f(x,\pi(x,y))g(y)} = \EE_{x,y\sim \PP}\sbr{ \max_{a\in \Acal} \rbr{f(x,a)g(y)}}. \label{eq: relation max on sets rel 2}
\end{align}
Let $\pi^\star_f \in\Pi_\Xcal$ be defined via
$$
\pi^\star_f(x)\in \max_a f(x,a).
$$
Observe that for any $x,y\in \Xcal\times\Ycal$ it holds that
\begin{align}
    \max_a\rbr{f(x,a)g(y)} \stackrel{\mathrm{(a)}}{=}  g(y)\max_a f(x,a) = g(y) f(x,\pi^\star_f(x)), \label{eq: relation max on sets rel 22}
\end{align}
where $\mathrm{(a)}$ holds because $g(y)\geq 0$. Plugging~\eqref{eq: relation max on sets rel 22} back into~\eqref{eq: relation max on sets rel 2} we find that
\begin{align}
    &\max_{\pi\in \Pi_{\Xcal,\Ycal}} \EE_{x,y\sim \PP}\sbr{f(x,\pi(x,y))g(y)} \stackrel{\mathrm{(a)}}{=}   \EE_{x,y\sim \PP}\sbr{ f(x,\pi^\star_f(x))g(y)} \stackrel{\mathrm{(b)}}{\leq}  \max_{\pi\in \Pi_{\Xcal}} \EE_{x,y\sim \PP}\sbr{f(x,\pi(x))g(y)}, \label{eq: relation max on sets rel 3} 
\end{align}
where $\mathrm{\mathrm{(a)}}$ holds by~\eqref{eq: relation max on sets rel 22}, and $\mathrm{(b)}$ holds since $\pi^\star_f\in \Pi_{\Xcal}$.
Finally, observe that we trivially have
\begin{align}
    \max_{\pi\in \Pi_{\Xcal,\Ycal}} \EE_{x,y\sim \PP}\sbr{f(x,\pi(x,y))g(y)} \geq  \max_{\pi\in \Pi_{\Xcal}} \EE_{x,y\sim \PP}\sbr{f(x,\pi(x))g(y)}, \label{eq: relation max on sets rel 4}
\end{align}
since  $\Pi_{\Xcal} \subseteq \Pi_{\Xcal,\Ycal}$. Combining \eqref{eq: relation max on sets rel 3} and \eqref{eq: relation max on sets rel 4} yields the result.
% $
% \max_{\pi\in \Pi_{\Xcal,\Ycal}} \EE_{x,y\sim \PP}\sbr{f(x,\pi(x,y))g(y)} = \max_{\pi\in \Pi_{\Xcal}} \EE_{x,y\sim \PP}\sbr{f(x,\pi(x))g(y)}.
% $
\end{proof}

\begin{lemma}\label{lem: numerical verification}
  Let $k,k_1,k_2\in \mathbb{N}$ satisfying $1\leq k_2\leq k_1-1\leq k$
  be given. Then, for all $\eps>0$, $$\rbr{1+1/k}^{k-k_1} \epsilon +\epsilon/3k < \rbr{1+1/k}^{k-k_2} \epsilon.$$
This further implies that $\rbr{1+1/k}^{k-k_1} c\epsilon +\epsilon/3k < \rbr{1+1/k}^{k-k_2} c\epsilon$ for all $c\geq 1$.
\end{lemma}
\begin{proof}[\pfref{lem: numerical verification}]
We prove the result by explicitly bounding the difference:
\begin{align*}
  \rbr{1+1/k}^{k-k_1} \epsilon + \epsilon/3k - \rbr{1+1/k}^{k-k_2} \epsilon
 & = \rbr{\rbr{1+1/k}^{k_2-k_1} - 1} \rbr{1+1/k}^{k-k_2}\epsilon + \epsilon/3k\\
 &\stackrel{\mathrm{(a)}}{\leq}  \rbr{\rbr{1+1/k}^{-1} - 1} \rbr{1+1/k}^{k-k_2}\epsilon + \epsilon/3k\\
 & =  - \rbr{1+1/k}^{k-k_2} \epsilon/(1+k)  + \epsilon/3k \\
 & \stackrel{\mathrm{(b)}}{\leq} -\epsilon/(1+k) + \epsilon/3k.
\end{align*}
Here, relation $\mathrm{(a)}$ holds since $k_2-k_1\leq -1$ and
$(1+1/k)\geq 1$, and relation $\mathrm{(b)}$ holds since $k-k_2\geq 1$ which implies that $\rbr{1+1/k}^{k-k_2}\geq 1$.
Observe that $3k> 1+k$ for $k\geq 1$ which implies
that $$-\epsilon/(1+k) + \epsilon/3k<0$$ for $\epsilon>0$. Thus, under
the assumptions of the lemma, we have that $\rbr{1+1/k}^{k-k_1} \epsilon + \epsilon/3k - \rbr{1+1/k}^{k-k_2} \epsilon<0$, which implies that 
$$
\rbr{1+1/k}^{k-k_1} \epsilon +\epsilon/3k < \rbr{1+1/k}^{k-k_2} \epsilon.
$$
% Lastly, to prove the second claim, observe that the first claim implies that
% \begin{align*}
%     \rbr{1+1/k}^{k-k_1} c\epsilon +c\epsilon/3k < \rbr{1+1/k}^{k-k_2} c\epsilon,
% \end{align*}
% for any positive $c$. Since $c\geq 1$ this implies that
% \begin{align*}
%     \rbr{1+1/k}^{k-k_1} c\epsilon +\epsilon/3k\leq \rbr{1+1/k}^{k-k_1} c\epsilon +c\epsilon/3k < \rbr{1+1/k}^{k-k_2} c\epsilon,
% \end{align*}
% which proves the second claim.
\end{proof}

The following result is standard, so we omit the proof.
\begin{lemma}[Skolemization]\label{lem:skolemization}
Let $\Scal$ and $\Acal$ be finite sets and $\Pi$ be the set of
mappings from $\Scal$ to $\Acal$. Then for any function $f: \Scal\times\Acal \rightarrow \mathbb{R}$,
$
\max_{\pi\in \Pi}\EE\brk*{f(s,\pi(s))} = \EE\brk{\max_a f(s,a)}.
$
\end{lemma}

\subsection{$\fullc_{\leq k }\rbr{\cI}$ is a $\pi$-System}\label{app: structural results on class}
We now prove that
$\fullc_{\leq k}\rbr{\cI}$ is a $\pi$-system (that is, a set system
that is closed under intersection). Importantly, this implies that if
$\Ic\in \fullc_{\leq k}\rbr{\cI}$, then for any $\cI\in \fullc_{\leq
  k}\rbr{\cI} $, $ \Ic\cap \cI\ldef\Iendo{\cI}\in \fullc_{\leq
  k}\rbr{\cI}$. This fact is repeatedly being in the design and analysis of \AlgName in~\pref{sec: learning endogenous policy cover}.

\begin{lemma}[$\fullc_{\leq k}\rbr{\cI}$ is a $\pi$ system] \label{lem: set of abstractions is a pi system}
For any $\cI \in \fullc_{\leq{}k}$, $\fullc_{\leq{}k}\rbr{\cI}$ is a $\pi$-system:
\begin{enumerate}
    \item $\fullc_{\leq{}k}\rbr{\cI}$ is non-empty.
    \item For any $\cI_1,\cI_2\in \fullc_{\leq{}k}\rbr{\cI}$, we have $\cI_1 \cap \cI_2\in \fullc_{\leq{}k}\rbr{\cI}$.
\end{enumerate}
\end{lemma}
\begin{proof}[\pfref{lem: set of abstractions is a pi system}]
Since $\cI \in \fullc_{\leq{}k}$, we have $\abr{\cI}\leq k$. Furthermore, it trivially holds that $\cI\subseteq\cI$. Thus, $\cI\in \fullc_{\leq{}k}\rbr{\cI}$, which implies that $\fullc_{\leq{}k}\rbr{\cI}$ is non-empty.

We now prove the second claim. By definition, every $\cJ\in \fullc_{\leq{}k}\rbr{\cI}$ has $\cI \subseteq \cJ$. Thus, for any $\cI_1,\cI_2\in \fullc_{\leq{}k}\rbr{\cI}$,
\begin{align}
    \cI \subseteq \cI_1 \cap \cI_2. \label{eq: pi system rel 1}
\end{align}
Furthermore, since, both $\abr{\cI_1}\leq k$ and  $\abr{\cI_2}\leq k$,
we have
\begin{align}
    \abr{\cI_1 \cap \cI_2} \leq \min\crl{\abr{\cI_1}, \abr{\cI_2}}\leq k. \label{eq: pi system rel 2}
\end{align}
Combining \eqref{eq: pi system rel 1} and \eqref{eq: pi system rel 2} implies that $\cI_1 \cap \cI_2\in \fullc_{\leq{}k}\rbr{\cI}$.
\end{proof}

%%% Local Variables:
%%% mode: latex
%%% TeX-master: "paper"
%%% End:

\section{Structural Results for \kEMDPs}
\label{app:structural}

\subsection{Bellman Rank for the \framework Setting}
\label{app:bellman_rank}

In this section we show that in general, the \framework setting does not admit low Bellman rank~\citep{jiang2017contextual}, which is a standard structural complexity measure that enables tractable reinforcement learning in large state spaces. We expect that similar arguments apply for the related complexity measures~\citep{jin2021bellman,du2021bilinear} and other variations. We note that \citet{efroni2021provable} showed that the more general Exogenous Block MDP model does not admit low Bellman rank. Here, we show that the same conclusion holds for the specialized \framework model.

Recall that Bellman rank is a complexity measure that depends on the underlying MDP and on a class of action-value functions $\cF$ used to approximate $Q^\star$. For a policy $\pi$, denote the average Bellman error of function $f \in \cF$ by
\begin{align*}
    \cE_h(\pi,f) := \EE_{s_h \sim \pi, a_h \sim \pi_f}\sbr{f(s_h,a_h) - r_h - f(s_{h+1},\pi_f(s_{h+1})}.
\end{align*}
With $\Pi_{\cF} \ldef{} \{\pi_f: f \in \cF\}$ we define $\Ecal_h(\Pi_{\cF}, \cF) = \crl*{\cE_h(\pi,f)}_{\pi\in\Pi_{\cF},f\in\cF}$
as the matrix of Bellman residuals indexed by policies and value functions. The Bellman rank is defined as $\max_h \textrm{rank}(\Ecal_h(\Pi_{\cF}, \cF))$. 

\begin{proposition}
  For every $d = 2^i$ for $i\in\bbN$, there exists (i) an \framework with $S=3$, $A=2$, $H=2$, $d$ exogenous factors and $1$ endogenous factor, and (ii) a function class $\cF$ containing of $d$ functions, one of which is $Q^\star$ and the rest of which induce policies that are $1/8$ sub-optimal, such that such that the Bellman rank is at least $d-1$. 
\end{proposition}
% We remark that our construction can also be use to show that algorithms for the Bellman rank setting such as \textsf{OLIVE} \citep{jiang2017contextual} require $\Omega(d)$ trajectories to learn with constant sub-optimality. 
\begin{proof}
We construct a \framework with $H=2$, $\cA=\crl*{1,2}$ (so that $A=2$), a single endogenous factor with values in $\{1,2,3\}$, and $d$ binary exogenous factors with values in  $\{0,1\}$. 

Let $e_i \in \RR^d$ denote the $i^{\mathrm{th}}$ standard basis element. We take the first factor to be endogenous, and construct the initial distribution, transition dynamics, and rewards as follows:
\begin{itemize}
\item $d_1 = \textrm{Unif}( \{(1, e_i)\}_{i \in [d]})$.
\item $T( (2,e_i) \mid (1,e_i), 1) = 1$, and $T( (3, e_i) \mid (1, e_i), 2) = 1$.
\item $R( (2,e_i),\cdot) = 1/2$, and $R( (3,e_i),\cdot) = 3/4$.
\end{itemize}
There is only a single, terminal action at states $(2,e_i), (3,e_i)$, which we suppress from the notation. It is straightforward to verify that this is an \framework. Note that the optimal policy takes action $2$ at the initial state, and we have $V^\star = 3/4$. 

We first construct the class $\cF$. Since $d$ is a power of $2$, there exist subsets $A_1,\ldots,A_{d-1} \subset [d]$ such that:\footnote{This can be seen by associating the sets with rows of a Walsh matrix.}
\begin{align*}
    \forall j \in [d-1]: |A_j| = d/2, \qquad \forall j \ne k \in [d-1]: |A_j \cap A_k| = d/4.
\end{align*}
We define $\cF = \{f_0,f_1,\ldots,f_{d-1}\}$, with $f_0 = Q^\star$ and each $f_j$ associated with subset $A_j$ as follows:
\begin{align*}
    f_j((1,e_i), 2) = 3/4, \qquad & f_j((3,e_i),\cdot) = 3/4\\
    f_j((1,e_i), 1) = \one\{i \in A_j\}, \qquad &f_j((2,e_i),\cdot) = \one\{i \in A_j\}
\end{align*}
Observe that since there is no reward, each function has zero Bellman error at the first timestep (that is, $\Ecal_1(\pi_{f_i}, f_j) = 0\;\; \forall i,j \in \{0,\ldots,d-1\}$). On the other hand for $j,k \in [d-1]$ we have
\begin{align*}
    \Ecal_2(\pi_{f_j}, f_k) &= \frac{1}{d} \sum_{i=1}^d\one\{i \in A_j\} (f_k((2,e_i),\cdot) - 1/2) + \one\{i \notin A_j\} (f_k((3,e_i),\cdot) - 3/4)\\
    & = \frac{1}{d} \sum_{i=1}^d \one\{i \in A_j\} (f_k((2,e_i),\cdot) - 1/2)\\
    & = \frac{1}{d} \sum_{i=1}^d \one\{i \in A_j \cap A_k\} (1 - 1/2) + \one\{i \in A_j \cap \bar{A}_k\} (0 - 1/2)\\
    & = \frac{1}{2} \one\{j = k\},
\end{align*}
where we have used that $\abs{A_j\cap{}A_k}=\abs{A_j\cap\bar{A}_k}=d/4$ when $j\neq{}k$. This shows that we can embed a $(d-1) \times (d-1)$ identity matrix in $\Ecal_2(\Pi_{\cF},\cF)$, so we have $\textrm{rank}(\Ecal_2(\Pi_{\cF}, \cF)) \geq d-1$. 
\end{proof}

\subsection{Structural Results for State Occupancies}
\label{app:structural_occupancy}
In this section we provide structural results concerning the state occupancy measures in the \kEMDP\ model. These results  refine certain results derived for the more general EX-BMDP model in~\citet{efroni2021provable}.

% We start by showing in~\pref{lem: decoupling of future state dis} that if the policy $\pi$ is endogenous, the future state distributions of the endogenous and exogenous parts are independent. Building on this result, we establish the restriction lemma~\pref{lem: invariance of reduced policy cover}. It assumeThis result shows that for any factor set $\cJ$ -- even if  $\cJ$ does not contain all endogenous factors $\Ic$ -- and any factor set $\cI$ and $s\brk{\cI}\in \cS\brk{\cI}$ it holds that 
% \begin{align*}
%   \max_{\pi\in \PiInd\sbr{\cJ}} d_h\rbr{s\sbr{\cI} \midsem \mu\circt[t] \pi \circt[t+1] \rho} =  \max_{\pi\in \PiInd\sbr{\Iendo{\cJ}}} d_h\rbr{s\sbr{\cI} \midsem \mu \circt \pi \circt[t+1] \rho}.
% \end{align*}
% Meaning, to reach any state factor $s\brk{\cI}$

For the first result, we adopt the shorthand
\begin{align*}
     d_{h}^{\pi}\rbr{s\brk{\cI}}   \ldef d_{h}\rbr{s\brk{\cI} \midsem \pi} \ldef \PP_\pi(s_h\brk{\cI} = s\brk{\cI}).
\end{align*}
% for notational convenience.
%First, define
%\[
%\bbP(\cdot\midsem\pi)
%\]
%Then, define $\bbP_h^{\pi}(\cdot)$ as a shorthand in certain places.

% \dfcomment{We refer to $\mu$ as a policy throughout this section when we should be referring to it as a mixture policy.} \YEcomment{Fixed}
% \dfcomment{Mention that this is a refinement of the decoupling lemma from efroni 2021}
\begin{lemma}[Decoupling of state occupancy measures]\label{lem: decoupling of future state dis}
Fix $t,h\in [H]$ such that $t\leq h$. Let $\pi\in\PiNS\brk{\Ic}$ be an endogenous policy and let $\cI$ be any factor set. Then for any $s'\sbr{\cI}\in \Scal\sbr{\cI}$ and $s\in \Scal,a\in \Acal$ the following claims hold.
\begin{enumerate}
    \item $d_{h}^{\pi}\rbr{s'\brk{\cI} \mid s_t=s,a_t=a}  
    = d_{h}^{\pi}\rbr{s'\brk{\Iendo{\cI}} \mid s_t\brk{\Ic}=s\brk{\Ic},a_t=a}\cdot{}d_{h}\rbr{s'\brk{\Iexo{\cI}} \mid s_t\brk{\Inc}=s\brk{\Inc}}$.
    \item $d_{h}^{\pi}\rbr{s'\brk{\cI} \mid s_t=s} = d_{h}^{\pi}\rbr{s'\brk{\Iendo{\cI}} \mid s_t\brk{\Ic}=s\brk{\Ic}}\cdot{}d_{h}\rbr{s'\brk{\Iexo{\cI}} \mid s_t\brk{\Inc}=s\brk{\Inc}}$.
    \item For any endogenous mixture policy $\mu\in \Pimix\brk{\Ic}$ and factor set $\cI$,
    \begin{align*}
         d^\mu_h(s\brk{\cI}) = d^\mu_h(s\brk{\Iendo{\cI}} )\cdot{}d_h(s\brk{\Iexo{\cI}}).
    \end{align*}
%    \dfcomment{         $d^\mu_h(s'\brk{\cI})=\En_{\pi\sim\mu}\brk*{d^\pi_h(s'\brk{\cI})}$}
    Hence, the random variables $(s_h\brk{\Iendo{\cI}}, s_h\brk{\Iexo{\cI}})$ are independent under $\mu$.
\end{enumerate}
\end{lemma}
\begin{proof}[\pfref{lem: decoupling of future state dis}]
  The proof follows a simple backwards induction argument.% over timesteps $t = h-1,h-2,\ldots, 1$. 
\paragraph{Proof of Claims $1$ and $2$.} We prove the two claims by induction on $t'=h-1,..,t$. 

\paragraphp{Base case: $t'=h-1$.} The base case holds as an immediate consequence of the \framework structure. In more detail, we have the following results.
% by the model assumption (see~\pref{sec: preliminaries}).
% That is, for any $s'\sbr{\cI}\in \Scal\sbr{\cI},a\in \Acal$ claim 1 and claim 2 holds.
\begin{enumerate}
    \item \textbf{Claim 1.} %The following relations hold.
    \begin{align}
    &d_{h}^{\pi}\rbr{s'\brk{\cI} \mid s_{h-1}=s,a_{h-1}=a} \nonumber \\ 
    &= \sum_{s'\brk{\setComp{\cI}} \in \Scal\brk{\setComp{\cI}}} T(s'\brk{\cI} \mid s,a) \nonumber\\
    &= \sum_{s'\brk{\Ic \setminus{}\Iendo{\cI}} \in \Scal\brk{\Ic \setminus{}\Iendo{\cI}}} \sum_{s'\brk{\Inc \setminus{}\Iexo{\cI}} \in \Scal\brk{\Inc \setminus{}\Iexo{\cI}}} T(s'\brk{\Ic} \mid s\brk{\Ic},a)T(s'\brk{\Inc} \mid s\brk{\Inc}) \nonumber\\
    &= \sum_{s\brk{\Ic \setminus{}\Iendo{\cI}} \in \Scal\brk{\Ic \setminus{}\Iendo{\cI}}}T(s'\brk{\Ic} \mid s\brk{\Ic},a) \sum_{s\brk{\Inc \setminus{}\Iexo{\cI}} \in \Scal\brk{\Inc \setminus{}\Iexo{\cI}}}T(s'\brk{\Inc} \mid s\brk{\Inc}) \nonumber \\
    & = d_{h}^{\pi}\rbr{s'\brk{\Iendo{\cI}} \mid s_{h-1}\brk{\Ic} = s\brk{\Ic},a_{h-1}=a}  d_{h}\rbr{s'\brk{\Iexo{\cI}} \mid s_{h-1}\brk{\Inc} = s\brk{\Inc}}.  \label{eq: base case condition on a}
    \end{align}
    \item  \textbf{Claim 2.} % The following relations hold.
    \begin{align*}
       &d_{h}^{\pi}\rbr{s'\brk{\cI} \mid s_{h-1}=s}   \\
        &\stackrel{\mathrm{(a)}}{=} \sum_{a\in \Acal}d_{h}^{\pi}\rbr{s'\brk{\cI} \mid s_{h-1}=s,a_{h-1}=a}\pi_{h-1}(a\mid s\sbr{\Ic}) \\
        &\stackrel{\mathrm{(b)}}{=}  \sum_{a\in \Acal}d_{h}^{\pi}\rbr{s'\brk{\Iendo{\cI}} \mid s_{h-1}\brk{\Ic} = s\brk{\Ic},a_{h-1}=a}  d_{h}\rbr{s'\brk{\Iexo{\cI}} \mid s_{h-1}\brk{\Inc} = s\brk{\Inc}} \pi_{h-1}(a\mid s\sbr{\Ic}) \\
        &=  d_{h}\rbr{s'\brk{\Iexo{\cI}} \mid s_{h-1}\brk{\Inc} = s\brk{\Inc}} \sum_{a\in \Acal}d_{h}^{\pi}\rbr{s'\brk{\Iendo{\cI}} \mid s_{h-1}\brk{\Ic} = s\brk{\Ic},a_{h-1}=a}  \pi_{h-1}(a\mid s\sbr{\Ic})\\
        & \stackrel{\mathrm{(c)}}{=} d_{h}\rbr{s'\brk{\Iexo{\cI}} \mid s_{h-1}\brk{\Inc} = s\brk{\Inc}} d_{h}^{\pi}\rbr{s'\brk{\Iendo{\cI}} \mid s_{h-1}\brk{\Ic} = s\brk{\Ic}}. 
    \end{align*}
    Here $\mathrm{(a)}$ holds by Bayes' rule and because $\pi\in \Pi\brk{\Ic}$ is endogenous policy, $\mathrm{(b)}$ holds by \eqref{eq: base case condition on a}, and $\mathrm{(c)}$ holds by Bayes' rule and the law of total probability.
\end{enumerate}

\paragraphp{Induction step} Fix $t'<h-1$ and assume the induction hypothesis holds for $t'+1$.
\begin{enumerate}
    \item \textbf{Claim 1.} %The following relations hold.
    \begingroup
    \allowdisplaybreaks
    \begin{align}
    &d_{h}^{\pi}\rbr{s'\brk{\cI} \mid s_{t'}=s,a_{t'}=a} \nonumber \\
    &= \sum_{\bar{s}\in \Scal} d_{h}^{\pi}\rbr{s'\brk{\cI} \mid s_{t'+1}=\bar{s}}\PP( s_{t'+1}=\bar{s} \mid s_{t'}=s,a_{t'}=a) \nonumber \\
    &\stackrel{\mathrm{(a)}}{=} \sum_{\bar{s}\in \Scal} d_{h}^{\pi}\rbr{s'\brk{\cI} \mid s_{t'+1}=\bar{s}}T(\bar{s}\brk{\Ic} \mid s\brk{\Ic},a)T(\bar{s}\brk{\Inc} \mid s\brk{\Inc}) \nonumber\\
    & \stackrel{\mathrm{(b)}}{=} \sum_{\bar{s}\brk{\Ic}\in \Scal\brk{\Ic}} d_{h}^{\pi}\rbr{s'\brk{\Iendo{\cI}} \mid s_{t'+1}\brk{\Ic}=\bar{s}\brk{\Ic}}T(\bar{s}\brk{\Ic} \mid s\brk{\Ic},a) \nonumber \\
    &\quad \times \sum_{\bar{s}\brk{\Inc}\in \Scal\brk{\Inc}} d_{h}\rbr{s'\brk{\Iexo{\cI}} \mid s_{t'+1}\brk{\Inc}=\bar{s}\brk{\Inc}} T(\bar{s}\brk{\Inc} \mid s\brk{\Inc}) \nonumber\\
    & = d_{h}^{\pi}\rbr{s'\brk{\Iendo{\cI}} \mid s_{t'}\brk{\Ic}=s\brk{\Ic},a_{t'}=a} d_{h}\rbr{s'\brk{\Iexo{\cI}} \mid s_{t'}\brk{\Inc}=s\brk{\Inc}},\label{eq: step case condition on a}
    \end{align}
    \endgroup
    where $\mathrm{(a)}$ holds by the \framework model assumption (\pref{sec: preliminaries}), and $\mathrm{(b)}$ holds by the induction hypothesis.
    
    \item \textbf{Claim 2.} %The following relations hold.
    \begin{align*}
        &d_{h}^{\pi}\rbr{s'\brk{\cI} \mid s_{t'}=s}   \\
        &\stackrel{\mathrm{(a)}}{=} \sum_{a\in \Acal}d_{h}^{\pi}\rbr{s'\brk{\cI} \mid s_{t'}=s,a_{t'}=a}\pi_{t'}(a\mid s\sbr{\Ic}) \\
        &\stackrel{\mathrm{(b)}}{=} \sum_{a\in \Acal}d_{h}^{\pi}\rbr{s'\brk{\Iendo{\cI}} \mid s_{t'}\brk{\Ic} = s\brk{\Ic},a_{t'}=a}  d_{h}\rbr{s'\brk{\Iexo{\cI}} \mid s_{t'}\brk{\Inc} = s\brk{\Inc}} \pi_{t'}(a\mid s\sbr{\Ic}) \\
        &=  d_{h}\rbr{s'\brk{\Iexo{\cI}} \mid s_{t'}\brk{\Inc} = s\brk{\Inc}} \sum_{a\in \Acal}d_{h}^{\pi}\rbr{s'\brk{\Iendo{\cI}} \mid s_{t'}\brk{\Ic} = s\brk{\Ic},a_{t'}=a}  \pi_{t'}(a\mid s\sbr{\Ic})\\
        & \stackrel{\mathrm{(c)}}{=} d_{h}\rbr{s'\brk{\Iexo{\cI}} \mid s_{t'}\brk{\Inc} = s\brk{\Inc}} d_{h}^{\pi}\rbr{s'\brk{\Iendo{\cI}} \mid s_{t'}\brk{\Ic} = s\brk{\Ic}}. 
    \end{align*}
    Here $\mathrm{(a)}$ holds by Bayes' rule and because $\pi\in \Pi\brk{\Ic}$ is endogenous policy, $\mathrm{(b)}$ holds by \eqref{eq: step case condition on a}, and $\mathrm{(c)}$ holds by Bayes' rule and law of total probability.
\end{enumerate}
This proves the induction step and both claims.

\paragraph{Proof of Claim $3$.} We first prove the claim holds for $\pi\in \PiNS\brk{\Ic}$. That is, for any $\pi\in \PiNS\brk{\Ic}$, factor set $\cI$ and $s\brk{\cI}$, we have
    \begin{align}
         d^\pi_h(s\brk{\cI}) = d^\pi_h(s\brk{\Iendo{\cI}} )\cdot{}d_h(s\brk{\Iexo{\cI}}). \label{eq: decoupling_3rd_WNT}
    \end{align} 
This yields the result, since for $\mu\in \Pimix\brk{\Ic}$, \eqref{eq: decoupling_3rd_WNT} implies that
\begin{align*}
  d^\mu_h(s\brk{\cI}) 
               &= \EE_{\pi\sim \mu}\brk{ d^\pi_h(s\brk{\cI}) }\\
             &= \EE_{\pi\sim \mu}\brk{ d^\pi_h(s\brk{\Iendo{\cI}} )\cdot{}d_h(s'\brk{\Iexo{\cI}})}\\
             &= \EE_{\pi\sim \mu}\brk{d^\mu_h(s\brk{\Iendo{\cI}})} d_h(s\brk{\Iexo{\cI}})
               = d^\mu_h(s\brk{\Iendo{\cI}} )\cdot{} d_h(s\brk{\Iexo{\cI}}).
\end{align*}
% where $\mathrm{(a)}$ holds by~\eqref{eq: decoupling_3rd_WNT}, and $\mathrm{(b)}$ holds because the endogenous and exogenous factors are independent at the first timestep by the \framework model assumptions~\eqref{eq: model assumption d}.
We now prove~\eqref{eq: decoupling_3rd_WNT}. Fix $\pi\in \PiNS\brk{\Ic}$, and observe that
\begin{align}
  d^\pi_h(s ) 
    &\stackrel{\mathrm{(a)}}{=}   \EE_{s_1\sim d_1}\sbr{d_{h}^{\pi}\rbr{s \mid s_1} }   \nonumber \\
    & \stackrel{\mathrm{(b)}}{=} \EE_{s_1\sim d_1}\sbr{ d_{h}^{\pi}\rbr{s\brk{\Ic} \mid s_1\brk{\Ic} = s\brk{\Ic} }d_{h}^{\pi}\rbr{s\brk{\Inc} \mid s_1\brk{\Inc} = s\brk{\Inc}}}\nonumber \\
    &\stackrel{\mathrm{(c)}}{=}  \EE_{s_1\brk{\Ic}\sim d_h}\sbr{ d_{h}^{\pi}\rbr{s\brk{\Ic} \mid s_1\brk{\Ic} = s\brk{\Ic} }} \EE_{s_1\brk{\Inc}\sim d_h}\sbr{d_{h}^{\pi}\rbr{s\brk{\Inc} \mid s_1\brk{\Inc} = s\brk{\Inc}}} \nonumber\\
    & = d^\pi_h(s\sbr{\Ic}) d_h(s\sbr{\Inc}). \label{eq: third claim decoupling proof}
\end{align}
Relation $\mathrm{(a)}$ holds by the tower property, and relation $\mathrm{(b)}$ holds by the second claim of the lemma, because $\pi$ is an endogenous policy. Relation $\mathrm{(c)}$ holds because $s_1\brk{\Ic}$ and $s_1\brk{\Inc}$ are independent (by the \kEMDP\ model assumption, we have $d_1(s) = d_1(s\brk{\Ic})d_1(s\brk{\Inc})$).

The relation in~\eqref{eq: third claim decoupling proof} now implies the result:
\begin{align*}
   d^\pi_h(s\brk{\cI}) 
    &\stackrel{\mathrm{(a)}}{=} \sum_{s\brk{\Ic \setminus{} \Iendo{\cI}} \in \Scal\brk{\Ic \setminus{} \Iendo{\cI}}} \sum_{s\brk{\Inc \setminus{} \Iexo{\cI}}\in \Scal\brk{\Inc \setminus{} \Iexo{\cI}} }d^\pi_h(s )\\
    &\stackrel{\mathrm{(b)}}{=} \sum_{s\brk{\Ic \setminus{} \Iendo{\cI}} \in \Scal\brk{\Ic \setminus{} \Iendo{\cI}}} \sum_{s\brk{\Inc \setminus{} \Iexo{\cI}} \in \Scal\brk{\Inc \setminus{} \Iexo{\cI}}}d^{\pi}_h(s\sbr{\Ic}) d_h(s\sbr{\Inc}) \\
    & = \sum_{s\brk{\Ic \setminus{} \Iendo{\cI}}\in \Scal\brk{\Ic \setminus{} \Iendo{\cI}}} d^{\pi}_h(s\sbr{\Ic}) \sum_{s\brk{\Inc \setminus{} \Iexo{\cI}} \in \Scal\brk{\Inc \setminus{} \Iexo{\cI}}}d_h(s\sbr{\Inc})\\
    &  =  d^\pi_h(s'\brk{\Iendo{\cI}} )d_h(s'\brk{\Iexo{\cI}}),
\end{align*}
where $\mathrm{(a)}$ holds by the law of total probability and $\mathrm{(b)}$ holds by \eqref{eq: third claim decoupling proof}. %This completes the proof.\loose
\end{proof}

\begin{lemma}[Restriction lemma]\label{lem: invariance of reduced policy cover}
Fix $h,t\in [H]$ where $t\leq h-1$. Let $\mu\in\Pimix\brk{\Ic}$ and $\rho\in\PiNS\brk{\Ic}$ be endogenous policies. Let $\cJ$ and $\cI$ be two factor sets.
% and let $\cJ = \cJen\cup \Iexo{\cJ}$ be decomposed into its endogenous factors $\cJen = \cJ\cap \Ic$ and exogenous factors $\Iexo{\cJ} = \cJ\cap \Inc$.
Then, for all $s\sbr{\cI}\in \Scal\sbr{\cI}$ it holds that
\begin{align*}
   \max_{\pi\in \PiInd\sbr{\cJ}} d_h\rbr{s\sbr{\cI} \midsem \mu\circt[t] \pi \circt[t+1] \rho} =  \max_{\pi\in \PiInd\sbr{\Iendo{\cJ}}} d_h\rbr{s\sbr{\cI} \midsem \mu \circt \pi \circt[t+1] \rho}.
\end{align*}
\end{lemma}
Let us briefly sketch the proof. To begin, we marginalize over the factor set $\setComp{\cJ} \ldef [d] \setminus \cJ$ at layer $t$. We then show that if $\mu$ and $\rho$ are endogenous policies, then for all $\pi\in \PiInd$ and $s\sbr{\cI}\in \cS\brk{\cI}$,
\begin{align}
    d_h\rbr{s\sbr{\cI} \midsem \mu\circt[t] \pi \circt[t+1] \rho} =  \EE_{s_t\sim d_t\rbr{s\brk{\cJ} \midsem \pi}}\sbr{f(s_t\sbr{\Iendo{\cJ}},\pi\rbr{s_t \sbr{\cJ}}) \bar{g}(s_t \sbr{\Iexo{\cJ}})} \label{eq: what we show restricition lemma sketch}
\end{align}
where both $f$ and $\bar{g}$ are maps to $\mathbb{R}_+.$ We observe that the policy
\begin{align*}
    \pi_{f}(s\sbr{\Iendo{\cJ}})\in \argmax_af(s\sbr{\Iendo{\cJ}},\pi\rbr{s \sbr{\cJ}})
\end{align*}
also maximizes \eqref{eq: what we show restricition lemma sketch}. The result follows by observing that $\pi_{f}\in \Pi\brk{\Iendo{\cJ}}$.
% Formally, we invoke~\pref{lem: maximizer are equivalent} which proves a more general statement via the above reasoning.

\begin{proof}[\pfref{lem: invariance of reduced policy cover}]
Fix $s\sbr{\cI}\in \Scal\sbr{\cI}$.  The following relations hold.
\begin{align}
    &d_h\rbr{s\sbr{\cI} \midsem \mu\circt \pi \circt[t+1] \rho } \nonumber  \\
    &\stackrel{\mathrm{(a)}}{=} \EE_{s\sbr{\cJ} \sim d_t\rbr{\cdot \midsem \mu}}\sbr{\EE_{s\sbr{\setComp{\cJ}}\sim d_t\rbr{\cdot \mid s_t\sbr{\cJ}=s\sbr{\cJ} \midsem \mu} }\sbr{ d_h\rbr{s\sbr{\cI} \mid s_t=s \midsem \mu\circt[t] \pi \circt[t+1] \rho} }} \nonumber \\
    & \stackrel{\mathrm{(b)}}{=} \EE_{s\sbr{\cJ} \sim d_t\rbr{\cdot \midsem \mu}}\sbr{\EE_{s\sbr{\setComp{\cJ}}\sim d_t\rbr{\cdot \mid s_t\sbr{\cJ}=s\sbr{\cJ} \midsem \mu} }\sbr{ d_h\rbr{s\sbr{\cI} \mid s_t=s \midsem \mu\circt[t] \pi \circt[t+1] \rho}  }}\label{eq: max relation main 1},
\end{align}
where $\mathrm{(a)}$ holds by the tower property, and $\mathrm{(b)}$ holds by the Markov assumption of the dynamics: conditioning on the full state $s$ at timestep $t$, the future is independent of the history.  

% We now calculate the following two terms: \dfcomment{Can we rename these to $(\star)$ or something like this?}
\begin{align*}
    &(\star) \ldef d_h\rbr{s\sbr{\cI} \mid s_t=s \midsem \mu\circt[t] \pi \circt[t+1] \rho},\\    &(\star\star) \ldef \EE_{s\sbr{\setComp{\cJ}}\sim d_t\rbr{\cdot \mid s_t\sbr{\cJ}=s\sbr{\cJ} \midsem \mu} }\sbr{d_h\rbr{s\sbr{\cI} \mid s_t=s \midsem \mu\circt[t] \pi \circt[t+1] \rho} }.
\end{align*}

% \dfcomment{Is $\bar{\Ical}$ below a typo?}\YEcomment{fixed}
\paragraphp{Analysis of term $(\star)$.} Let $\pi\in \Pi\brk{\cJ}$. Fix $s\in \Scal$ at the $t^{\mathrm{th}}$ timestep, and observe that $a = \pi(s\sbr{\cJ})$ is also fixed, since the policy $\pi$ is a deterministic function of $s\sbr{\cJ}$
.% , and, thus, it is also a deterministic function of~$s$.
\begin{align}
    &d_h\rbr{s\sbr{\cI} \mid s_t=s \midsem \mu\circt[t] \pi \circt[t+1] \rho}  \nonumber \\
    &\stackrel{\mathrm{(a)}}{=} d_h\rbr{s\sbr{\cI} \mid s_t=s, a_t = \pi\rbr{s\brk{\cJ}} \midsem \rho}  \nonumber\\
    &\stackrel{\mathrm{(b)}}{=}\underbrace{d_h\rbr{s\sbr{\Iendo{\cI}} \mid s_t\brk{\Ic}=s\brk{\Ic}, a_t = \pi\rbr{s\brk{\cJ}} \midsem \rho}}_{\rdef \bar{f}(s_t\sbr{\Ic},\pi\rbr{s\brk{\cJ}})} \cdot\underbrace{d_h\rbr{s\sbr{\Iexo{\cI}} \mid s_t\brk{\Inc}=s\brk{\Inc}}}_{\rdef \bar{g}(s_t\sbr{\Inc})}. \label{eq: restriction lemma rel 1}
\end{align}
Relation $\mathrm{(a)}$ holds by the Markov property for the MDP, and relation $\mathrm{(b)}$ holds by the first statement of~\pref{lem: decoupling of future state dis}, which shows the the endogenous and exogenous state factors are decoupled; note that the assumptions of~\pref{lem: decoupling of future state dis} hold because $\rho$ is endogenous policy and $a=\pi\rbr{s\brk{\cJ}}$ is fixed. In addition, both $\bar{f}(\cdot)$ and $\bar{g}(\cdot)$ are mappings to $\mathbb{R}_+$.

% \dfcomment{Shouldn't this be $f(s_t\brk{\cI^{\star}_{\mathrm{nc}}},
%     \pi(s_t\brk{\bar{I}}))$ given that $\pi$ potentially depends on
%     factors not in $\cI^{\star}_{\mathrm{nc}}$?}\YEcomment{fixed}

\paragraphp{Analysis of term $(\star\star)$.} We consider term $(\star\star)$ and analyze it by marginalizing over the state factors not contained in $s\sbr{\cJ}$. Observe that $d_t\rbr{s\brk{\setComp{\cJ}} \mid s_t\sbr{\cJ}=s\sbr{\cJ} \midsem \mu}$  also factorizes between the endogenous and exogenous factors due to decoupling lemma (\pref{lem: decoupling of future state dis}, Claim 3):
\begin{align}
    &d_t\rbr{s\brk{\setComp{\cJ}} \mid s_t\sbr{\cJ}=s\sbr{\cJ} \midsem \mu} \nonumber \\
    &= d_t\rbr{s\brk{\Ic \setminus \Iendo{\cJ}} \mid s_t\sbr{\Iendo{\cJ}}=s\sbr{\Iendo{\cJ}} \midsem \mu}d_t\rbr{\Inc \setminus \Iexo{\cJ} \mid s_t\sbr{\Iexo{\cJ}}=s\sbr{\Iexo{\cJ}}}. \label{eq: inverse probability factorizes}
\end{align}
Hence, we have
\begin{align}
    &\EE_{s\sbr{\setComp{\cJ}}\sim d_t\rbr{\cdot \mid s_t\sbr{\cJ}=s\sbr{\cJ} \midsem \mu} }\sbr{d_h\rbr{s\sbr{\cI} \mid s_t=s \midsem \mu\circt[t] \pi \circt[t+1] \rho} } \nonumber \\
    &\stackrel{\mathrm{(a)}}{=}\EE_{s\sbr{\setComp{\cJ}}\sim d_t\rbr{\cdot \mid s_t\sbr{\cJ}=s\sbr{\cJ} \midsem \mu} }\sbr{\bar{f}(s\sbr{\Ic},\pi\rbr{s\sbr{\cJ}}) \bar{g}(s\sbr{\Inc}) } \nonumber\\
    &\stackrel{\mathrm{(b)}}{=} \underbrace{\EE_{s\sbr{\Ic \setminus{} \Iendo{\cJ}}\sim  d_t\rbr{\cdot \mid s_t\sbr{\Iendo{\cJ}}=s\sbr{\Iendo{\cJ}} \midsem \mu}} \sbr{\bar{f}(s\sbr{\Ic},\pi\rbr{s\sbr{\cJ}} )}}_{\rdef f(s\sbr{\Iendo{\cJ}},\pi\rbr{s\sbr{\cJ}})}  \underbrace{\EE_{s\sbr{\Inc \setminus{} \Iexo{\cJ}}\sim  d_t\rbr{\cdot  \mid s_t\sbr{\Iexo{\cJ}}=s\sbr{\Iexo{\cJ}}}}\sbr{\bar{g}(s\sbr{\Inc})}}_{\rdef g(s \sbr{\Iexo{\cJ}})}, \label{eq: inverse probability factorizes rel2}
\end{align}
where $\mathrm{(a)}$ holds by the calculation of term $(\star)$ in~\eqref{eq: restriction lemma rel 1}, and $\mathrm{(b)}$ holds by the decoupling of the occupancy measure $d_t\rbr{s\brk{\setComp{\cJ}} \mid s_t\sbr{\cJ}=s\sbr{\cJ} \midsem \mu}$ in \eqref{eq: inverse probability factorizes}.

% \dfcomment{Again, write dependence of $\pi$ on $\cJ$ explicitly.}\YEcomment{is the issue solved now?}
% \begin{align*}
%      &\bar{f}(s\sbr{\Iendo{\cJ}},\pi) \ldef \EE_{s\sbr{\Ic \setminus{} \Iendo{\cJ}}\sim  d_t\rbr{\cdot \mid s_t\sbr{\Iendo{\cJ}}=s\sbr{\Iendo{\cJ}} \midsem \mu}} \sbr{f(s\sbr{\Ic},\pi\rbr{s\sbr{\cJ}} )}  \text{ and }\\
%      & \bar{g}(s \sbr{\Iexo{\cJ}}) \ldef \EE_{s\sbr{\Inc \setminus{} \Iexo{\cJ}}\sim  d_t\rbr{\cdot \mid s_t\sbr{\Iexo{\cJ}}=s\sbr{\Iexo{\cJ}}}}\sbr{g(s\sbr{\Inc})}.
% \end{align*}

\paragraphp{Combining the results.} Plugging the expression in \eqref{eq: inverse probability factorizes rel2} back into~\eqref{eq: max relation main 1} yields
\begin{align}
    d_h\rbr{s\sbr{\cI} \midsem \mu\circt \pi \circt[t+1] \rho } = \EE_{s\sbr{\cJ} \sim d_t\rbr{\cdot \midsem \mu}}\sbr{f(s_t\sbr{\Iendo{\cJ}},\pi\rbr{s_t \sbr{\cJ}}) g(s_t \sbr{\Iexo{\cJ}})}. \label{eq: equivalence of maximizers bar I relation 1}
\end{align}
We conclude the proof by invoking~\pref{lem: maximizer are equivalent}, which gives
\begin{align*}
  \max_{\pi\in \PiInd\brk{\cJ}} d_h\rbr{s\sbr{\cI} \midsem \mu\circt \pi \circt[t+1] \rho} 
   &\stackrel{\mathrm{(a)}}{=} \max_{\pi\in \PiInd\brk{\cJ}}\EE_{s\sbr{\cJ} \sim d_t\rbr{\cdot \midsem \mu}}\sbr{  \rbr{ f(s\sbr{\Iendo{\cJ}},\pi\rbr{s \sbr{\cJ}})}g(s \sbr{\Iexo{\cJ}}) } \\
    &\stackrel{\mathrm{(b)}}{=} \max_{\pi\in \PiInd\brk{\Iendo{\cJ}}}\EE_{s\sbr{\cJ} \sim d_t\rbr{\cdot \midsem \mu}}\sbr{  \rbr{f(s\sbr{\Iendo{\cJ}},\pi\rbr{s \sbr{\Iendo{\cJ}}}) g(s \sbr{\Iexo{\cJ}}) }}\\
   &\stackrel{\mathrm{(c)}}{=}  \max_{\pi\in \PiInd\brk{\Iendo{\cJ}}} d_h\rbr{s\sbr{\cI} \midsem \mu\circt \pi \circt[t+1] \rho}.
\end{align*}
Relations $\mathrm{(a)}$ and $\mathrm{(c)}$ hold by  \eqref{eq: equivalence of maximizers bar I relation 1}. Relation $\mathrm{(b)}$ holds by invoking~\pref{lem: maximizer are equivalent} with $\Xcal = \Scal\brk{\Iendo{\cJ}},\ \Ycal = \Scal\sbr{\Iexo{\cJ}}, \Xcal\times \Ycal =\Scal\brk{\cJ}$, $f(x,a) = f(s\brk{\Iendo{\cJ}},a)$, $g(y) = g(s\brk{\Iexo{\cJ}})$, $\Pi_{\Xcal\times \Ycal} = \PiInd\brk{\cJ}$ and  $\Pi_{\Xcal} = \PiInd\brk{\Iendo{\cJ}}$. 
\end{proof}

The result is proven as a consequence of the restriction lemma~(\pref{lem: invariance of reduced policy cover}). 

% \dfcomment{I would prefer if we rename this to lemma given that it has a proof associated with it.} \YEcomment{done}.

\begin{lemma}[Existence of endogenous policy cover]\label{lem: endognous set is the maximizer}
Fix $h,t\in [H]$ with $t\leq h-1$. Let $\mu\in \Pimix\brk{\Ic}$ and $\rho\in \PiNS\brk{\Ic}$ be endogenous policies. Let $\cI$ be a factor set and $\fullc$ be a collection of factor sets with $\Ic\in \fullc$. Then for all $s\sbr{\cI}\in \Scal\sbr{\cI}$, 
\begin{align*}
   \max_{\pi\in \PiInduced\sbr{\fullc}} d_h\rbr{s\sbr{\cI} \midsem \mu\circt \pi \circt[t+1] \rho} =  \max_{\pi\in \PiInd\sbr{\Ic}} d_h\rbr{s\sbr{\cI} \midsem \mu\circt \pi \circt[t+1] \rho}.
\end{align*}
\end{lemma}
\begin{proof}[\pfref{lem: endognous set is the maximizer}]
For all $\cJ=\cJen\cup \Iexo{\cJ}\in \fullc$ and $s\sbr{\cI} \in \Scal\sbr{\cI}$, we have
\begin{align}
  \max_{\pi\in \PiInd\brk{\cJ}} d_h\rbr{s\sbr{\cI} \midsem\mu\circt \pi \circt[t+1] \rho} 
     &\stackrel{\mathrm{(a)}}{=} \max_{\pi\in \PiInd\brk{\Iendo{\cJ}}} d_h\rbr{s\sbr{\cI} \midsem\mu\circt \pi \circt[t+1] \rho} \nonumber\\ &\stackrel{\mathrm{(b)}}{\leq} \max_{\pi\in \PiInd\brk{\Ic}} d_h\rbr{s\sbr{\cI} \midsem\mu\circt \pi \circt[t+1] \rho}, \label{eq:max_is_endogenous_set}
\end{align}
where $\mathrm{(a)}$ holds by \pref{lem: invariance of reduced policy cover}, and $\mathrm{(b)}$ holds because $\PiInd\brk{\Iendo{\cJ}} \subseteq \PiInd\brk{\Ic}$ (since $\Iendo{\cJ} \subseteq \Ic$). Since~\eqref{eq:max_is_endogenous_set} holds for all $\cJ\in \fullc$, we conclude that
\begin{align}
    \max_{\pi\in \PiInduced\brk{\fullc}} d_h\rbr{s\sbr{\cI} \midsem \mu\circt \pi \circt[t+1] \rho} \leq \max_{\pi\in \PiInd\brk{\Ic}} d_h\rbr{s\sbr{\cI} \midsem\mu\circt \pi \circt[t+1] \rho }. \label{eq:max_is_endogenous_set_rel1}
\end{align}
On the other hand, since $\PiInd\brk{\Ic} \subseteq \PiInduced\brk{\fullc}$ it trivially holds that
\begin{align}
    \max_{\pi\in \PiInduced\brk{\fullc}} d_h\rbr{s\sbr{\cI} \midsem \mu\circt \pi \circt[t+1] \rho} \geq \max_{\pi\in \PiInd\brk{\Ic}} d_h\rbr{s\sbr{\cI} \midsem\mu\circt \pi \circt[t+1] \rho }. \label{eq:max_is_endogenous_set_rel2}
\end{align}
Combining~\eqref{eq:max_is_endogenous_set_rel1} and~\eqref{eq:max_is_endogenous_set_rel2} yields the result.
\end{proof}

% Building on the previous results, our next lemma (\pref{lem: approximately reaching endo part is approximately optimal}).
Consider the problem of finding a policy $\pi$ that maximizes
\begin{align}
    d_h\rbr{s\sbr{\cI} \midsem \mu\circt \pi \circt[t+1] \rho}, \label{eq: approximately reaching endo is sufficient explanation}
\end{align}
where both $\mu$ and $\rho$ are endogenous policies. Our next result (\pref{lem: approximately reaching endo part is approximately optimal}) shows that if $\widehat{\pi}$ is an endogenous policy that is approximately optimal for reaching $s\sbr{\Iendo{\cI}}$ in the sense that
    \begin{align}
        \max_{\pi\in \PiInduced\brk{\fullc}} d_h\rbr{s\sbr{\Iendo{\cI}} \midsem \mu\circt \pi \circt[t+1] \rho} \leq d_h\rbr{s\sbr{\Iendo{\cI}} \midsem \mu\circt \widehat{\pi} \circt[t+1] \rho} + \epsilon,
    \end{align}
then it is also approximately optimal for~\eqref{eq: approximately reaching endo is sufficient explanation}, in the sense that
\begin{align*}
    \max_{\pi\in \Pi\brk{\fullc}}d_h\rbr{s\sbr{\cI} \midsem \mu\circt \pi \circt[t+1] \rho} \leq d_h\rbr{s\sbr{\cI} \midsem \mu\circt \widehat{\pi} \circt[t+1] \rho} + \epsilon.
\end{align*}

\begin{lemma}[Optimizing for endogenous factors is sufficient]\label{lem: approximately reaching endo part is approximately optimal}
  Fix $h,t\in [H]$ with $t\leq h-1$. Let $\mu\in \Pimix,\widehat{\pi}\in \PiInd$ and $\rho\in \PiNS$ be given. Let $\cI$ be a factor set and $\fullc$ be a collection of factor sets such that $\Ic\in \fullc$. Fix $s\brk{\cI}\in\cS\brk*{\cI}$ and assume that:
\begin{enumerate}[label=$(\mathrm{A}\arabic*)$]
    \item $\mu, \rho$ and $\widehat{\pi}$ are endogenous .
    \item $\widehat{\pi}$ is approximately optimal for $s\sbr{\Iendo{\cI}}$:
\[
\max_{\pi\in \PiInduced\brk{\fullc}} d_h\rbr{s\sbr{\Iendo{\cI}} \midsem \mu\circt \pi \circt[t+1] \rho} \leq d_h\rbr{s\sbr{\Iendo{\cI}} \midsem \mu\circt \widehat{\pi} \circt[t+1] \rho} + \epsilon. 
\]
\end{enumerate}
Then
\begin{align*}
     \max_{\pi\in \PiInduced\brk{\fullc}} d_h\rbr{s\sbr{\cI} \midsem \mu\circt \pi \circt[t+1] \rho} \leq d_h\rbr{s\sbr{\cI} \midsem \mu\circt \widehat{\pi} \circt[t+1] \rho } + \epsilon.
\end{align*}
\end{lemma}
\begin{proof}[\pfref{lem: approximately reaching endo part is approximately optimal}]
By assumption $(\mathrm{A1})$, $\mu$ and $\rho$ are endogenous policies, so \pref{lem: endognous set is the maximizer} yields
\begin{align}
    \max_{\pi\in \PiInduced\brk{\fullc}} d_h\rbr{s\sbr{\cI} \midsem \mu\circt \pi \circt[t+1] \rho} = \max_{\pi\in \PiInduced\brk{\Ic}} d_h\rbr{s\sbr{\cI} \midsem \mu\circt \pi \circt[t+1] \rho}.\label{eq: approxiamate maximzer rel 1}
\end{align}
Next, we observe that the following relations hold
% \dfcomment{Never do this thing where you put eqref in the math environment} \YEcomment{fixed}
\begin{align} 
  \max_{\pi\in \PiInduced\brk{\Ic}} d_h\rbr{s\sbr{\cI} \midsem \mu\circt \pi \circt[t+1] \rho} 
    & \stackrel{\mathrm{(a)}}{=} \rbr{\max_{\pi\in \PiInduced\brk{\Ic}} d_h\rbr{s\sbr{\Iendo{\cI}} \midsem \mu\circt \pi \circt[t+1] \rho}}d_h\rbr{s\sbr{\Iexo{\cI}} }\nonumber\\
    &\stackrel{\mathrm{(b)}}{\leq}  d_h\rbr{s\sbr{\Iendo{\cI}} \midsem \mu\circt \widehat{\pi}
      \circt[t+1] \rho}d_h\rbr{s\sbr{\Iexo{\cI}}} + \epsilon \nonumber\\
     &\stackrel{\mathrm{(c)}}{=}d_h\rbr{s\sbr{\Iendo{\cI}} ,s\sbr{\Iexo{\cI}}\midsem \mu\circt \widehat{\pi} \circt[t+1] \rho} + \epsilon  \nonumber\\
     &= d_h\rbr{s\sbr{\cI}\midsem \mu\circt \widehat{\pi} \circt[t+1] \rho} + \epsilon \label{eq: approxiamate maximzer rel 2}.
\end{align}
Relation $\mathrm{(a)}$ holds by \pref{lem: decoupling of future state dis}, as $\mu \circt \pi \circt[t+1] \rho$ is an endogenous policy. Relation $\mathrm{(b)}$ holds by assumption $(\mathrm{A}2)$ and because $d_h\rbr{s\sbr{\Iexo{\cI}}}\leq 1$. Relation $\mathrm{(c)}$ holds by \pref{lem: decoupling of future state dis}; note that assumptions of the lemma are satisfied because $\mu \circt \widehat{\pi} \circt[t+1] \rho$ is endogenous. Combining~\eqref{eq: approxiamate maximzer rel 1} and~\eqref{eq: approxiamate maximzer rel 2} concludes the proof.
\end{proof}

\subsection{Structural Results for Value Functions}\label{app:structural_reward}
In this section we provide a structural results concerning the values functions for endogenous policies in the \kEMDP\ model. These results leverage the assumption that the rewards depend only on endogenous components. % These structural results are essential to the proof of  $\PSDPE$~\pref{app: psdp in the presence of exogenous information}.
We repeatedly invoke the notion of an \emph{endogenous MDP} ${\Mcal_{\mathrm{en}} = \rbr{\Scal\brk*{\Ic},\Acal,\Tc, \rc,H,\dc}}$, which corresponds to the restriction of an \framework $\cM$ to the endogenous component of the state space. Note that only endogenous policies are well-defined in the endogenous MDP. We also denote the state-action and state value functions of an endogenous policy measured in $\Mcal_{\mathrm{en}}$ as $Q^\pi_{h,\mathrm{en}}(s\brk{\Ic},a)$, and $V^\pi_{h,\mathrm{en}}(s\brk{\Ic})$.

% \dfcomment{Explain notation for value functions in the endogenous MDP here.} \dfcomment{Also need to just formally define Q- and V-value functions somewhere.} \yecomment{these are formallly defined in 2.6, second paragraph}

Our first result is a straightforward extension of Proposition 5 in~\citet{efroni2021provable}. It shows that the value function for any endogenous policy in an \framework is an \emph{endogenous function} in the sense that it only depends on the endogenous state factors.

% It is relatively simple to construct counter-examples for such claim if the policy is not endogenous; if the policy is not endogenous and depends on the exogenous state factors, its value may depend on the exogenous part of the state.
\begin{lemma}[Value functions for endogenous policies are endogenous]
  \label{lem:value_endogenous}
Let $\pi\in \PiIndNS\brk{\Ic}$ be an endogenous policy, and assume that the reward function is endogenous. Then, for any $t\in [H]$ and $s\in \Scal$, we have
\begin{align*}
    V^\pi_t(s) = V^\pi_{t,\mathrm{en}}(s\brk{\Ic}) \text{ and } Q^\pi_t(s,a) = Q^\pi_{t,\mathrm{en}}(s\brk{\Ic},a),
\end{align*}
where $V^\pi_{t,\mathrm{en}}$ and $Q^\pi_{t,\mathrm{en}}$ are value functions for $\pi$ in the endogenous MDP ${\Mcal_{\mathrm{en}} = \rbr{\Scal\brk*{\Ic},\Acal,\Tc, \rc,H,\dc}}$.
\end{lemma}
\begin{proof}[\pfref{lem:value_endogenous}]
Let $R=\crl*{R_h}_{h=1}^{H}$ denote the reward function. We prove the result via induction. The base case $t=H$ holds by the assumption that the reward is endogenous. Next, assume the claim is correct for $t+1$, and let us prove it for $t$. Since $R_t$ is endogenous, the inductive hypothesis yields
\begin{align}
    &V^\pi_t(s) \nonumber \\
    &= \EE_\pi\brk*{\rct(s\brk{\Ic},\pi_t(s\brk{\Ic})) + V^\pi_{t,\mathrm{en}+1}(s_{t+1}\brk{\Ic})| s_t=s, a=\pi_{t+1}(s\brk{\Ic})} \nonumber \\
    &\stackrel{\mathrm{(a)}}{=} \rct(s\brk{\Ic},\pi_t(s\brk{\Ic})) \nonumber \\
    &\quad\ + \sum_{s'\brk{\Ic}\in \cS\brk{\Ic}}\Tc\rbr{s'\brk{\Ic} \mid s\brk{\Ic},\pi_{t+1}(s\brk{\Ic})}V^\pi_{t,\mathrm{en}+1}(s'\brk{\Ic}) \sum_{s'\brk{\Ic^{\mathrm{c}}}\in \cS\brk{\Ic^{\mathrm{c}}}}\Tc\rbr{s'\brk{\Ic^{\mathrm{c}}} \mid s\brk{\Ic^{\mathrm{c}}}} \nonumber\\
    & \stackrel{\mathrm{(b)}}{=} \rct(s\brk{\Ic},\pi_t(s\brk{\Ic})) + \sum_{s'\brk{\Ic}\in \cS\brk{\Ic}}\Tc\rbr{s'\brk{\Ic} \mid s\brk{\Ic},\pi_{t+1}(s\brk{\Ic})) }V^\pi_{t,\mathrm{en}+1}(s'\brk{\Ic}), \label{eq: value of endo policies is endogenous}
\end{align}
where $\mathrm{(a)}$ holds by the factorization of the transition operator (see~\eqref{eq: model assumption d}), and $\mathrm{(b)}$ holds by marginalizing the exogenous factors, since $ \sum_{s'\brk{\Ic^{\mathrm{c}}}\in \cS\brk{\Ic^{\mathrm{c}}}}\Tc\rbr{s'\brk{\Ic^{\mathrm{c}}} \mid s\brk{\Ic^{\mathrm{c}}}}=1$. Finally, observe that~\eqref{eq: value of endo policies is endogenous} is the precisely the value function for $\pi$ in the endogenous MDP $\Mcal_{\mathrm{en}} = \rbr{\Scal\brk*{\Ic},\Acal,\Tc, \rc,H,\dc}$, which concludes the proof.
\end{proof}

\begin{lemma}[Performance difference lemma for endogenous policies]\label{lem:pd_for_endo_policies}
Let $\pi,\pi'\in \PiIndNS\brk{\Ic}$ be endogenous policies. Then
\begin{align*}
  \Jpi- \Jpi[\pi']= \EE_{\pi}\sbr{\sum_{t=1}^H Q_{t}^{\pi'}(s_t\brk{\Ic},\pi_t(s_t\brk{\Ic})) - Q_t^{\pi'}(s_t\brk{\Ic},\pi'_t(s_t\brk{\Ic}))}.
\end{align*}
\end{lemma}
\begin{proof}[\pfref{lem:pd_for_endo_policies}]
For any endogenous policy $\pi$, observe that
\begin{align}
    \Jpi \ldef \EE_{s_1\sim d_1}\brk{V^\pi_1(s_1)} \stackrel{\mathrm{(a)}}{=} \EE_{s_1\sim d_1}\brk{V^\pi_1(s_1\brk{\Ic})} \stackrel{\mathrm{(b)}}{=} \EE_{s_1\brk{\Ic}\sim \dc}\brk{V^\pi_1(s_1\brk{\Ic})} =  J_{\mathrm{en}}(\pi), \label{eq:equivalence_of_values}
\end{align}
Relation $\mathrm{(a)}$ holds by~\pref{lem:value_endogenous}, since $J_{\mathrm{en}}(\pi)$ is the averaged value  of $V^\pi_1(s_1)$ with respect to the initial endogenous distribution. Relation $\mathrm{(b)}$ holds by marginalizing out $s_1\brk{\Inc}$, since $V^\pi_1(s_1\brk{\Ic})$ does not depend on this quantity. Using \pref{eq:equivalence_of_values} and applying the standard performance difference lemma to the endogenous MDP $\Mcal_{\mathrm{en}} $ now yields
\begin{align*}
  \Jpi- \Jpi[\pi']
    &= J_{\mathrm{en}}(\pi) - J_{\mathrm{en}}(\pi')
    & =\EE_{\pi}\sbr{\sum_{t=1}^H Q_{t}^{\pi'}(s_t\brk{\Ic},\pi_t(s_t\brk{\Ic})) - Q_t^{\pi'}(s_t\brk{\Ic},\pi'_t(s_t\brk{\Ic}))}.
\end{align*}
% where $\mathrm{(a)}$ holds by~\eqref{eq:equivalence_of_values}, and $\mathrm{(b)}$ holds by the performance difference lemma (\pref{lem: value difference}).
\end{proof}

\begin{lemma}[Restriction lemma for endogenous rewards]\label{lem:restriction_endo_rewards}
Fix $t\leq{}h$. Let $\mu\in \Pimix\brk{\Ic}$ and $\rho\in \PiIndNS\brk{\Ic}$ be endogenous policies. Define
\begin{align}
    V_{t,h}\rbr{\mu\circt \pi \circt[t+1] \rho } \ldef \EE_{\mu\circt \pi \circt[t+1] \rho }\brk*{ \sum_{t'=t}^h r_{t'}}. \label{eq:expected_value_def}
\end{align}
Assume that $R$ is an endogenous reward function. Then
for any factor set $\cI$, we have
\[\max_{\pi \in \PiInd\brk{\cI}} V_{t,h}\rbr{\mu \circt \pi \circt[t+1] \psi} = \max_{\pi \in \PiInd\brk{\Iendo{\cI}}} V_{t,h}\rbr{\mu \circt \pi \circt[t+1] \psi}.\]
\end{lemma}

To prove this result, we generalize the proof technique used in the restriction lemma for state occupancy measures (\pref{lem: invariance of reduced policy cover}).
% Fix a factor set $\cI$. Then, we explicitly marginalize over the factor set $\cI^\mathrm{c}$ at the $t^{\mathrm{th}}$ timestep; meaning, we marginalize over all the information that is not accessible to a policy $\pi\in \Pi\brk{\cI}$. With this we show that
% \begin{align*}
%     V_{t,h}\rbr{\mu \circt \pi \circt[t+1] \psi} = \EE_{s\brk{\cI} \sim d_t\rbr{\cdot \midsem \mu}}\brk*{f(s_t\brk{\Iendo{\cI}}, \pi_t\rbr{s\brk*{\cI}}) }.
% \end{align*}
% It is now evident that the policy $\pi_f$ defined by
% $$
% \pi_f(s\brk{\Iendo{\cI}})\in \argmax_a f(s\brk{\Iendo{\cI}}, \pi_t\rbr{s\brk*{\cI}})
% $$ 
% maximizes $V_{t,h}\rbr{\mu \circt \pi \circt[t+1] \psi}$ and $\pi_f\in \PiInd\brk{\Iendo{\cI}}$. 

\begin{proof}[\pfref{lem:restriction_endo_rewards}]
Since $\mu\in \Pimix\brk{\Ic}$ is an endogenous policy, the occupancy measure at the $t^{\mathrm{th}}$ timestep factorizes. That is, by the third statement of~\pref{lem: decoupling of future state dis}, we have that 
\begin{align*}
d_t\rbr{s\brk{\cI} \midsem \mu} = d_t\rbr{s\brk{\Iendo{\cI}} \midsem \mu} d_t\rbr{s\brk{\Iexo{\cI}}}.
\end{align*}
For each $s\brk{\cI}\in \cS\brk*{\cI}$, the conditional state occupancy measure factorize as well:
\begin{align}
    &d_t\rbr{s\brk{\setComp{\cI}} \mid s_t\brk{\cI} = s\brk{\cI} \midsem \mu} \nonumber \\
    &= d_t\rbr{ s\brk{\Ic \setminus \Iendo{\cI}} \mid s_t\brk{\Iendo{\cI}}= s\brk{\Iendo{\cI}} \midsem\mu}  d_t\rbr{s\brk{\Ic^{\mathrm{c}}  \setminus \Iexo{\cI}} \mid s_t\brk{\Iexo{\cI}} = s\brk{\Iexo{\cI}}}. \label{eq: factorization of cond state dist}
\end{align}
Let $Q^\rho_{t,\mathrm{en}}$ be the $Q$ function on the endogenous MDP $\Mcal_{\mathrm{en}} = \rbr{\Scal\brk*{\Ic},\Acal,\Tc, \rc,h,\dc}$ when executing policy $\rho$ starting from timestep $t+1$. We can express the value function as follows:
\begin{align}
    &V_{t,h}\rbr{\mu\circt \pi \circt[t+1] \rho } \nonumber \\
    &= \EE_{\mu}\brk*{ Q^\rho_t( s_t\brk*{[d]},  \pi_t\rbr{s_t\brk*{\cI}})} \nonumber\\ &\stackrel{\mathrm{(a)}}{=} \EE_{\mu}\brk*{ Q^\rho_{t,\mathrm{en}}( s_t\brk*{\Ic},  \pi_t\rbr{s_t\brk*{\cI}})} \nonumber\\
    &=\EE_{s\brk{\cI} \sim d_t\rbr{\cdot \midsem\mu}}\brk*{ \EE_{s\brk{\setComp{\cI}} \sim d_t\rbr{\cdot \mid s_t\brk{\cI} = s\brk{\cI}\midsem \mu}}\brk*{ Q^\rho_{t,\mathrm{en}}( s\brk*{\Ic},  \pi_t\rbr{s\brk*{\cI}})}} \nonumber\\
    &\stackrel{\mathrm{(b)}}{=} \EE_{s\brk{\cI} \sim d_t\rbr{\cdot \midsem \mu}}\brk*{ \EE_{s\brk{[\Ic \setminus \Iendo{\cI}} \sim d_t\rbr{\cdot \mid s_t\brk{\Iendo{\cI}}= s\brk{\Iendo{\cI}}\midsem \mu}}\brk*{ Q^\rho_{t,\mathrm{en}}( s\brk*{\Ic},  \pi_t\rbr{s\brk*{\cI}})}}. \label{eq: value restriction rel 1}
\end{align}
Relation $\mathrm{(a)}$ holds by~\pref{lem:value_endogenous}, since $\rho$ is an endogenous policy. Relation $\mathrm{(b)}$ holds by decoupling of conditional occupancy measure (\eqref{eq: factorization of cond state dist}), and because $Q^\rho_{t,\mathrm{en}}\rbr{ s\brk*{\Ic},  \pi_t\rbr{s\brk*{\cI}}}$ does not depend on state factors in $\Ic^{\mathrm{c}} \setminus \Iexo{\cI}$, which are marginalized out.

To proceed, define
\begin{align*}
    f(s_t\brk{\Iendo{\cI}}, \pi_t\rbr{s\brk*{\cI}}) \ldef \EE_{s\brk{[\Ic \setminus \Iendo{\cI}} \sim d_t\rbr{\cdot \mid s_t\brk{\Iendo{\cI}}= s\brk{\Iendo{\cI}}\midsem \mu}}\brk*{ Q^\rho_{t,\mathrm{en}}( s\brk*{\Ic},  \pi_t\rbr{s\brk*{\cI}})}.
\end{align*}
With this notation, we can rewrite the expression in~\eqref{eq: value restriction rel 1} as
\begin{align}
    V_{t,h}\rbr{\mu\circt \pi \circt[t+1] \rho } = \EE_{s\brk{\cI} \sim d_t\rbr{\cdot \midsem \mu}}\brk*{f(s_t\brk{\Iendo{\cI}}, \pi_t\rbr{s\brk*{\cI}}) }.\label{eq: value restriction rel 2}
\end{align}
% Let
% \begin{align*}
%     f(s\brk{\Iendo{\cI}}, \pi(s\brk{\cI})) \ldef \EE_{s\brk{[\Ic \setminus \Iendo{\cI}} \sim d_t\rbr{\cdot \mid s_t\brk{\Iendo{\cI}}= s\brk{\Iendo{\cI}}\midsem \pi}}\brk*{ Q^\rho_{t,\mathrm{en}}( s\brk*{\Ic},  \pi_t\rbr{s\brk*{\cI}})},
% \end{align*}
% and define the policy
% \begin{align}
%     \pi_f(s\brk{\Iendo{\cI}})\in \arg\max_{a}f(s\brk{\Iendo{\cI}}, \pi(s\brk{\cI}))\label{eq:def_pi_f_value_res}
% \end{align}
% Observe that $\pi_f(s\brk{\Iendo{\cI}})\in \PiInd\brk{\Iendo{\cI}}$.

We now invoke~\pref{lem: maximizer are equivalent}, which shows that
\begin{align}
  \max_{\pi\in \PiInd\brk{\cI}}V_{t,h}\rbr{\mu\circt \pi \circt[t+1] \rho } 
    &\stackrel{\mathrm{(a)}}{=} \max_{\pi\in \PiInd\brk{\cI}}\EE_{s\brk{\cI} \sim d_t\rbr{\cdot \midsem \pi}}\brk*{ f(s\brk{\Iendo{\cI}}, \pi(s\brk{\cI}))} \nonumber\\
    &\stackrel{\mathrm{(b)}}{=} \max_{\pi\in \PiInd\brk{\Iendo{\cI}}}\EE_{s\brk{\cI} \sim d_t\rbr{\cdot \midsem \pi}}\brk*{ f(s\brk{\Iendo{\cI}}, \pi(s\brk{\cI}))} \nonumber\\
    &\stackrel{\mathrm{(c)}}{=} \max_{\pi\in \PiInd\brk{\Iendo{\cI}}}\EE_{s\brk{\cI} \sim d_t\rbr{\cdot \midsem \pi}}\brk*{ f(s\brk{\Iendo{\cI}}, \pi(s\brk{\cI}))} \nonumber
\end{align}
Relations $\mathrm{(a)}$ and $\mathrm{(c)}$ holds by~\eqref{eq: value restriction rel 2}. Relation $\mathrm{(b)}$ holds by invoking~\pref{lem: maximizer are equivalent}, with $\Xcal = \Scal\brk{\Iendo{\cI}},\ \Ycal = \Scal\sbr{\Iexo{\cI}}, \Xcal\times \Ycal =\Scal\brk{\cI}$, $f(x,a) = f(s\brk{\Iendo{\cJ}},a), g(y) = 1$, and $\Pi_{\Xcal\times \Ycal} = \PiInd\brk{\cI}$ and  $\Pi_{\Xcal} = \PiInd\brk{\Iendo{\cI}}$. 
\end{proof}

%%% Local Variables:
%%% mode: latex
%%% TeX-master: "paper"
%%% End:

\section{Noise-Tolerant Search over
  Endogenous Factors: Algorithmic Template}
\label{app:abstract_search}
% \subsubsection{Ensuring Optimality and Endogeneity}\label{sec: finding nearlt optimal endogenous factors}

% \dfcomment{Give people a TLDR for what this section is doing and how it relates to the other algorithms.}

In this section we provide a general template for designing error-tolerant algorithms that search over endogenous factors sets. This template is used in both $\EndoPolicyMaximizer^\eps_{t,h}$  and $\FactorDetectionAlg^\eps_{t,h}$ (subroutines of $\AlgName$).

Our algorithm design template, $\AMFD$ is presented in~\pref{alg: general approximate minimal factor}. Let us describe the motivation. Let $\cZ$ be an abstract ``dataset'' (typically, a collection of trajectories), let $\veps>0$ be a precision parameter, and let $\Condition(\Zcal, \eps, \cI)\in \cbr{\True,\False}$ be an abstract function defined over factor sets $\cI$. $\AMFD$ addresses the problem of finding an endogenous factor set $\cIhat\subseteq \Ic$ such that
\begin{equation}
  \Condition(\cZ,C\cdot\eps,\cIhat)=\True \label{eq: what we want AFS} 
\end{equation}
for a numerical constant $C\geq{}1$, assuming that the endogenous factors $\Ic$ satisfy the condition themselves:
\begin{align}
    \Condition(\Zcal, \eps, \Ic)=\True.\label{eq: Ic satisfies condition}
\end{align}
For example, within $\EndoPolicyMaximizer^{\eps}_{t,h}$, $\Condition(\cZ,\eps,\cI)$ checks whether policies that act on the factor set $\cI$ lead to $\eps$-optimal value for a given reward function (approximated using trajectories in $\cZ$).

% The goal of  $\AMFD$ is to find a certifiable endogenous factor set $\cI$ that satisfies the condition up to precision of $O\rbr{\eps}$. Formally, the goal of \AMFD is to return a factor set $\widehat{\cI}$ such that   $\widehat{\cI}\subseteq \Ic$ and for some  $C\in \mathbb{R}$ it holds that
% \begin{align}
%     \Condition(\Zcal, C\eps, \widehat{\cI})=\True \label{eq: what we want AFS} 
% \end{align}

$\AMFD$ begins with an initial set of endogenous factors $\cI_0\subseteq \Ic$. Naturally, since $\Ic\in \fullc_{\leq k}\rbr{\cI_0}$ and $\Ic$ is known to satisfy \eqref{eq: Ic satisfies condition}, a naive approach would be to enumerate over the collection $\fullc_{\leq k}\rbr{\cI_0}$ to find a factor set $\widehat{\cI}\in \fullc_{\leq k}\rbr{\cI_0}$ that satisfies~\eqref{eq: what we want AFS}. For example, considering the following procedure:
\begin{itemize}
    \item For each $\cI\in \fullc_{\leq k}\rbr{\cI_0}$, check whether $\Condition(\Zcal, C\eps, \cI)=\True$.
    \item If so, return $\widehat{\cI}\gets \cI$.
    \end{itemize}
    It is straightforward to see that this approach returns a factor set $\widehat{\cI}\in \fullc_{\leq k}\rbr{\cI_0}$ that satisfies~\eqref{eq: what we want AFS}, but the issue is that there is nothing preventing $\widehat{\cI}$ from containing exogenous factors. $\AMFD$ resolves this problem by searching for factors in a bottom-up fashion. The algorithm begins by searching over factor sets with minimal cardinality ($k'=\abr{\cI_0}$), and gradually increases the size until a factor set satisfying \pref{eq: what we want AFS} is found. 

    In more detail, observe that we have
\begin{align*}
    \fullc_{\leq k}\rbr{\cI_0} = \cup_{k'= \abr{\cI_0}}^k \fullc_{k}\rbr{\cI_0},
\end{align*}
where
\begin{align*}
    \fullc_{k}\rbr{\cI_0}\ldef \cbr{\cI' \subseteq[d] \mid \cI_0 \subseteq \cI',\ \abr{\cI'} = k}.
\end{align*}
Starting from $k'=\abs{\cI_0}$, $\AMFD$ checks whether exists a set of factors $\cI\in\fullc_{k'}\rbr{\cI_0}$ that satisfies $\Condition(\cdots)$ with respect to an accuracy parameter $\eps_{k'} = \rbr{1+1/k}^{k-k'}\eps$; this choice allows for larger errors for smaller $k'$. When a set of factors $\cI$ satisfies~\eqref{eq: what we want AFS} $\AMFD$ halts and returns this set; otherwise, $k'$ is increased. For this approach to succeed, we assume that $\Condition$ satisfies the following property.
\begin{assumption}
  \label{ass:AFS}
      For any set of factors $\cI=\Iendo{\cI}\cup \Iexo{\cI}$ with $\abr{\Iexo{\cI}}\geq 1$, it holds that
    \begin{align}
        \Condition(\Zcal,\eps_{\abs{\cI}},\cI)= \True \implies \Condition(\Zcal,\eps_{\abr{\Iendo{\cI}}},\Iendo{\cI}) = \True. \label{eq: main paper key contradiction}
    \end{align}
  \end{assumption}
% Otherwise, if no such set of factors is detected, it returns $\fail$. 

\begin{algorithm}[t]
\caption{\AMFD}
\label{alg: general approximate minimal factor}
\begin{algorithmic}[1]
\State \textbf{require:} abstract dataset $\Zcal$, precision $\eps$, initial endogenous factor $\cI_0\subseteq \Ic$.
\For{$k' = \abr{\cI_0},\abr{\cI_0}+1,\ldots, k$}
\State Set $\eps_{k'} = \rbr{1+1/k}^{k-k'} \eps$.
\For{$\cI\in \fullc_{k'}\rbr{\cI_0}$}
    \State \textbf{if} $\Condition(\Zcal, \eps_{k'}, \cI)= \True$ \textbf{then} \textbf{return} $\widehat{\cI}\gets \widehat{\cI}$.
\EndFor
\EndFor
\State \textbf{return}  $\fail$.
\end{algorithmic}
\end{algorithm}

We now describe three key steps used to prove that this scheme succeeds.
% we follow to establish the success of an algorithmic scheme that is based on \AMFD. Such proof technique is used to establish the correctness of both  $\EndoPolicyMaximizer^\eps_{t,h}$  and $\FactorDetectionAlg^\eps_{t,h}$. To show that \AMFD returns $\widehat{\cI}$ such that $(1)$ $\widehat{\cI}$ contains only endogenous factors, and, $(2)$ $\widehat{\cI}$ satisfies~\eqref{eq: what we want AFS} we establish the following.
\begin{enumerate}
    \item \emph{$\AMFD$ does not return $\fail.$}  This follows immediately from the assumption that \pref{eq: Ic satisfies condition} is satisfied.
    % This can be proved by construction of condition~\eqref{eq: what we want AFS} or by using the fact that $\Ic\in \fullc_{\leq k}\rbr{\cI_0}$ satisfies the conditions.
    % We prove this either by definition or by using the fact that $\Ic\in \fullc_{\leq k}\rbr{\cI_0}.$
    \item \emph{$\AMFD$ returns an endogenous set of factors.} 
    % Assuming~\eqref{eq: main paper key contradiction}, \AMFD is guaranteed to return an endogenous factor set. To see this,
      Observe that the assumption $\Ic\in \fullc_{\leq k }\rbr{\cI_0}$ implies that for any  $\cI\in \fullc_{\leq k }\rbr{\cI_0}$,  $\Iendo{\cI} \ldef \Ic\cap \cI \in \fullc_{\leq k }\rbr{\cI_0}$; this follows from \pref{lem: set of abstractions is a pi system}. Hence, if $\cI$ satisfies \eqref{eq: what we want AFS}, \pref{ass:AFS} implies that $\Iendo{\cI}$ satisfies \eqref{eq: what we want AFS} as well. Since \AMFD scans $\fullc_{\leq k }\rbr{\cI_0}$ in a bottom-up fashion, this means it must return an endogenous factor set, since it will verify that $\Iendo{\cI}$ satisfies~\eqref{eq: what we want AFS} prior to $\cI$.
      % Hence, \AMFD necessarily returns an endogenous factor set.
    
    % $\widehat{\cI}$ holds and $\cI_0\subseteq \Ic$ it implies that \AMFD returns 
    % we apply a contradiction-based argument to prove that \AMFD returns a set of endogenous factors. We sketch the argument assuming equation~\eqref{eq: main paper key contradiction} holds. First, observe that if $\cI\in \fullc_{\leq k }\rbr{\cI_0}$ then $\Iendo{\cI}\in \fullc_{\leq k }\rbr{\cI_0}$. This holds since $\Ic\in \fullc_{\leq k }\rbr{\cI_0}$ if $\cI_0\subseteq \Ic$, and that for any $\cI_1,\cI_2\in \fullc_{\leq k }\rbr{\cI_0}, \ \cI_1\cap \cI_2 \in \fullc_{\leq k }\rbr{\cI_0}$\footnote{I.e., $\fullc_{\leq k }\rbr{\cI_0}$ is a $\pi$-system, see~\pref{lem: set of abstractions is a pi system}.}. Since $\abr{\Iendo{\cI}} \leq \abr{\cI}-1$ and $\Iendo{\cI}\in \fullc_{\leq k}\rbr{\cI_0}$, \AMFD checks whether $\Iendo{\cI}$ satisfies the conditions prior to checking whether $\cI$ satisfies the conditions. Thus, $\Condition(\Zcal,\eps_{k'=\abr{\Iendo{\cI}}},\Iendo{\cI})=\False$; otherwise $\Iendo{\cI}$ is returned and $\Iendo{\cI}$ contains only endogenous factors. On the other hand, if $\cI$ is returned equation~\eqref{eq: main paper key contradiction} results in a contradiction to the fact $\Condition(\Zcal,\eps_{k'=\abr{\Iendo{\cI}}},\Iendo{\cI})=\False$.
    \item \emph{$\AMFD$ is near-optimal.} Since $(1+1/k)^{k-k'}\eps\leq 3\eps$ for all $k'\in [k]$, the factor set $\widehat{\cI}$ returned by $\AMFD$ satisfies $\Condition(\Zcal,3\eps, \widehat{\cI})=\True$.
\end{enumerate}

%%% Local Variables:
%%% mode: latex
%%% TeX-master: "paper"
%%% End:

\newpage

% \part{Omitted Algorithms: Descriptions and Formal Guarantees}
% \arxiv{\part{Omitted Algorithms}}
\part{Omitted Subroutines}\label{part:ommited_parts}

% \section{Finding a Near-Optimal Endogenous Policy:
%   $\EndoPolicyMaximizer^\eps_{t,h}$}
\section{Finding a Near-Optimal Endogenous Policy:
  $\EndoPolicyMaximizer$}
\label{app: near optimal endogenous policy}

\begin{algorithm}[htp]
\caption{$\EndoPolicyMaximizer_{t,h}^{\eps}$: One-Step Endogenous Policy Optimization}
\label{alg: approximate 1-step endogenous policy maximizer}
\begin{algorithmic}[1]
  \item[] \algcomment{Find an endogenous policy $\pi\in \PiInd\brk{\fullc_{\leq k}}$ that approximately maximizes  $V_{t,h}\rbr{\mu\circt\pi\circt[t+1]\psi}$, where $\mu\in \Pimix$ and $\psi\in \PiIndNS$ are fixed policies.}
%   \item[] \algcomment{.}
   \State {\bf require:}  
   \begin{itemize}
   \item Starting timestep $t$, end timestep $h$, and target precision $\epsilon\in (0,1)$.
   % \item 
   \item Collection $\crl[\big]{\widehat{V}_{t,h}\rbr{\mu\circt\pi\circt[t+1]\psi}}_{\pi\in \Pi\brk{\fullc_{\leq k}}}$ of estimates for $V_{t,h}\rbr{\mu\circt\pi\circt[t+1]\psi}$ for all $\pi\in \PiInd\brk{\fullc_{\leq k}}$.\loose
   \end{itemize}
   \For{$k'= 0,1,\cdots, k$}
        \State Let $\epsilon_{k'} = \rbr{1+1/k}^{k-k'}\epsilon$.
        \For{$\cI\in \fullc_{k'}$}
            \State \label{line: endo policy maximizer optimization step} Set
            $\iscover=\True$
            if             
            $$
              \max_{\pi\in \PiInd\brk{\fullc_{\leq{}k}}}  \widehat{V}_{t,h}\rbr{\mu\circt\pi\circt[t+1]\psi}
              \leq
              \max_{\pi\in \PiInd\sbr{\cI}} \widehat{V}_{t,h}\rbr{\mu\circt\pi\circt[t+1]\psi}
              +\epsilon_{k'}.
            $$
            \label{line: is_cover endopolicy}
            \State {\bf if} $\iscover=\True$ {\bf then} {\bf return:} $\widehat{\pi}\in \argmax_{\pi\in \PiInd\sbr{\cI}} \widehat{V}_{t,h}\rbr{\mu\circt\pi\circt[t+1]\psi}$.
            
       \EndFor
    \EndFor
   \State {\bf return:} $\fail$. 
%   \algcomment{Low probability failure event.}
\end{algorithmic}
\end{algorithm}

% \YEcomment{need to assume that $V_1$ sums elements starting from timestep $t$ to $h$ (?) replace $h\rightarrow H$}

In this section, we introduce and analyze the $\EndoPolicyMaximizer^\eps_{t,h}$ algorithm (\pref{alg: approximate 1-step endogenous policy maximizer}), which is used in the optimization phase of $\AlgName_{h}^{\eps,\delta}$ (\pref{app: learnin eps endogenous policy cover}) and in $\PSDPE$ (\pref{app: psdp in the presence of exogenous information}). In~\pref{app:endo_pol_max_description} we give a high-level description and intuition for the algorithm, and in \pref{app:proof_endo_opt} we prove the main theorem regarding its correctness and sample complexity.%formally provide the proof of its guarantees.

\subsection{Description of $\EndoPolicyMaximizer$.}\label{app:endo_pol_max_description}

The goal of $\EndoPolicyMaximizer^\eps_{t,h}$ is to return a policy $\widehat{\pi}\in \PiInd\brk*{\cI}$ such that:
% for some $\cI\in \fullc_{\leq k}$ such that:
\begin{enumerate}
    \item $\widehat{\pi}$ is endogenous in the sense that $\widehat{\pi}\in \PiInd\brk{\cI}$ for some $\cI\subseteq \Ic$.
    \item $\widehat{\pi}$ is near-optimal in the sense that
    \begin{align*}
        \max_{\pi\in \PiInd\brk{\fullc_{\leq k}}} V_{t,h}\rbr{\mu\circt\pi \circt[t+1]\psi} \leq V_{t,h}\rbr{\mu\circt \widehat{\pi} \circt[t+1]\psi} +O\rbr{\eps},
    \end{align*}
    where $V_{t,h}\rbr{\pi} \ldef \EE_\pi\sbr{\sum_{t'=t}^h r_t}$ for a given reward function $R$.
  \end{enumerate}
$\EndoPolicyMaximizer$ assumes access to approximate value functions $\widehat{V}_{t,h}\rbr{\mu\circt\pi\circt[t+1]\psi}$ that are $\eps$-close to the true value functions $V_{t,h}\rbr{\mu\circt \pi \circt[t+1]\psi}$. Given these approximate value functions, finding a near-optimal policy is trivial; it suffices to take the empirical maximizer $\widehat{\pi}\in \argmax_{\pi\in \PiInd\brk{\fullc_{\leq{}k}}} \widehat{V}_{t,h}\rbr{\mu\circt\pi\circt[t+1]\psi}$. However, finding a near-optimal \emph{endogenous} policy is a more challenging task. For this, $\EndoPolicyMaximizer$ applies the abstract endogenous factor search scheme described in \pref{app:abstract_search} ($\AMFD$), which regularizes toward factors with smaller cardinality.

$\EndoPolicyMaximizer^\eps_{t,h}$ splits the set $\fullc_{\leq{}k}$ as $\fullc_{\leq k} = \cup_{k'=0}^k \fullc_{k'}$, where $\fullc_{k'}$ is the collection of factor sets with cardinality exactly $k'\in [k]$, and follows the bottom-up search strategy in $\AMFD$. Beginning from $k'=0,\ldots,k$, the algorithm checks whether there exists a near-optimal policy in the class $\PiInd\brk*{\fullc_{k'}}$. If such a policy is found, the algorithm returns it, and otherwise it proceeds to $k'+1$.
% near-optimal policy is encountered for some $\cI$, it halts and returns the policy $\widehat{\pi}\in \Pi\brk{\cI}$.

\paragraph{Intuition for correctness.} We prove the correctness of the $\EndoPolicyMaximizer^\eps_{t,h}$ procedure by following the general template in~\pref{app:abstract_search}. In particular, we view \epm as a special case of the $\AMFD$ (\pref{alg: general approximate minimal factor}) scheme with
\[
\Condition(\cZ,\eps,\cI) = \indic\crl*{\max_{\pi\in \PiInd\brk*{\fullc_{\leq{}k}}}  \widehat{V}_{t,h}\rbr{\mu\circt\pi\circt[t+1]\psi}
              \leq
              \max_{\pi\in \PiInd\sbr{\cI}} \widehat{V}_{t,h}\rbr{\mu\circt\pi\circt[t+1]\psi}
              +\epsilon}.
\]
Most the effort in proving the correctness of the algorithm is in showing that this condition satisfies \pref{ass:AFS}.
% effort are devoted to showing
% to establish that the statement in~\eqref{eq: main paper key contradiction} that should hold for $\AMFD$ holds for $\EndoPolicyMaximizer^\eps_{t,h}$.
In particular, we need to show that if some $\cI\in \fullc_{\leq k}$ satisfies the condition in~\pref{line: endo policy maximizer optimization step},
\begin{align*}
  \max_{\pi\in \PiInd\brk*{\fullc_{\leq{}k}}}  \widehat{V}_{t,h}\rbr{\mu\circt\pi\circt[t+1]\psi}
              \leq
              \max_{\pi\in \PiInd\sbr{\cI}} \widehat{V}_{t,h}\rbr{\mu\circt\pi\circt[t+1]\psi}
              +\epsilon_{\abr{\cI}},
\end{align*}
then $\Iendo{\cI} \ldef \cI \cap \Ic$ also satisfies the condition in the sense that
\begin{align*}
    \max_{\pi\in \PiInd\brk*{\fullc_{\leq{}k}}}  \widehat{V}_{t,h}\rbr{\mu\circt\pi\circt[t+1]\psi}
              \leq
              \max_{\pi\in \PiInd\sbr{\Iendo{\cI}}} \widehat{V}_{t,h}\rbr{\mu\circt\pi\circt[t+1]\psi}
              +\epsilon_{\abr{\Iendo{\cI}}}.
\end{align*}
This can be shown to hold as a consequence of assumptions $(\mathrm{A}1)$ and $(\mathrm{A}2)$ in \pref{thm: correctness of endo policy maximizer}. Assumption $(\mathrm{A}1)$ asserts the following restriction property holds: For any $\cI$,
$$
\max_{\pi \in \Pi\brk{\cI}} V_{t,h}\rbr{\mu \circt \pi \circt[t+1] \psi} = \max_{\pi \in \Pi\brk{\Iendo{\cI}}} V_{t,h}\rbr{\mu \circt \pi \circt[t+1] \psi}.
$$
Hence, optimizing over a larger policy class that acts on exogenous factors does not improve the value. Assumption $(\mathrm{A}2)$ asserts that the estimates for $V_{t,h}\rbr{\mu \circt \pi \circt[t+1] \psi}$ are uniformly $\eps$-close, so that optimizing with respect to these estimates is sufficient.
% so that optimizing over
% Building on the proof technique described in~\pref{app:abstract_search}, we prove that $\EndoPolicyMaximizer^\eps_{t,h}$ returns a near-optimal policy that is also endogenous.

% We comment that formulating $\EndoPolicyMaximizer^\eps_{t,h}$ in such generality---such that it requires general value functions---allows us to utilize this algorithm for both $\AlgName$ and $\PSDPE$.

% \begin{remark}[Importance of the decoupling property of endogenous and exogenous factors]
\paragraphp{Importance of the decoupling property.}
We emphasize that assumption $(\mathrm{A}1)$ is non-trivial. We show it holds for several choices for the reward function in the \framework (\pref{lem: invariance of reduced policy cover} and~\pref{lem:restriction_endo_rewards}), which are used when we invoke the algorithm within $\AlgName$. However, the condition my not hold if the endogenous and exogenous factors are correlated. In this case, optimizing over exogenous state factors may improve the value, leading the algorithm to fail.
% $V_{t,h}\rbr{\mu \circt \pi \circt[t+1] \psi}$, and, in that case,
% \begin{align*}
%     \max_{\pi \in \Pi\brk{\cI}} V_{t,h}\rbr{\mu \circt \pi \circt[t+1] \psi} > \max_{\pi \in \Pi\brk{\Iendo{\cI}}} V_{t,h}\rbr{\mu \circt \pi \circt[t+1] \psi}.
% \end{align*}
% This stresses the importance of the decoupling property when applying the $\EndoPolicyMaximizer$ procedure.
% It relies on the decoupling of the initial distribution (see \eqref{eq: model assumption d}), as well as the fact that $\mu\ind{t}$ and $\psi$ are an endogenous policies that keeps the endogenous and exogenous state factors uncorrelated.
% \end{remark}

% Assuming $(\mathrm{A}1)$ and $(\mathrm{A}2)$ holds we can prove that if $\PiInd\brk*{\cI}$ satisfies the conditions~\pref{line: is_cover endopolicy} then also $\PiInd\brk{\Iendo{\cI}}$ is. Hence, \eqref{eq: main paper key contradiction} holds. 

\paragraph{Formal guarantee for $\EndoPolicyMaximizer$.}
The following result shows that $\EndoPolicyMaximizer^\eps_{t,h}$ returns a near-optimal endogenous policy.

% \dfcomment{It seems like this is missing a claim that a solution $\cIbar$ to the test in $\EndoPolicyCover$ actually exists (whp). I suppose this should be a corollary of Lemma 8.}

\begin{theorem}[Correctness of $\EndoPolicyMaximizer^\eps_{t,h}$]\label{thm: correctness of endo policy maximizer}
Fix $h\in [H]$ and $t\in [h]$. Let $\mu\in \Pimix$ and $\psi\in \PiIndNS$ be fixed policies. Assume the following conditions hold:
\begin{enumerate}[label=$(\mathrm{A}\arabic*)$]
    \item \emph{Restriction property:} For any set of factors $\cI$, $$\max_{\pi \in \Pi\brk{\cI}} V_{t,h}\rbr{\mu \circt \pi \circt[t+1] \psi} = \max_{\pi \in \Pi\brk{\Iendo{\cI}}} V_{t,h}\rbr{\mu \circt \pi \circt[t+1] \psi}.$$
    \item \emph{Quality of estimation.} For all $\pi\in\Pi\brk{\fullc_{\leq{}k}}$,
    $$
        \abr{V_{t,h}\rbr{\mu \circt \pi \circt[t+1] \psi} - \widehat{V}_{t,h}\rbr{\mu \circt \pi \circt[t+1] \psi}} \leq \eps/12k.
    $$
\end{enumerate}
Then the policy $\wh{\pi}$ output by $\EndoPolicyMaximizer^\eps_{t,h}$ satisfies the following properties:
\begin{enumerate}
    \item $\wh{\pi}$ is endogenous: $\widehat{\pi}\in \PiInd\sbr{\cI}$, where $\cI \subseteq \Ic$.
    \item $\widehat{\pi}$ is near-optimal:
    $
        \max_{\pi\in \PiInduced\sbr{\fullc_{\leq{}k}}} V_{t,h}\rbr{\mu\circt \pi \circ \psi} \leq V_{t,h}\rbr{\mu\circt \widehat{\pi} \circ \psi} + 4\epsilon.
    $
\end{enumerate}
\end{theorem}
% We follow the three-step proof recipe described in~\pref{app:abstract_search} for proving the correctness $\EndoPolicyMaximizer$.  

\subsection{Proof of \preftitle{thm: correctness of endo policy maximizer}}\label{app:proof_endo_opt}
We use the three-step proof recipe described in~\pref{app:abstract_search} to prove correctness of $\EndoPolicyMaximizer$.  
  \paragraph{Step 1: $\EndoPolicyMaximizer^\eps_{t,h}$ does not return $\fail$.} By definition, there exists $\cI\in \fullc_{\leq k}$ such that
\begin{align*}
              \max_{\pi\in \PiInd\brk*{\fullc_{\leq{}k}}}  \widehat{V}_{t,h}\rbr{\mu\circt\pi\circt[t+1]\psi}
              =
              \max_{\pi\in \PiInd\sbr{\cI}} \widehat{V}_{t,h}\rbr{\mu\circt\pi\circt[t+1]\psi}.
            \end{align*}
Thus,~\pref{line: is_cover endopolicy} is satisfied, since $\eps_{k'}\geq 0$.

\paragraph{Step 2: $\EndoPolicyMaximizer^\eps_{t,h}$ returns an endogenous policy.} Since $\EndoPolicyMaximizer^\eps_{t,h}$ does not return $\fail$, it returns a policy $\pihat\in\Pi\brk{\cI}$ for some factor set $\cI$.
We prove that $\cI$ is an endogenous factor set, which implies that $\widehat{\pi}$ is an endogenous policy. We show this by proving the following claim:
\begin{claim}
  \label{claim:epm}
  If $\cI$ satisfies the condition in~\pref{line: is_cover endopolicy} ($\iscover=\True$ for $\cI$), then $\Iendo{\cI}$ satisfies the condition  as well ($\iscover=\True$ for $\Iendo{\cI}$).
\end{claim}
Given this claim, it is straightforward to see that \epm returns an endogenous policy.
First, observe that for any $\cI\in \fullc_{\leq k}$, we have $\Iendo{\cI} \ldef \cI\cap \Ic\in \cI\in \fullc_{\leq k}$ by~\pref{lem: set of abstractions is a pi system} (since $\Ic\in \fullc_{\leq k}$). If $\abs{\Iendo{\cI}}<\abs{\cI}$, then $\EndoPolicyMaximizer^\eps_{t,h}$ verifies that $\Iendo{\cI}\in \fullc_{\leq k}$ satisfies~\pref{line: is_cover endopolicy} prior to verifying whether $\cI\in \fullc_{\leq k}$ satisfies the condition. It follows that the factor set returned by the algorithm must be endogenous.
% if $\cI$ satisfies the condition, then $\Iendo{\cI}$, which is an endogenous factor, is returned.
% We now prove \pref{claim:epm}.
% the statement that implies $\EndoPolicyMaximizer^\eps_{t,h}$ returns an endogenous policy.
% On the way to contradiction, assume that $\cI = \Iendo{\cI} \cup \Iexo{\cI}$ where $\Iexo{\cI}\neq \emptyset.$ 
\paragraphp{Proof of \pref{claim:epm}.}
Assume that $\cI$ contains at least one exogenous factor, so
\begin{align}
    \abr{\Iendo{\cI}} \leq \abr{\cI}-1. \label{eq: cardinality of I and Iendo}
\end{align}
Suppose that $\iscover=\True$ for $\cI$. By construction, it holds that for $k_1 \ldef \abr{\cI}\leq k $,
\begin{align}
     &\max_{\pi\in \PiInd\sbr{\fullc_{\leq{}k}}} \widehat{V}_{t,h}\rbr{\mu \circt \pi \circt[t+1] \psi}  \leq \max_{\pi\in \PiInd\sbr{\cI}} \widehat{V}_{t,h}\rbr{\mu \circt \pi \circt[t+1] \psi}  +\epsilon_{k_1}.\label{eq: direct application of policy certificate}
\end{align}
This statement, which holds for the approximate value $\widehat{V}_{t,h}\rbr{\mu \circt \pi \circt[t+1] \psi}$ implies a similar statement on the true value $V_{t,h}\rbr{\mu \circt \pi \circt[t+1] \psi}$.  Specifically,~\eqref{eq: direct application of policy certificate} together with~\pref{lem: near optimal maximizers and minimzers help lemma} (which can be applied using assumption $(\mathrm{A}2)$), implies that
\begin{align}
      \max_{\pi\in \PiInd\sbr{\fullc_{\leq{}k}} } V_{t,h}\rbr{\mu \circt \pi \circt[t+1] \psi} &\leq \max_{\pi\in \PiInd\sbr{\cI}} V_{t,h}\rbr{\mu \circt \pi \circt[t+1] \psi} +\epsilon_{k_1} +\epsilon/6k \nonumber \\
    &\stackrel{\mathrm{(a)}}{=}\max_{\pi\in \PiInd\sbr{\Iendo{\cI}}} V_{t,h}\rbr{\mu \circt \pi \circt[t+1] \psi}+ \epsilon_{k_1} +\epsilon/6k, \label{eq: towards contradiction 1}
\end{align}
and $\mathrm{(a)}$ holds by the restriction property in assumption $(\mathrm{A}1)$. 

We now relate the inequality in~\eqref{eq: towards contradiction 1}, which holds for the true values $V_{t,h}\rbr{\mu \circt \pi \circt[t+1] \psi}$, back to an inequality on the approximate values. Using~\pref{lem: near optimal maximizers and minimzers help lemma} and assumption $(\mathrm{A}2)$ on~\eqref{eq: towards contradiction 1}, we have htat
\begin{align}
    \max_{\pi\in \PiInd\sbr{\fullc_{\leq{}k}} } \widehat{V}_{t,h}\rbr{\mu \circt \pi \circt[t+1] \psi} &\leq \max_{\pi\in \PiInd\sbr{\Iendo{\cI}}} \widehat{V}_{t,h}\rbr{\mu \circt \pi \circt[t+1] \psi}+ \epsilon_{k_1} +\eps/3k \nonumber \\
    &\stackrel{\mathrm{(a)}}{\leq} \max_{\pi\in \PiInd\sbr{\Iendo{\cI}}} \widehat{V}_{t,h}\rbr{\mu \circt \pi \circt[t+1] \psi}+ \epsilon_{k_2},\label{eq: towards contradiction 2 endomaximizer}
\end{align}
where $\mathrm{(a)}$ holds for all $k_1,k_2\in [k]$ such that $k_2 \leq k_1 -1 $, since 
\begin{align*}
    \epsilon_{k_1} +\eps/3k  \ldef \rbr{1 + 1/k}^{k-k_1}\eps +\eps/3k \leq \rbr{1 + 1/k}^{k-k_2}\eps \ldef \eps_{k_2},
\end{align*}
by~\pref{lem: numerical verification}. Setting $k_2 =\abr{\Iendo{\cI}} \leq  k_1-1 =\abr{\cI}$ (the cardinality of $\Iendo{\cI}$ is strictly smaller than that of $\cI$ by~\eqref{eq: cardinality of I and Iendo}) and plugging this value into \eqref{eq: towards contradiction 2 endomaximizer} yields
\begin{align}
     \max_{\pi\in \PiInd\sbr{\fullc_{\leq{}k}} } \widehat{V}_{t,h}\rbr{\mu \circt \pi \circt[t+1] \psi} \leq \max_{\pi\in \PiInd\sbr{\Iendo{\cI}}} \widehat{V}_{t,h}\rbr{\mu \circt \pi \circt[t+1] \psi}+ \epsilon_{\abr{\Iendo{\cI}}}. \label{eq: Iendo is near optimal for maximization}
\end{align}
Hence, $\Iendo{\cI}$ also satisfies the conditions in~\pref{line: endo policy maximizer optimization step}.

\paragraph{Step 3: $\EndoPolicyMaximizer^\eps_{t,h}$ returns a near-optimal policy.}  When the condition of $\EndoPolicyMaximizer^\eps_{t,h}$ at~\pref{line: is_cover endopolicy} holds and $\iscover=\True$, the factor set $\cI$ satisfies 
\begin{align}
  \max_{\pi\in \PiInd\sbr{\fullc_{\leq{}k}}} \widehat{V}_{t,h}\rbr{\mu \circt \pi \circt[t+1] \psi} 
    &\leq \max_{\pi\in \Pi\sbr{\cI}} \widehat{V}_{t,h}\rbr{\mu \circt \pi \circt[t+1] \psi} +\epsilon_{\abr{\cI}} \nonumber\\
    &=  \widehat{V}_{t,h}\rbr{\mu \circt \widehat{\pi} \circt[t+1] \psi} +\epsilon_{\abr{\cI}} \nonumber\\
    &\leq \widehat{V}_{t,h}\rbr{\mu \circt \widehat{\pi} \circt[t+1] \psi} +3\epsilon, \label{eq: policy is near optimal consequence of is_cover}
\end{align}
where the last relation holds because $\epsilon_{\abr{\cI}} \leq \rbr{1+1/k}^{k}\eps\leq 3\eps$. Applying \pref{lem: near optimal maximizers and minimzers help lemma} with $(\mathrm{A2})$ then gives
\begin{align*}
    \max_{\pi\in \PiInd\sbr{\fullc_{\leq{}k}}} V_{t,h}\rbr{\mu \circt \pi \circt[t+1] \psi}  &\leq V_{t,h}\rbr{\mu \circt \widehat{\pi} \circt[t+1] \psi} +\underbrace{3\epsilon +\eps/6k}_{\leq 4\epsilon}.
\end{align*}
\qed

%%% Local Variables:
%%% mode: latex
%%% TeX-master: "paper"
%%% End:

\newpage
% \section{Selecting Endogenous Factors with Strong Coverage:
%   $\FactorDetectionAlg^\eps_{t,h}$}
\section{Selecting Endogenous Factors with Strong Coverage:  $\FactorDetectionAlg$}
\label{app: no false detection of factors}

\begin{algorithm}[htp]
\caption{$\FactorDetectionAlg_{t,h}^{\eps}$: Simultaneous Policy Cover for all Factors}
\label{alg: sufficient exploration check}
\begin{algorithmic}[1]
  \item[] \algcomment{Find $\cI$ such that reaching $\cI$ implicitly leads to
  good coverage for all $\cJ\in\GenericCollection_{\leq{}k}(\cI\ind{t+1,h})$.}
  \State {\bf require:} 
  \begin{itemize}
  \item Starting timestep $t$ and end timestep $h$, target precision $\epsilon\in (0,1)$.
  \item Set of endogenous factors $\cI\ind{t+1,h}\subseteq\Ic$.
  \item Collection of policy sets $\cbr{\Gamma\ind{t}\sbr{\cI}}_{\cI\in
       \scrI_{\leq{}k}(\cI\ind{t+1,h})}$, where $$\Gamma\ind{t}\brk{\cI} = \crl[\big]{\pi\ind{t}_{s\brk{\cI}}\mid{}s\brk{\cI}\in\cS\brk{\cI}}.$$
   % \item Set of policies $\Psi^{(t+1,h}$ that is in correspondence
   %   with $\cI^{(t+1,h}$ \YEcomment{define correspondence}
   \item Set of $(t+1\to{}h)$ policies
     $$\Psi\ind{t+1,h}=\crl[\big]{\psi\ind{t+1,h}_{s\sbr{\cI\ind{t+1,h}}}\mid{}s\brk{\cI\ind{t+1,h}}\in\cS\brk{\cI\ind{t+1,h}}}.$$
   \item Collection $\EmpOccupancy$ of approximate occupancy measures for layer $h$ under the sampling process $\mut\circt \pi\circt[t+1] \psi\ind{t+1,h}$.
   \end{itemize}
   \item[] \algcomment{Pick $\Psi\in \Gamma\ind{t}\brk[\big]{\cI}\circt[t+1] \Psi\ind{t+1:h}$ that explores $\cI\subseteq \Ic$ and sufficiently explores other factors.}
      \For{$k'= |\cI\ind{t+1,h}|,|\cI\ind{t+1,h}|+1,..,k$} 
            \State Define $\epsilon_{k'} = \rbr{1+1/k}^{k-k'}5\epsilon$. 
            \For{$\cI\in \fullc_{k'}(\cI\ind{t+1,h})$}
            \item[]\algspace\algspace\algcomment{Test whether reaching states in $\cI$ leads to good
              coverage for all factors $\cIbar\in\fullc_{\leq{}k}(\cI\ind{t+1,h})$.}
              \State \label{line: factor detection conditions} \multiline{Set $\sufficientcover=\True$ if for all
              $\cJ\in \fullc_{\leq{}k}(\cI\ind{t+1,h})$ and
              for all
              $
                s\sbr{\cJ} \in
                \Scal\sbr{\cJ}:
              $
                    
                          \begin{align}
                            &\max_{\pi\in \PiInd\brk{\fullc_{\leq k}}}
                              \dhat_h\rbr{s\brk{\cJ}\midsem \mut\circt\pi \circt[t+1] \psi\ind{t+1,h}_{s\brk{\cI\ind{t+1,h}}}} \nonumber \\
                                  &\quad \leq \dhat_h                                 \rbr{s\sbr{\cJ} \midsem \mut\circt \pi\ind{t}_{s\brk{\cJ\cap  \cI}} \circt[t+1] \psi\ind{t+1,h}_{s\brk{\cI\ind{t+1,h}}}
                                  } +  \epsilon_{k'}, \label{eq: sufficient coverage condition}
                          \end{align}
                          where $\pi\ind{t}_{s\brk{\cJ\cap
                              \cI} } \in \Gamma\ind{t}\brk{\cJ\cap
                            \cI}$.  \algcomment{Recall $\displaystyle\pi\ind{t}_{s\brk{\cJ\cap  \cI} }\approx \argmax_{\pi\in \PiInd\brk{\fullc_{\leq k}}}\dhat_h\prn[\Big]{s\sbr{\cJ\cap \cI}\midsem\mut\circ_t\pi \circt[t+1] \psi\ind{t+1,h}_{s\brk{\cI\ind{t+1,h}}}}$.}}
                \If{$\sufficientcover=\True$}
                    \State $\widehat{\cI} \gets \cI$.
                    \State {\bf return} $(\widehat{\cI},\Gamma\ind{t}\brk{\widehat{\cI}})$.
                \EndIf 
                
            \EndFor
        \EndFor
        \State {\bf return:} $\fail$. \algcomment{Low probability failure event.}
\end{algorithmic}
\end{algorithm}

% \dfcomment{Put short overview for how the section is organized (and consider adding more subsections if this helps)}

In this section, we describe and analyze the $\FactorDetectionAlg_{t,h}^\eps$ algorithm (\pref{alg: sufficient exploration check}). $\FactorDetectionAlg_{t,h}^\eps$ is a subroutine used in the selection phase of $\AlgName_{h}^{\eps,\delta}$, and generalizes the selection phase used in $\ExactAlgName_h$ to the setting where only approximate occupancy measures are available.
% the latter can query exact values of the state occupancy measures, whereas $\FactorDetectionAlg_{t,h}^\eps$ has access to approximations of the state occupancy measures.  
In \pref{app:fact_det_goals}, we give a high-level description $\FactorDetectionAlg_{t,h}^\eps$, give intuition, and state the main theorem concerning its performance. Then, in \pref{app:proof_fact_det_sel} we prove this result.

\subsection{Description of $\FactorDetectionAlg$}\label{app:fact_det_goals}

To motivate $\FactorDetectionAlg_{t,h}^\eps$, let us first recall the selection phase of $\ExactAlgName_h$ (\pref{line: exact DP-SR state2} of \pref{alg: exact dynamics state refinement}). The selection phase assumes access to a collection of policy sets $\cbr{\Gamma\ind{t}\brk{\cI}}_{\cI\in \fullc_{\leq k}\rbr{\cI\ind{t+1,h}}}$, which are calculated in the optimization step. In particular, for each set $\cI$ and each $s\brk*{\cI}\in\cS\brk*{\cI}$, $\pi\ind{t}_{s\brk{\cI}}\in \Gamma\ind{t}\brk{\cI}$ is an endogenous policy that maximizes the probability of reaching $s\brk{\cI}$ at layer $h$ in the following sense:
\[
\pi\ind{t}_{s\brk{\cI}}\in \argmax_{\pi\in \PiInd\brk{\fullc_{\leq k}}} d_h\rbr{s\brk{\cI} \midsem \mut\circt\pi\circt[t+1]\psi\ind{t+1,h}_{s\brk{\cI\ind{t+1,h}}} }.
\]
The selection phase of $\ExactAlgName_h$ find the factor set $\widehat{\cI}\in \fullc_{\leq k}\rbr{\cI\ind{t+1,h}}$ of minimal size such that for all $\cJ\in \fullc_{\leq k}\rbr{\cI\ind{t+1,h}}$ and $s\brk{\cJ}\in \cS\brk{\cJ}$,
\begin{align}
\max_{\pi\in \PiInd\brk{\fullc_{\leq k}}}  d_h\rbr{s\brk{\cJ} \midsem \mut\circt\pi\circt[t+1]\psi\ind{t+1,h}_{s\brk{\cI\ind{t+1,h}}} } = d_h\rbr{s\brk{\cJ} \midsem \mut\circt \pi\ind{t}_{s\brk{\cJ \cap \widehat{\cI}}}\circt[t+1]\psi\ind{t+1,h}_{s\brk{\cI\ind{t+1,h}}}}. \label{eq:fact_det_app_decrip1}
\end{align}
At the end of the selection step, $\ExactAlgName_h$ outputs the tuple $(\widehat{\cI}, \Gamma\ind{t}\brk{\widehat{\cI}})$. Since $\widehat{\cI}$ is chosen as the \emph{minimal} factor set that satisfies~\eqref{eq:fact_det_app_decrip1} it can be shown it is  an endogenous factors set. Furthermore, $\Gamma\ind{t}\brk{\widehat{\cI}}$ satisfies condition~\eqref{eq:fact_det_app_decrip1}. 

$\FactorDetectionAlg_{t,h}^\eps$ is similar to $\ExactAlgName_h$, but only requires access to \emph{approximate state occupancy measures}.  Analogous to $\ExactAlgName_h$, the algorithm outputs a tuple $(\widehat{\cI}, \Gamma\ind{t}\brk{\widehat{\cI}})$, where $\widehat{\cI}$ is an endogenous factors set and $\Gamma\ind{t}\brk{\cIhat}$ ensures good coverage at layer $h$.%. in the sense of \pref{eq:fact_det_app_decrip1}.
However, since $\FactorDetectionAlg_{t,h}^\eps$ has only has access to approximate state occupancy measures, the policy set $\Gamma\ind{t}\brk{\cIhat}$ returned by the algorithm is only guaranteed to satisfy an approximate version of ~\eqref{eq:fact_det_app_decrip1}:
\begin{align}
&\max_{\pi\in \PiInd\brk{\fullc_{\leq k}}}  d_h\rbr{s\brk{\cJ} \midsem \mut\circt\pi\circt[t+1]\psi\ind{t+1,h}_{s\brk{\cI\ind{t+1,h}}} } \nonumber \\
&\quad \leq  d_h\rbr{s\brk{\cJ} \midsem \mut\circt \pi\ind{t}_{s\brk{\cJ \cap \widehat{\cI}}}\circt[t+1]\psi\ind{t+1,h}_{s\brk{\cI\ind{t+1,h}}}} +O(\eps), \label{eq:fact_det_app_decrip2}
\end{align}
where $\pi\ind{t}_{s\brk{\cJ \cap \widehat{\cI}}}\in \Gamma\ind{t}\brk{\widehat{\cI}}$.

To ensure find an endogenous factor set $\widehat{\cI}$ such that $\Gamma\ind{t}\brk{\widehat{\cI}}$ satisfies~\eqref{eq:fact_det_app_decrip2}, $\FactorDetectionAlg_{t,h}^\eps$ follows the $\AMFD$ scheme described in~\pref{app:abstract_search}. It enumerates the collection of factor sets $\fullc_{\leq k}\rbr{\cI\ind{t+1,h}}$ in a bottom-up fashion---starting from factor sets of minimimal cardinality--- and checks whether each factor set approximately satisfies the optimality condition. 

\paragraph{Intuition for correctness.} To establish the correctness of  $\FactorDetectionAlg_{t,h}^\eps$, we view the algorithm as an instance of $\AMFD$ with 
% \dfcomment{plz double check this}\yecomment{yep--that's great to explicitly write the reduction}
\begin{align*}
  &\Condition(\cZ,\eps,\cI)\\
  &= \indic\left\{
    \begin{aligned}
      &\max_{\pi\in \PiInd\brk{\fullc_{\leq k}}}
                              \dhat_h\rbr{s\brk{\cJ}\midsem \mut\circt\pi \circt[t+1] \psi\ind{t+1,h}_{s\brk{\cI\ind{t+1,h}}}}\\
                              &~~\leq \dhat_h                                 \rbr{s\sbr{\cJ} \midsem \mut\circt \pi\ind{t}_{s\brk{\cJ\cap  \cI}} \circt[t+1] \psi\ind{t+1,h}_{s\brk{\cI\ind{t+1,h}}}
                                  } +  \epsilon,
                                \end{aligned}
                                \quad\forall{}\cJ\in\fullc(\cI\ind{t+1,h}),s\brk{\cJ}\in\cS\brk{\cJ}
    \right\},
\end{align*}
and recall that $\pi\ind{t}_{s\brk{\cJ\cap  \cI}}\in \Gamma\ind{t}\brk{\cJ\cap \cI}$ is the output from the optimization step at $\EndoPolicyMaximizer$. The analysis of $\FactorDetectionAlg_{t,h}^\eps$ follow the recipe sketched in~\pref{app:abstract_search}. Most of our efforts are devoted to proving that the condition in \eqref{eq: main paper key contradiction} required by $\AMFD$ holds for $\FactorDetectionAlg^\eps_{t,h}$. In particular, we wish to prove the following claim:
% \begin{claim}
%   \label{claim:efs}
\emph{If $\cI$ satisfies the condition in \pref{line: factor detection conditions} ($\sufficientcover=\True$ for $\cI$), then $\Iendo{\cI}$ satisfies the condition as well ($\sufficientcover=\True$ for $\Iendo{\cI}$).}
% \end{claim}
To show that the statement is true, we use a key structural result, \pref{lem: approximately reaching endo part is approximately optimal}, which generalizes certain structural results used in the analysis of $\ExactAlgName$ (\pref{prop: exact case recovery of 1-step endo policy cover}). Let $\mu$ and $\rho$ be endogenous policies, and consider a fixed state factor $s\brk{\cI}\in \cS\brk{\cI}$.~\pref{lem: approximately reaching endo part is approximately optimal} asserts that if an endogenous policy $\pi_{s\brk{\Iendo{\cI}}}$ approximately maximizes the probability of reaching the endogenous part of $s\brk{\cI}$, which is given by
\begin{align*}
   d_h\rbr{s\brk{\Iendo{\cI}} \midsem \mu\circt \pi_{s\brk{\Iendo{\cI}}} \circt[t+1] \rho},
\end{align*}
then the policy also approximately maximizes the probability of reaching $s\brk{\cI}$, which is given by
\begin{align*}
   d_h\rbr{s\brk{\cI} \midsem \mu\circt\pi_{s\brk{\Iendo{\cI}}} \circt[t+1] \rho}.
\end{align*}
Hence, to approximately maximize the probability of reaching $s\brk{\cI}$, it suffices to execute a policy that approximately maximizes the probability of reaching the endogenous part of the state, $s\brk{\cIen}$. We use this observation to show that exogenous factors are redundant in the sense that  if $\sufficientcover=\True$ for $\cI$, then $\sufficientcover=\True$ for $\Iendo{\cI}$; this proves the claim

\paragraph{Formal guarantee for $\FactorDetectionAlg$}
The following result is the main guarantee for $\FactorDetectionAlg^\eps_{t,h}$.
\begin{theorem}[Success of $\FactorDetectionAlg^{\eps}_{t,h}$]\label{thm:  no false detection of controllable factors}
Fix $h\in [H]$ and $t\in [h]$. Assume the following conditions hold:
\begin{enumerate}[label=$(\mathrm{A}\arabic*)$]
\item \emph{Endogeneity of arguments.} $\mu\ind{t}\in \Pimix\brk{\Ic}$ is endogenous, $\Psi\ind{t+1,h}$ contains only endogenous policies, and 
  $\Gamma\ind{t}\brk*{\cI}$ contains only endogenous policies for all $\cI\in \fullc_{\leq k}\rbr{\cI\ind{t+1,h}}$.
  In addition, $\cI\ind{t+1,h}\subseteq\Ic$.
    \item \emph{Quality of estimation.} $\EmpOccupancy$ is a collection of $\epsilon/12k$-approximate state occupancy measures with respect to $\rbr{\mu\ind{t}\circ\Pi\brk{\fullc_{\leq{}k}} \circ \Psi\ind{t+1,h}, \fullc_{\leq{}k}\rbr{\cI\ind{t+1,h}},h}$ (\pref{def: approximate set of occupancy measures}). 
    % \dfcomment{reference definition}
    \item \emph{Optimality for $\Gamma\ind{t}\brk{\cI}$.} For any factor set$\cI\in \fullc_{\leq k}\rbr{\cI\ind{t+1,h}}$ and any $s\brk*{\cI}\in \cS\brk*{\cI}$, the policy $\pi\ind{t}_{s\sbr{\cI}}\in \Gamma\ind{t}\sbr{\Ical}$ satisfies the following optimality guarantee:
      \begin{align*}
        \max_{\pi\in \PiInd\brk{\fullc_{\leq k}}} d_h\rbr{s\sbr{\cI} \midsem \mu\ind{t}\circt \pi \circt[t+1] \psi\ind{t+1,h}_{s\sbr{\cI\ind{t+1,h}}}} \leq d_h\rbr{s\sbr{\cI} \midsem \mu\ind{t}\circt \pi\ind{t}_{s\sbr{\cI}} \circt[t+1] \psi\ind{t+1,h}_{s\sbr{\cI\ind{t+1,h}}}} + 4\epsilon.
    \end{align*}
\end{enumerate}
% \begin{enumerate}
%     \item The set $\cI\ind{t+1,h}$ contains only endogenous factors.
%     \item The set $\Psi\ind{t+1,h}$ contains endogenous policies.
%     \item $\EmpOccupancy$ is $\epsilon/12k$-approximate with respect to $\rbr{\Pi\sbr{\fullc_{\leq{}k}\rbr{\cI\ind{t+1,h}}}, \fullc_{\leq{}k}\rbr{\cI\ind{t+1,h}},h}$.
% \end{enumerate}
Then $\FactorDetectionAlg_{t,h}^{\eps}$ does not output $\fail$, and the tuple $(\widehat{\cI}, \Gamma\ind{t}\brk{\widehat{\cI}})$ output by the algorithm satisfies the following guarantees:
\begin{enumerate}
    \item $\widehat{\cI}\subseteq \Ic$.
    \item For all $s\sbr{\Ic}\in \Scal\sbr{\Ic}$, we have
    \begin{align*}
      \max_{\pi\in \PiInd\brk{\Ic}}d_h\rbr{s\sbr{\Ic} ; \mut\circt \pi\circt[t+1]\psi\ind{t+1,h}_{s\sbr{\cI\ind{t+1,h}}}} - d_h\rbr{s\sbr{\Ic} ;  \mut\circt \pi\ind{t}_{s\brk{\widehat{\cI}}}\circt[t+1]\psi\ind{t+1,h}_{s\sbr{\cI\ind{t+1,h}}}} \leq 16\epsilon,
    \end{align*}
where we note that we can write
$s\sbr{\Ic} = \prn{s\brk[\big]{\widehat{\cI}},s\brk[\big]{\Ic \setminus{} \widehat{\cI}}} = \rbr{s\sbr{\cI\ind{t+1,h}},s\sbr{\Ic \setminus{} \cI\ind{t+1,h}} }$
because $\cI\ind{t+1,h},\widehat{\cI}\subseteq \Ic$.
    % Let $\psi\ind{t+1,h}_{s\sbr{\cI\ind{t+1,h}}}\in \Psi\ind{t+1,h}$ and $ \pi\ind{t}_{s\brk{\widehat{\cI}}}\in \Gamma\ind{t}\brk*{\widehat{\cI}}$.
\end{enumerate}
\end{theorem}

\subsection{Proof of \preftitle{thm:  no false detection of controllable factors}}\label{app:proof_fact_det_sel}

We use the three-step proof strategy described in~\pref{app:abstract_search} to prove correctness for \efs.  

\paragraph{Step 1: $\FactorDetectionAlg^\eps_{t,h}$ does not return $\fail$.} We show that given assumptions $(\mathrm{A}1)-(\mathrm{A}3)$ $\FactorDetectionAlg^\eps_{t,h}$  does not return fail. First, observe that $\Ic \in \fullc_{\leq{}k}\rbr{\cI\ind{t+1,h}}$, since $\cI\ind{t+1,h}\subseteq \Ic$ by $(\mathrm{A}1)$ and $\abr{\Ic}\leq k$ by assumption. We prove that $\FactorDetectionAlg^\eps_{t,h}$ halts for $\cI \gets \Ic$; meaning that $\Ic$ satisfies the condition at~\pref{line: factor detection conditions} of $\FactorDetectionAlg^\eps_{t,h}$.

% \dfcomment{Why are we talking about EndoPolicyCover here? This is not called in the algorithm.}

% \emph{Consequences of $\EndoPolicyMaximizer$} By~\pref{corr: correctness of endo policy cover}, assuming $(\mathrm{A}1),(\mathrm{A}2)$ and $(\mathrm{A}3)$, for any $\cI\in \fullc_{\leq{}k}\rbr{\cI\ind{t+1,h}}$, $\EndoPolicyCover$ (see~\pref{line: endo policy cover for all abstractions}) returns a policy class $\Gamma\ind{t}\brk*{\cI}$ such that for all $$s\sbr{\cI} = \rbr{s\sbr{\cI\ind{t+1,h}},s\sbr{\cI/\cI\ind{t+1,h}}}\in \Scal\sbr{\cI}$$ it holds that:
% \begin{enumerate}
%     \item $\pi_{s\sbr{\cI}}\in \Gamma\ind{t}\brk*{\cI}$ is an endogenous policy, and,
%     \item $\pi_{s\sbr{\cI}}\in \Gamma\ind{t}\brk*{\cI}$ is a near-optimal:
%     \begin{align*}
%         \max_{\pi\in \PiInd\brk{\fullc_{\leq k}}} d_h\rbr{s\sbr{\cI} \midsem \mu\ind{t}\circt \pi \circt[t+1] \psi\ind{t+1,h}_{s\sbr{\cI\ind{t+1,h}}}} \leq d_h\rbr{s\sbr{\cI} \midsem \mu\ind{t}\circt \pi_{s\sbr{\cI}} \circt[t+1] \psi\ind{t+1,h}_{s\sbr{\cI\ind{t+1,h}}}} + 4\epsilon.
%     \end{align*}
% \end{enumerate}

% \paragraphp{Reaching endogenous part is sufficient.}
Fix $\cI\in \fullc_{\leq{}k}\rbr{\cI\ind{t+1,h}}$ and $s\brk*{\cI}\in \cS\brk{\cI}$.  Let $\Iendo{\cI} \in  \fullc_{\leq{}k}\rbr{\cI\ind{t+1,h}}$\footnote{\label{fot1} $\Iendo{\cI} \in \fullc_{\leq{}k}\rbr{\cI\ind{t+1,h}}$ since $\Ic\in  \fullc_{\leq{}k}\rbr{\cI\ind{t+1,h}}$ and $\fullc_{\leq{}k}\rbr{\cI\ind{t+1,h}}$ is a $\pi$-system by~\pref{lem: set of abstractions is a pi system}. } be the endogenous component of $\cI$, so that $s\brk*{\cI} = (s\brk*{\Iendo{\cI}},s\brk*{\Iexo{\cI}})$. Consider the policy $\pi\ind{t}_{s\sbr{\Iendo{\cI}}} \in \Gamma\ind{t}\brk{\Iendo{\cI}}$. By assumption $(\mathrm{A}3)$, $\pi\ind{t}_{s\sbr{\Iendo{\cI}}}$ is endogenous and satisfies
    \begin{align}
        \max_{\pi\in \PiInd\brk{\fullc_{\leq k}}} &d_h\rbr{s\sbr{\Iendo{\cI}} \midsem \mu\ind{t}\circt \pi \circt[t+1] \psi\ind{t+1,h}_{s\sbr{\cI\ind{t+1,h}}}} \nonumber\\
        \leq \ & d_h\rbr{s\sbr{\Iendo{\cI}} \midsem \mu\ind{t}\circt \pi\ind{t}_{s\sbr{\Iendo{\cI}}} \circt[t+1] \psi\ind{t+1,h}_{s\sbr{\cI\ind{t+1,h}}}} + 4\epsilon. \label{eq:step_1_factor_seletion}
    \end{align}
\eqref{eq:step_1_factor_seletion} shows that $\pi\ind{t}_{s\sbr{\Iendo{\cI}}}$ has near-optimal probability for the endogenous component of $s\brk{\cI}$ near optimally (when the rollout policy $\psi\ind{t+1,h}_{s\sbr{\cI\ind{t+1,h}}}$ is fixed). Combined with the fact that both $\pi\ind{t}_{s\sbr{\Iendo{\cI}}}$ and $\psi_{s\sbr{\Iendo{\cI}}}\ind{t+1,h}$ are endogenous (by $(\mathrm{A1})$), this allows us to apply~\pref{lem: approximately reaching endo part is approximately optimal}, which asserts that $\pi\ind{t}_{s\sbr{\Iendo{\cI}}}$ reaches the any state factor $s\brk{\cI}$ with $\Iendo{\cI}\subseteq\cI$ near-optimally as well. In particular,
\begin{align}
    \max_{\pi\in \PiInd\brk{\fullc_{\leq k}}} &d_h\rbr{s\sbr{\cI} \midsem \mu\ind{t}\circt \pi \circt[t+1] \psi\ind{t+1,h}_{s\sbr{\cI\ind{t+1,h}}}} \nonumber \\
    \leq\ &d_h\rbr{s\sbr{\cI} \midsem \mu\ind{t}\circt \pi\ind{t}_{s\sbr{\Iendo{\cI}}} \circt[t+1] \psi\ind{t+1,h}_{s\sbr{\cI\ind{t+1,h}}}} + 4\epsilon. \label{eq: endgenous policy cover is sufficient}
\end{align}
% \paragraphp{$\Ic$ satisfies the conditions at~\pref{line: factor detection conditions} of $\FactorDetectionAlg$.}
Now, observe that since $\EmpOccupancy$ is $\epsilon/12k$-approximate with respect to $\rbr{\Pi\sbr{\fullc_{\leq{}k}\rbr{\cI\ind{t+1,h}}}, \fullc_{\leq{}k}\rbr{\cI\ind{t+1,h}},h}$ (cf. $(\mathrm{A}2)$), \eqref{eq: endgenous policy cover is sufficient} and~\pref{lem: near optimal maximizers and minimzers help lemma} imply that
\begin{align}
    \max_{\pi\in \PiInd\brk{\fullc_{\leq k}}} &\dhat_h\rbr{s\sbr{\cI} \midsem \mu\ind{t}\circt \pi \circt[t+1] \psi\ind{t+1,h}_{s\sbr{\cI\ind{t+1,h}}}}\nonumber \\
    \leq\ &\dhat_h\rbr{s\sbr{\cI} \midsem \mu\ind{t}\circt \pi\ind{t}_{s\sbr{\Iendo{\cI}}} \circt[t+1] \psi\ind{t+1,h}_{s\sbr{\cI\ind{t+1,h}}}} + 5\epsilon. \label{eq: endgenous policy cover is sufficient rel 2}
\end{align}
Since $\Iendo{\cI} = \cI \cap \Ic$, and since 
\begin{align*}
    5\eps \leq \rbr{1+1/k}^{k-k'+1}5\eps
\ldef \eps_{k'}
\end{align*}
for all $k'\in [k]$, this implies that the condition at~\pref{line: factor detection conditions} of $\FactorDetectionAlg^\eps_{t,h}$ is satisfied by $\Ic$.% by~\eqref{eq: endgenous policy cover is sufficient rel 2}.

\paragraph{Step 2: Proof of first claim ($\widehat{\cI}\subseteq \Ic$ is a set of endogenous factors).} Since $\FactorDetectionAlg^\eps_{t,h}$ does not return $\fail$, it necessarily returns a pair $(\cIhat, \Gamma\ind{t}\brk{\cIhat})$. We now show that $\cIhat$ is endogenous.
% exists $\widehat{\cI}=\Ical$ for which the condition in~\pref{line: factor detection conditions} of~$\FactorDetectionAlg^\eps_{t,h}$ holds. 
To do so, we prove the following claim.
\begin{lemma}
  \label{claim:efs}
  If $\cI \in \fullc_{\leq{}k}\rbr{\cI\ind{t+1,h}}$ satisfies the condition in \pref{line: factor detection conditions} ($\sufficientcover=\True$ for $\cI$), then $\Iendo{\cI}$ satisfies the condition as well ($\sufficientcover=\True$ for $\Iendo{\cI}$). 
\end{lemma}
% show the following: if $\cI \in \fullc_{\leq{}k}\rbr{\cI\ind{t+1,h}}$ satisfies the condition in~\pref{line: factor detection conditions}, then its endogenous component  $\Iendo{\cI} \ldef \cI\cap \Ic$ also satisfies this condition.
Conditioned on \pref{claim:efs}, the result quickly follows. Observe that for any $\cI\in \fullc_{\leq{}k}\rbr{\cI\ind{t+1,h}}$, we have $\Iendo{\cI}\in  \fullc_{\leq{}k}\rbr{\cI\ind{t+1,h}}$\footref{fot1}. Furthermore, if $\abr{\Iendo{\cI}} > \abr{\cI}$, then $\FactorDetectionAlg$ will check whether $\Iendo{\cI}$ satisfies the condition in~\pref{line: factor detection conditions} prior to checking whether $\cI$ satisfies it. Thus, $\FactorDetectionAlg$ necessarily returns a set of endogenous factors; it remains to prove \pref{claim:efs}.

\begin{proof}[\pfref{claim:efs}]
Fix $\cI \in \fullc_{\leq{}k}\rbr{\cI\ind{t+1,h}}$ with $\Iexo{\cI}\neq \emptyset$. Assume that $\cI$ satisfies the conditions in~\pref{line: factor detection conditions}. That is, for $k_1 \ldef \abr{\cI}\leq k$, it holds that for all $\cJ\in \fullc_{\leq{}k}(\cI\ind{t+1,h})$ and all $s\sbr{\cJ} = \rbr{s\sbr{\cI},s\sbr{\cJ\setminus{}\cI}} \in \Scal\sbr{\cJ}$,
\begin{align}
     \max_{\pi\in \PiInd\brk{\fullc_{\leq k}}}               & \dhat_h\rbr{s\brk{\cJ}\midsem \mut\circt\pi \circt[t+1] \psi\ind{t+1,h}_{s\brk{\cI\ind{t+1,h}}}}       \nonumber \\              
     \leq\ &
                                  \dhat_h                                 \rbr{s\sbr{\cJ} \midsem \mut\circt \pi\ind{t}_{s\brk{\cJ\cap  \cI}} \circt[t+1] \psi\ind{t+1,h}_{s\brk{\cI\ind{t+1,h}}}
                                  } +  \epsilon_{k_1},\label{eq: direct application of factor certificate proof}
\end{align}
where $\pi\ind{t}_{s\brk{\cJ\cap\cI} } \in \Gamma\ind{t}\brk{\cJ\cap\cI}$. We will show that this implies that $\Iendo{\cI}$ also satisfies the conditions in~\pref{line: factor detection conditions}. 

\paragraphp{$\Iendo{\cI}$ satisfies the conditions in~\pref{line: factor detection conditions}.}  Since $\cI$ satisfies~\eqref{eq: direct application of factor certificate proof} for all $\cJ\in \fullc_{\leq{}k}(\cI\ind{t+1,h})$, it must also satisfy the condition for all $\Iendo{\cJ}\subseteq\cJ$.
% and for all $s\sbr{\cJ}\in\cS\brk{\cJ}$, it must also hold for all $\Iendo{\cJ}\in \cbr{\cJ \in \fullc_{\leq{}k}(\cI\ind{t+1,h}) \mid \cJ\subseteq \Ic}$, all sets of endogenous factors contained in~$\fullc_{\leq{}k}(\cI\ind{t+1,h}),$ and all $s\brk*{\Iendo{\cJ}} \in \cS\brk*{\Iendo{\cJ}}$.
Fix $\cJ\in\fullc_{\leq{}k}(\cI\ind{t+1,h})$. Then for all $s\brk*{\Iendo{\cJ}} \in \cS\brk*{\Iendo{\cJ}}$, we have
\begin{align}
  &\max_{\pi\in \PiInd\brk{\fullc_{\leq k}}}               \dhat_h\rbr{s\brk{\Iendo{\cJ}}\midsem \mut\circt\pi \circt[t+1] \psi\ind{t+1,h}_{s\brk{\cI\ind{t+1,h}}}}  \nonumber \\                     
  &\leq \dhat_h                                 \rbr{s\sbr{\Iendo{\cJ}} \midsem \mut\circt \pi\ind{t}_{s\brk{\Iendo{\cJ}\cap  \cI}} \circt[t+1] \psi\ind{t+1,h}_{s\brk{\cI\ind{t+1,h}}} } +  \epsilon_{k_1} \nonumber \\
  &\stackrel{\mathrm{(a)}}{\leq} 
                                  \dhat_h                                 \rbr{s\sbr{\Iendo{\cJ}} \midsem \mut\circt \pi\ind{t}_{s\brk{\Iendo{\cJ}\cap  \Iendo{\cI}}} \circt[t+1] \psi\ind{t+1,h}_{s\brk{\cI\ind{t+1,h}}} } +  \epsilon_{k_1}, \label{implicit consequence 1 factor detection}    
\end{align}
where $\mathrm{(a)}$ follows because $\Iendo{\cJ} \cap \Ical = \Iendo{\cJ} \cap \Iendo{\cI}$. 

Since $(\mathrm{A}2)$ asserts that $\EmpOccupancy$ is $\epsilon/12k$-approximate with respect to $\rbr{\Pi\sbr{\fullc_{\leq{}k}\rbr{\cI\ind{t+1,h}}}, \fullc_{\leq{}k}\rbr{\cI\ind{t+1,h}},h}$, we can relate the inequality above to the analogous inequality for the true occupancies using~\pref{lem: near optimal maximizers and minimzers help lemma}. After multiplying both sides by $d_h\rbr{s\brk{\Iexo{\cJ}}}\in[0,1]$, this yields
\begin{align}
    &d_h\rbr{s\brk{\Iexo{\cJ}}} \max_{\pi\in \PiInd\brk{\fullc_{\leq k}}}                d_h\rbr{s\brk{\Iendo{\cJ}}\midsem \mut\circt\pi \circt[t+1] \psi\ind{t+1,h}_{s\brk{\cI\ind{t+1,h}}}} \nonumber \\
    &\leq d_h\rbr{s\brk{\Iexo{\cJ}}} d_h                                 \rbr{s\brk{\Iendo{\cJ}} \midsem \mut\circt \pi\ind{t}_{s\brk{\Iendo{\cJ}\cap  \Iendo{\cI}}} \circt[t+1] \psi\ind{t+1,h}_{s\brk{\cI\ind{t+1,h}}}  } +\epsilon_{k_1}   + \eps/6k.   \label{eq:exact_value_inequality_fact_detect}
\end{align}
We now manipulate both sides~\eqref{implicit consequence 1 factor detection} to relate these quantities to the occupancy measure for $s\brk{\cJ}$. This is done by appealing to the decoupling property for occupancy measures of endogenous policies (\pref{app:structural_occupancy}). To begin, for the \lhs of~\eqref{eq:exact_value_inequality_fact_detect}, we have
\begin{align}
    &d_h\rbr{s\brk{\Iexo{\cJ}}} \max_{\pi\in \PiInd\brk{\fullc_{\leq k}}}                d_h\rbr{s\brk{\Iendo{\cJ}}\midsem \mut\circt\pi \circt[t+1] \psi\ind{t+1,h}_{s\brk{\cI\ind{t+1,h}}}} \nonumber \\
    &\stackrel{\mathrm{(a)}}{=} d_h\rbr{s\brk{\Iexo{\cJ}}} \max_{\pi\in \PiInd\brk{\Ic}}                d_h\rbr{s\brk{\Iendo{\cJ}}\midsem \mut\circt\pi \circt[t+1] \psi\ind{t+1,h}_{s\brk{\cI\ind{t+1,h}}}}  \nonumber\\
    &\stackrel{\mathrm{(b)}}{=} \max_{\pi\in \PiInd\brk{\Ic}}                d_h\rbr{s\brk{\cJ}\midsem \mut\circt\pi \circt[t+1] \psi\ind{t+1,h}_{s\brk{\cI\ind{t+1,h}}}}  \nonumber\\
    &\stackrel{\mathrm{(c)}}{=} \max_{\pi\in \PiInd\brk{\fullc_{\leq k}}}                d_h\rbr{s\brk{\cJ}\midsem \mut\circt\pi \circt[t+1] \psi\ind{t+1,h}_{s\brk{\cI\ind{t+1,h}}}},\label{eq:rhs_exact_ineq_factor_detec}
\end{align}
where relations $\mathrm{(a)}$ and $\mathrm{(c)}$ hold by~\pref{lem: endognous set is the maximizer} and relation $\mathrm{(b)}$ holds by~\pref{lem: decoupling of future state dis}; note that the assumptions of these lemmas hold because $\mut$ and $\psi\ind{t+1,h}_{s\brk{\cI\ind{t+1,h}}}$ are assumed to be endogenous, and because $\pi\in \Pi\brk{\Ic}$ is also endogenous.
% Relation $\mathrm{(b)}$ holds by, which can be applied since since $\mu\ind{t},\pi$
Moving on, we analyze the \rhs of~\eqref{eq:exact_value_inequality_fact_detect}. We have
\begin{align}
    &d_h\rbr{s\brk{\Iexo{\cJ}}} d_h                                 \rbr{s\brk{\Iendo{\cJ}} \midsem \mut\circt \pi\ind{t}_{s\brk{\Iendo{\cJ}\cap  \Iendo{\cI}}} \circt[t+1] \psi\ind{t+1,h}_{s\brk{\cI\ind{t+1,h}}}  } \nonumber\\
  &=d_h                                 \rbr{s\brk{\cJ} \midsem \mut\circt \pi\ind{t}_{s\brk{\Iendo{\cJ}\cap  \Iendo{\cI}}} \circt[t+1] \psi\ind{t+1,h}_{s\brk{\cI\ind{t+1,h}}}  },\label{eq:lhs_exact_ineq_factor_detec}
\end{align}
by~\pref{lem: decoupling of future state dis} (the assumptions of the lemma hold because $\mut,\pi\ind{t}_{s\brk{\Iendo{\cJ}\cap  \Iendo{\cI}}}$ and $\psi\ind{t+1,h}_{s\brk{\cI\ind{t+1,h}}}$ are endogenous). Plugging~\eqref{eq:lhs_exact_ineq_factor_detec} and~\eqref{eq:rhs_exact_ineq_factor_detec} back into~\eqref{eq:exact_value_inequality_fact_detect}, we have that
\begin{align}
  &\max_{\pi\in \PiInd\brk{\fullc_{\leq k}}}                d_h\rbr{s\brk{\cJ}\midsem \mut\circt\pi \circt[t+1] \psi\ind{t+1,h}_{s\brk{\cI\ind{t+1,h}}}} \nonumber \\
    &\leq  d_h                                 \rbr{s\brk{\cJ} \midsem \mut\circt \pi\ind{t}_{s\brk{\Iendo{\cJ}\cap  \Iendo{\cI}}} \circt[t+1] \psi\ind{t+1,h}_{s\brk{\cI\ind{t+1,h}}}  } +\epsilon_{k_1}   + \eps/6k. \label{eq:implicit consequence 1 factor detection central}
\end{align}
It remains to relate this to the analogous inequality for the approximate occupancy measures. Since $\EmpOccupancy$ is $\epsilon/12k$-approximate with respect to $\rbr{\Pi\sbr{\fullc_{\leq{}k}\rbr{\cI\ind{t+1,h}}}, \fullc_{\leq{}k}\rbr{\cI\ind{t+1,h}},h}$ by $(\mathrm{A}2)$, \pref{lem: near optimal maximizers and minimzers help lemma}, and \eqref{eq:implicit consequence 1 factor detection central} imply that
\begin{align}
  &\max_{\pi\in \PiInd\brk{\fullc_{\leq k}}}                    \dhat_h\rbr{s\brk{\cJ}\midsem \mut\circt\pi \circt[t+1] \psi\ind{t+1,h}_{s\brk{\cI\ind{t+1,h}}}}                       \nonumber  \\
  &\leq
                                   \dhat_h                                 \rbr{s\brk{\cJ} \midsem \mut\circt \pi\ind{t}_{s\brk{\cJ\cap  \Iendo{\cI}}} \circt[t+1] \psi\ind{t+1,h}_{s\brk{\cI\ind{t+1,h}}}  } +\epsilon_{k_1}   + \eps/3k \nonumber \\
  &\stackrel{\mathrm{(a)}}{\leq }  
                                   \dhat_h                                 \rbr{s\brk{\cJ} \midsem \mut\circt \pi\ind{t}_{s\brk{\cJ\cap  \Iendo{\cI}}} \circt[t+1] \psi\ind{t+1,h}_{s\brk{\cI\ind{t+1,h}}}  } +\epsilon_{k_2},
                                   \label{eq:implicit consequence 2 factor detection central}   
\end{align}
where $\mathrm{(a)}$ holds for all $k_1,k_2\in [k]$ such that $k_2 \leq k_1 -1 $, since 
\begin{align*}
    \epsilon_{k_1} +\eps/3k  \ldef \rbr{1 + 1/k}^{k-k_1} 5\eps + \eps/3k \leq \rbr{1 + 1/k}^{k-k_2}5\eps \ldef \eps_{k_2}
\end{align*}
by~\pref{lem: numerical verification} (with $c=5$). Since $\abr{\Iendo{\cI}}< \abr{\cI} \ldef k_1$, we can set $k_2=\abr{\Iendo{\cI}}$ in~\eqref{eq:implicit consequence 2 factor detection central}, which implies that
\begin{align}
    \hspace{-0.5cm}\max_{\pi\in \PiInd\brk{\fullc_{\leq k}}}                    \dhat_h\rbr{s\brk{\cJ}\midsem \mut\circt\pi \circt[t+1] \psi\ind{t+1,h}_{s\brk{\cI\ind{t+1,h}}}}                         \leq
                                   \dhat_h                                 \rbr{s\brk{\cJ} \midsem \mut\circt \pi\ind{t}_{s\brk{\cJ\cap  \Iendo{\cI}}} \circt[t+1] \psi\ind{t+1,h}_{s\brk{\cI\ind{t+1,h}}}  } +\epsilon_{\abr{\Iendo{\cI}}}.
                                   \label{eq:implicit consequence 3 factor detection central}  
\end{align}
Since~\eqref{eq:implicit consequence 3 factor detection central} holds for all $\cJ\in \fullc_{\leq{}k}(\cI\ind{t+1,h})$ and $ s\sbr{\cJ}\in\cS\brk{\cJ}$, this yields the result.
\end{proof}
\paragraph{Step 3: Proof of second claim ($\Gamma\ind{t}\brk{\widehat{\cI}}$ is near-optimal).} This claim is a direct consequence of the condition in~\pref{line: factor detection conditions}. Let $\widehat{\cI}$ be the output of $\FactorDetectionAlg^\eps_{t,h}$.  Since $\sufficientcover=\True$, then the conditions at~\pref{line: factor detection conditions} are satisfied, and for all $\cJ\in \fullc_{\leq{}k}(\cI\ind{t+1,h})$, for all $s\sbr{\cJ} =\rbr{ s\sbr{\cI\ind{t+1}},s\sbr{\cJ\setminus{}\cI\ind{t+1,h}} }\in
\Scal\sbr{\cJ}:$
\begin{align}
\max_{\pi\in \PiInd\brk{\fullc_{\leq k}}}
  &\dhat_h\rbr{s\brk{\cJ}\midsem \mut\circt\pi \circt[t+1] \psi\ind{t+1,h}_{s\brk{\cI\ind{t+1,h}}}} \nonumber \\
  \leq \ &\dhat_h                                 \rbr{s\sbr{\cJ} \midsem \mut\circt \pi\ind{t}_{s\brk{\cJ\cap  \widehat{\cI}}} \circt[t+1] \psi\ind{t+1,h}_{s\brk{\cI\ind{t+1,h}}}
      } +  15\epsilon, \label{eq: sufficient coverage condition proof }
\end{align}
where $\pi\ind{t}_{s\brk{\cJ\cap\cI} } \in \Gamma\ind{t}\brk{\cJ\cap\cI}$; the upper bound holds because $\epsilon_{k'} \ldef \rbr{1+1/k}^{k-k'}5\eps\leq 15\eps$ for all $k'\in[k]$. Applying \eqref{eq: sufficient coverage condition proof } with $\cJ\gets \Ic\in \fullc_{\leq{}k}(\cI\ind{t+1,h})$, and using~\pref{lem: near optimal maximizers and minimzers help lemma} (which is admissible by assumption $(\mathrm{A}2)$), we have that for all $s\sbr{\Ic}\in
\Scal\sbr{\Ic}$,
\begin{align*}
\max_{\pi\in \PiInd\brk{\fullc_{\leq k}}} &
  d_h\rbr{s\brk{\Ic}\midsem \mut\circt\pi \circt[t+1] \psi\ind{t+1,h}_{s\brk{\cI\ind{t+1,h}}}}  \\
  \leq \ & d_h                                 \rbr{s\sbr{\Ic} \midsem \mut\circt \pi\ind{t}_{s\brk{\Ic \cap  \widehat{\cI}}} \circt[t+1] \psi\ind{t+1,h}_{s\brk{\cI\ind{t+1,h}}}
      } +  16\epsilon\\
 \stackrel{\mathrm{(a)}}{\leq} \ & d_h                                 \rbr{s\sbr{\Ic} \midsem \mut\circt \pi\ind{t}_{s\brk{\widehat{\cI}}} \circt[t+1] \psi\ind{t+1,h}_{s\brk{\cI\ind{t+1,h}}}
      } +  16\epsilon,
\end{align*}
where $\mathrm{(a)}$ holds because $\Ic \cap  \widehat{\cI} =\widehat{\cI}$, since $\widehat{\cI}\subseteq \Ic$ by the first claim.
\qed

\newpage

\section{\PSDP with Exogenous Information: \PSDPE}
\label{app: psdp in the presence of exogenous information}

\begin{algorithm}[htp]
\caption{$\PSDPE$: $\PSDP$ with Exogenous Information}
\label{alg: psdp with exogenous noise}
\begin{algorithmic}[1]
  \State {\bf require:}
  \begin{itemize}
  \item Target precision $\eps\in(0,1)$ and failure probablitity $\delta\in (0,1)$.
    \item Collection $\cbr{\Psi\ind{h}}_{h=2}^H$ of endogenous $\eta/2$-approximate policy covers.
\end{itemize}

    \State {\bf initialize:}
    \begin{itemize}
    \item Let $N = C\cdot{} A S^{4k} H^2 k^3 \log\rbr{\frac{dSAH}{\delta}} \epsilon^{-2}$ for sufficiently large constant $C>0$ and $\epsnot=\frac{\eps}{2S^kH}$.
    \item For all $t\in [H]$, define $\mu\ind{t} \ldef \unf\rbr{\Psi\ind{t}}$.
    \item Let $\widehat{\pi}\ind{H,H} = \emptyset$.
  \end{itemize}

    \For{$t=H-1,..,1$}
    \item[]\algspace {\algcommentbig{Estimate average value functions via importance weighting.}}
    \State Get dataset $\cbr{(s_{t,n}, a_{t,n}, \cbr{r_{t',n}}_{t'=1}^H)}_{n=1}^N$ by executing
        $\mut\circ_{t} \unf(\Acal)\circ_{t+1}
        \widehat{\pi}\ind{t+1,H}$. 
    \State \label{line: PSDPE estimation} Estimate the $(t\rightarrow H)$ value for all $\pi\in \Pi\brk*{\fullc_{\leq k}}$ via importance weighting:
    \begin{align*}
        \widehat{V}_{t,H}\rbr{\mut\circ_{t}\pi\circ_{t+1}\widehat{\pi}_{t+1:H}} = \frac{1}{N}\sum_{n=1}^N \frac{\indic\cbr{a_{t,n} = \pi\rbr{s_{t,n}}}}{1/A}\rbr{\sum_{t'=t}^H r_{t',n} }.
    \end{align*}
  \item[]\algspace {\algcommentbig{Apply policy optimization with estimated value functions.}}
    \State \label{line: PSDPE optimization step}$\widehat{\pi}\ind{t} \gets \EndoPolicyMaximizer_{t,h}^{\epsnot}\prn[\Big]{\crl[\big]{ \widehat{V}_{t,H}\rbr{\mut\circ_{t} \pi\circ_{t+1}
        \widehat{\pi}_{t+1:H}}}_{\pi\in \Pi\sbr{\scrI_{\leq{}k}}} }$.
    \State $\widehat{\pi}\ind{t,H} = \widehat{\pi}\ind{t} \circt[t+1] \widehat{\pi}\ind{t+1,H}$. 
    \EndFor
    \State {\bf return:} $\widehat{\pi}\ind{1,H}$.
\end{algorithmic}
\end{algorithm}

In this section we present and analyze the $\PSDPE$ algorithm (\pref{alg: psdp with exogenous noise}). $\PSDPE$ is based on the classical $\PSDP$ algorithm~\citep{bagnell2004policy}, but incorporates modifications to ensure that the policies produced are endogenous. In~\pref{app:desc_ex_psdp}, we motivate $\PSDPE$ and state the main guarantee concerning its performance (\pref{thm:ex_PSDP_sample_comp}). Then, in \pref{app:proof_psdp_exo}, we prove this result. 
% besides of its optimization step at~\pref{line: PSDPE optimization step}.
% Before describe $\PSDPE$, we give an overview on the $\PSDP$ algorithm.

% in greater details and explain the need for an alternative for the standard $\PSDP$ in the \framework setting.

\subsection{Description of $\PSDPE$}\label{app:desc_ex_psdp}

The $\PSDPE$ algorithm solves the following problem:
\begin{quote}
  \emph{Given a collection of endogenous policy covers $\crl*{\Psi\ind{t}}_{t=1}^{H}$ for an \framework $\cM$, find a policy $\pihat$ that is $\eps$-optimal in the sense that $J(\pihat)\geq{}\max_{\pi}J(\pi) - \eps$.
  }
\end{quote}
To motivate the approach behind the algorithm, we first remind the reader of the classical $\PSDP$ algorithm.

\paragraph{Background on $\PSDP$.} Suppose we have a set of mixture policies $\cbr{\mu\ind{h}}_{h=1}^H$ that ensure good coverage at every layer for an MDP $\cM$, and our goal is to optimize the MDP's reward function.
% where $\mu\ind{h}$ ensures good coverage at layer $h$.
The $\PSDP$ algorithm \citep{bagnell2004policy} addresses this problem by using the dynamic programming principle to learn a near-optimal policy through a series of backward steps $t=H,\ldots,1$. Assume access to a policy class $\PiInd.$ At each step $t$, assuming that step $t+1$ has already produced a near-optimal $(t+1)\to{}H$ policy $\widehat{\pi}\ind{t+1,H}$, the algorithm estimates the value function $V_{t,H}\rbr{\mut\circ_{t}\pi\circ_{t+1}\widehat{\pi}_{t+1:H}}$ for all $\pi\in \PiInd$ where (see also~\eqref{eq:expected_value_def})
\begin{align*}
  V_{t,H}\rbr{\mut\circ_{t}\pi\circ_{t+1}\widehat{\pi}_{t+1:H}} \ldef \EE_{\mut\circt \pi \circt[t+1] \widehat{\pi}_{t+1:H} }\brk*{ \sum_{t'=t}^H r_{t'}}.
\end{align*}
The estimates are calculated via importance-weighting by 
% \dfcomment{Remind people where the true value function $V_{t,H}\rbr{\mut\circ_{t}\pi\circ_{t+1}\widehat{\pi}_{t+1:H}}$ is defined.}
\[
     \widehat{V}_{t,H}\rbr{\mut\circ_{t}\pi\circ_{t+1}\widehat{\pi}_{t+1:H}} = \frac{1}{N}\sum_{n=1}^N \frac{\indic\cbr{a_{t,n} = \pi\rbr{s_{t,n}}}}{1/A}\rbr{\sum_{t'=t}^H r_{t',n} }
\]
where the data is generated by rolling in with $\mu\ind{t},$ taking random action on the $t^{\mathrm{th}}$ time-step and rolling out with $\widehat{\pi}_{t+1:H}$ using $N$ trajectories. Then, $\PSDP$ computes 
\begin{align}
\pi\ind{t}\in \argmax_{\pi\in \PiInd}  \widehat{V}_{t,H}\rbr{\mut\circ_{t}\pi\circ_{t+1}\widehat{\pi}_{t+1:H}}, \label{eq: PSDP stanard near optimal}    
\end{align}
and sets $\widehat{\pi}\ind{t,h} = \pi\ind{t} \circt[t] \widehat{\pi}\ind{t+1,H}$. The final policy $\pihat\ldef{}\pihat\ind{1,H}$ is guaranteed to be near-optimal as long as $\cbr{\mu\ind{h}}_{h=1}^H$ have good coverage.

%\dfcomment{We discuss Q-functions above but V-functions below and in the algorithm. This needs more discussion or should be changed.} YE: Changed it the value functions

\paragraph{Insufficiency of vanilla $\PSDP$.}
% The optimization step for the classical $\PSDP$ scheme above can be written as solving
% \begin{align}
%     \widehat{\pi}\ind{t} \in  \argmax_{\pi\in \PiInd} \widehat{V}_{t,H}\rbr{\mut\circ_{t} \pi\circ_{t+1}
%         \widehat{\pi}_{t+1:H}}.\label{eq: PSDP stanard near optimal}
% \end{align}
% As long as the number of trajectories $N$ is sufficiently large (typically, $\poly(\log\abs{\Pi},H,\eps^{-1})$), this implies $\pi\ind{t}$ is near-optimal in the following sense:
% \begin{align}
%     \max_{\pi\in \PiInd} V_{t,H}\rbr{\mut\circ_{t} \pi\circ_{t+1}
%         \widehat{\pi}_{t+1:H}}\leq V_{t,H}\rbr{\mut\circ_{t} \widehat{\pi}\ind{t}\circ_{t+1}
%         \widehat{\pi}_{t+1:H}} + O(\eps).
% \end{align}
The first issue with applying $\PSDP$ to the \framework model is that, if we want the policy class $\Pi$ to contain all possible policies, we will have $\abr{\Pi} = \Theta\prn{A^{S^d}}$, which leads to sample complexity scaling with $\log\abs{\Pi} = \Omega(\poly(S^d))$; this is prohibitively large. An alternative policy class one my hope can address this issue is $\Pi\brk{\fullc_{\leq k}}$. Indeed, this class has much smaller cardinality: $\abr{\Pi\brk{\fullc_{\leq k}}}= \Theta\prn{d^k A^{S^k}}$. However, for an \framework, naively optimizing over this class via~\eqref{eq: PSDP stanard near optimal} may lead to roll-out policies $\widehat{\pi}_{t+1:H}$ that depend on the exogenous state factors, since there is no mechanism in place to ensure endogeneity. This in turn may invalidate the \emph{realizability assumption} needed to apply standard $\PSDP$ (see~\citet{misra2020kinematic}, Assumption 2). In particular, $\PSDP$ requires that the policy class $\Pi$ contains the optimal policy in the sense that 
\begin{align}
    \max_{\pi\in \PiIndNS}V_{t,H}\rbr{\mut\circ_{t}\pi\circ_{t+1}\widehat{\pi}_{t+1:H}} = \max_{\pi\in \PiInd}V_{t,H}\rbr{\mut\circ_{t}\pi\circ_{t+1}\widehat{\pi}_{t+1:H}}. \label{eq:realizability_psdp_explicit}
\end{align}
If the roll-out policy $\widehat{\pi}_{t+1:H}$ depends on the exogenous state factors, then the optimal policy that maximizes $V_{t,H}\rbr{\mut\circ_{t}\pi\circ_{t+1}\widehat{\pi}_{t+1:H}}$ may depend on exogenous state factors as well. Then,~\eqref{eq:realizability_psdp_explicit} may be violated when instantiating $\PSDP$ with the policy class $\Pi\brk{\fullc_{\leq k}}$. \loose
% \dfcomment{Add citation for realizabiility assumption.}\loose
% , and, hence, cannot be represented with only the endogenous state factors $\Ic$.

\paragraph{A solution: $\PSDPE$.}
To address the issues above, $\PSDPE$ applies an alternative to the optimization step in \pref{eq: PSDP stanard near optimal}. In particular, $\PSDPE$ uses the sub-routine $\EndoPolicyMaximizer$ (see~\pref{line: PSDPE optimization step}), which finds an \emph{endogenous} near-optimal policy. In particular, as long as $\widehat{\pi}\ind{t+1,H}$ is endogenous, which can be guaranteed inductively, $\EndoPolicyMaximizer$, will succeed in finding an endogenous policy at step $t$.
% implies that $\widehat{\pi}\ind{t}$ returned by $\EndoPolicyMaximizer$ is $(1)$  a near-optimal with similar guarantee as in \eqref{eq: PSDP stanard near optimal}, and $(2)$ an endogenous policy. 
Importantly, since (i) the reward in a \framework depends only on the endogenous factors, and (ii) the policy  $\widehat{\pi}\ind{t+1,H}$ is endogenous (by the guarantees of  $\EndoPolicyMaximizer$), $\widehat{\pi}\ind{t}$ can be shown to be near-optimal with respect to the entire policy class $\PiInd$. Hence, in spite of optimizing over the restricted policy class $\scrI_{\leq{}k}$, we are able to find a near-optimal policy with respect set of all policies. Using this argument inductively allows us to prove that $\widehat{\pi}\ind{1,H}$ is near-optimal and endogenous.

\begin{theorem}[Main guarantee for $\PSDPE$]\label{thm:ex_PSDP_sample_comp}
  \label{thm:psdp}
  Suppose that the sets $\crl*{\Psi\ind{t}}_{t=1}^{H}$ passed into $\PSDPE$ are endogenous $\eta/2$-approximate policy covers for all $t$. Then, for any $\eps,\delta>0$, with probability at least $1-\delta$, 
\begin{enumerate}
    \item $\widehat{\pi}\ind{1,H}$ is endogenous.
    \item $\widehat{\pi}\ind{1,H}$ is $\eps$-optimal in the sense that 
    \[
        \max_{\pi\in \PiIndNS}\Jpi \leq \Jpi[\widehat{\pi}\ind{1,H}] + \eps.
    \]
  \end{enumerate}
  Furthermore, the algorithm uses at most $N = \bigoh\prn[\Big]{\frac{AH^4k^3 S^{3k} \log\rbr{\frac{dAH}{\delta}}}{\eps^2}}$ trajectories.
  % trajectories where  $\epsnot\ldef \eps/(2S^kH)$.
\end{theorem}
\subsection{Proof of \preftitle{thm:psdp}}\label{app:proof_psdp_exo}

% \dfcomment{remind people how $Q_t$ is defined, and why it only depends on endogenous state}

Fix a pair of endogenous policies $\pi,\widehat{\pi}\in \PiIndNS\brk*{\Ic}$. Further, let $\Mcal_{\mathrm{en}} = \rbr{\Scal,\Acal,\Tc, R_{s\brk{\Ic}},H, \dc}$ denote the restriction of the \framework to its endogenous component, and let $Q^\pi_{t,\mathrm{en}}(s\brk{\Ic},a)$ denote the associated state-action value function for $\cM_{\en}$. 

We decompose the difference in performance as follows.
\begin{align}
    &J(\pi) - J(\widehat{\pi}) \nonumber  \\
    &\stackrel{\mathrm{(a)}}{=} \sum_{t=1}^h \EE_{\pi} \sbr{Q^{\widehat{\pi}}_{t,\mathrm{en}}(s_t\sbr{\Ic},\pi_{t}(s_t\sbr{\Ic}) - Q^{\widehat{\pi}}_{t,\mathrm{en}}(s_t\sbr{\Ic},\widehat{\pi}_{t}(s_t\sbr{\Ic}))}  \nonumber\\
    &\leq \sum_{t=1}^h \EE_{s\brk{\Ic}\sim d_{t}\rbr{\cdot \midsem \pi}} \sbr{\max_a Q^{\widehat{\pi}}_{t,\mathrm{en}}(s\sbr{\Ic},a) - Q^{\widehat{\pi}}_{t,\mathrm{en}}(s\sbr{\Ic},\widehat{\pi}_{t}(s\sbr{\Ic}))}  \nonumber\\
    % &\leq  \sum_{t=1}^h \sum_{s\sbr{\Ic}:  \PiIndNS\brk*{\Ic}d_{t}(s\sbr{\Ic} \midsem \pi)> 0} \max_{\pi\in \PiInd\brk{\Ic}}d_{t}(s\sbr{\Ic} \midsem \pi) \rbr{\max_a Q^{\widehat{\pi}}_{t}(s\sbr{\Ic},a) - Q^{\widehat{\pi}}_{t}(s\sbr{\Ic},\widehat{\pi}_{t}(s\sbr{\Ic}))} . \nonumber\\
    &\leq  \sum_{t=1}^h \sum_{s\sbr{\Ic}\in \Scal\brk{\Ic}} \max_{\pi\in \PiInd\brk{\Ic}}d_{t}(s\sbr{\Ic} \midsem \pi) \rbr{\max_a Q^{\widehat{\pi}}_{t,\mathrm{en}}(s\sbr{\Ic},a) - Q^{\widehat{\pi}}_{t,\mathrm{en}}(s\sbr{\Ic},\widehat{\pi}_{t}(s\sbr{\Ic}))} . \nonumber\\
    &\stackrel{\mathrm{(b)}}{\leq}  2S^k\sum_{t=1}^h \sum_{s\sbr{\Ic}\in \Scal\brk{\Ic}} d_{t}(s\sbr{\Ic} \midsem \mut)\rbr{\max_a Q^{\widehat{\pi}}_{t,\mathrm{en}}(s\sbr{\Ic},a) - Q^{\widehat{\pi}}_{t,\mathrm{en}}(s\sbr{\Ic},\widehat{\pi}_{t}(s\sbr{\Ic}))} . \nonumber\\
    &= 2S^k\sum_{t=1}^h \EE_{\mut}\sbr{\max_a Q^{\widehat{\pi}}_{t,\mathrm{en}}(s_t\sbr{\Ic},a) - Q^{\widehat{\pi}}_{t,\mathrm{en}}(s_t\sbr{\Ic},\widehat{\pi}_{t,\mathrm{en}}(s_t\sbr{\Ic}))}. \nonumber\\
    &\stackrel{\mathrm{(c)}}{=}  2S^k\sum_{t=1}^h \max_{\pi'\in \PiInd\brk{\Ic}}\EE_{\mut}\sbr{ Q^{\widehat{\pi}}_{t,\mathrm{en}}(s_t\sbr{\Ic},\pi'\rbr{s_t\sbr{\Ic}}) - Q^{\widehat{\pi}}_{t,\mathrm{en}}(s_t\sbr{\Ic},\widehat{\pi}_{t}(s_t\sbr{\Ic}))}. \nonumber\\
     & = 2S^k\sum_{t=1}^h \max_{\pi'\in \PiInd\brk{\Ic}}V_{t,H}\rbr{\mut \circt \pi'\circt[t+1] \widehat{\pi}} -  V_{t,H}\rbr{\mut \circt \widehat{\pi}\ind{t}\circt[t+1] \widehat{\pi}}. \label{eq: exo psdpd rel1}
\end{align}
The key steps above are justified as follows:
\begin{itemize}
  \item Relation $\mathrm{(a)}$ holds by the performance difference lemma for endogenous policies (\pref{lem:pd_for_endo_policies}), since both $\pi,\widehat{\pi}\in \PiIndNS\brk*{\Ic}$ by assumption.
% Relation $\mathrm{(b)}$ holds by the reachability assumption, which asserts that either 
% \begin{align*}
%      \max_{\pi\in \PiInd\brk{\Ic}}d_{t}(s\sbr{\Ic} \midsem \pi)\geq \eta \text{ or } \max_{\pi\in \PiInd\brk{\Ic}}d_{t}(s\sbr{\Ic} \midsem \pi) = 0.
    %   \end{align*}
    \item 
Relation $\mathrm{(b)}$  holds because $$\frac{\max_{\pi\in \PiInd\brk{\Ic}}d_t\rbr{\cdot \midsem \pi}}{d_t\rbr{\cdot \midsem \mut}}\leq 2S^k,$$ which is a consequence of~\pref{lem: policy cover bound importance sampling ratio}. In particular, we use that (i) $\cbr{\Psi\ind{t}}_{t=1}^{h-1}$ are endogenous $\eta/2$-approximate policy covers, (ii) for all states, either $\max_{\pi\in \PiInd\brk{\Ic}}d_t\rbr{s\brk{\Ic}\midsem \pi}\geq \eta$ or $\max_{\pi\in \PiInd\brk{\Ic}}d_t\rbr{s\brk{\Ic}\midsem \pi}=0$ (by the reachability assumption), and (iii) 
\[
\max_a Q^{\widehat{\pi}}_{t,\mathrm{en}}(s\sbr{\Ic},a) - Q^{\widehat{\pi}}_{t,\mathrm{en}}(s\sbr{\Ic},\widehat{\pi}_{t}(s\sbr{\Ic}))\geq{}0.
\]
\item
  Relation $\mathrm{(c)}$ holds by the skolemization principle (\pref{lem:skolemization}).
\end{itemize}
Let $\GS_{\mathrm{\PSDPE}}$ denote the success event for~\pref{lem: central result for PSDPE} (stated and proven in the sequel), which is the event in which for all $t\in [H]$, $\EndoPolicyMaximizer_{t,h}^{\epsnot}$ returns a policy $\widehat{\pi}\ind{t}$ such that
\begin{enumerate}
    \item $\widehat{\pi}\ind{t}$ is endogenous.
    \item $\widehat{\pi}\ind{t}$ is near-optimal in the following sense:
    \begin{align}
    \max_{\pi'\in \PiInd\brk{\Ic}}V_{t,H}\rbr{\mut \circt \pi'\circt[t+1] \widehat{\pi}\ind{t+1,H}} -  V_{t,H}\rbr{\mut \circt \widehat{\pi}\ind{t}\circt[t+1] \widehat{\pi}\ind{t+1,H}} \leq \epsnot. \label{eq:pspde_near_opt_guarantee}
\end{align}
\end{enumerate}
\pref{lem: central result for PSDPE} asserts that $\GS_{\mathrm{\PSDPE}}$ holds with probability at least $1-\delta$ whenever $N = \bigom\prn[\Big]{\frac{AH^2k^3 S^k \log\rbr{\frac{dAH}{\delta}}}{\epsnot^2}}$. Conditioning on $\GS_{\mathrm{\PSDPE}}$, it follows immediately that $\widehat{\pi}\ind{1,H}$ is endogenous. To show that the policy is near-optimal, we apply \eqref{eq: exo psdpd rel1} with $\pihat=\pihat\ind{1,H}$ and bound each term in the sum using
\eqref{eq:pspde_near_opt_guarantee}. Maximizing over $\pi\in \PiIndNS\brk*{\Ic}$ yields
\begin{align*}
    \max_{\pi\in \PiInd\brk*{\Ic}}J(\pi) - J(\widehat{\pi}) \leq 2 S^k H \epsnot= \eps,
\end{align*}
by the choice $\epsnot\ldef{}\eps/2S^kH$. Finally, by the fact that $\max_{\pi\in \PiInd\brk*{\Ic}}J(\pi) =  \max_{\pi\in \PiIndNS}J(\pi)$, which holds because the reward is endogenous (\citet{efroni2021provable}, Proposition 5), we conclude the proof.

% \dfcomment{THIS RESCALING NEEDS TO BE DONE WITHIN THE ALGORITHM DESCRIPTION. The $\eps$ in the call to EndoPolicyMaximizer needs to be changed. It is also very confusing having the $N$ in the algo description be different from the $N$ within the proof. I suggest defining $\eps_0=\eps/2S^kH$ within the algo and doing things in terms of this.}
\qed

\subsection{Computational Complexity of \PSDPE} \label{sec:comp_comp_pspde}

We now show that \PSDPE can be implemented with computational complexity of
\begin{align*}
    \bigoh\rbr{ d^{k}N \S^k A H },
\end{align*}
where $N$ is the number of trajectories. % Rest of parameters are defined in~\pref{sec: preliminaries}.
The main computational bottleneck of \PSDPE occurs at \pref{line: endo policy maximizer optimization step} of $\EndoPolicyMaximizer_{t,h}^{\epsnot}$. There, we need to optimize over $ \widehat{V}_{t,H}\rbr{\mut\circ_{t}\pi\circ_{t+1}\widehat{\pi}_{t+1:H}}$ estimated by the empirical averages (\pref{line: PSDPE estimation}) for all $\cI\in \fullc_k$. Meaning, \begin{align*}
    \max_{\pi\in \PiInd\sbr{\cI}} \widehat{V}_{t,H}\rbr{\mut\circ_{t}\pi\circ_{t+1}\widehat{\pi}_{t+1:H}}
\end{align*}
To sketch how to do this efficiently, we first show how to optimize over the set $\PiInd\brk*{\cI}$ when a factor set $\cI$ is fixed. We show that instead of enumerating over all policies, one can optimize  $ \widehat{V}_{t,H}\rbr{\mut\circ_{t}\pi\circ_{t+1}\widehat{\pi}_{t+1:H}}$ as follows. Observe that
\begin{align*}
    \widehat{V}_{t,h}\rbr{\mut\circ_{t}\pi\circ_{t+1}\widehat{\pi}_{t+1:H}} = \sum_{s\brk{\cI} \in \cS\brk*{\cI} }  \widehat{Q}_{t,h}^{\mut\circ_{t}\pi\circ_{t+1}\widehat{\pi}_{t+1:H}}\rbr{s\brk{\cI}, \pi\rbr{s\brk{\cI}}},
\end{align*}
where we note that $\abr{\cS\brk*{\cI}}\leq \S^k$, and where
\[
\widehat{Q}_{t,h}^{\mut\circ_{t}\pi\circ_{t+1}\widehat{\pi}_{t+1:H}}\rbr{s\brk{\cI}, a} \ldef \frac{1}{N} \sum_{n=1}^N \indic\crl*{s_t\brk{\cI} = s\brk{\cI}, a_t = a}\rbr{\sum_{t'=t}^h r_{n,t'}}.
\]
To maximize $\widehat{V}_{t,h}\rbr{\mut\circ_{t}\pi\circ_{t+1}\widehat{\pi}_{t+1:H}}$ it suffices to maximize each individual function \\ $\widehat{Q}_t^{\mut\circ_{t}\pi\circ_{t+1}\widehat{\pi}_{t+1:H}}\rbr{s\brk{\cI}, a}$. Letting
$$
    \widehat{\pi}_{\cI}\rbr{s\brk{\cI}} \in \argmax_a \widehat{Q}_t^{\mut\circ_{t}\pi\circ_{t+1}\widehat{\pi}_{t+1:H}}\rbr{s\brk{\cI}, a},
$$
we have that
\begin{align*}
    \max_{\pi\in \PiInd\brk{\cI}} \widehat{V}_{t,h}\rbr{\mut\circ_{t}\pi\circ_{t+1}\widehat{\pi}_{t+1:H}} = \widehat{V}_{t,h}\rbr{\mu\circt\widehat{\pi}_{\cI}\circt[t+1]\psi}.
\end{align*}
Furthermore, observe that $\widehat{\pi}_{\cI}\rbr{s\brk{\cI}}\in \PiInd\brk*{\cI}$.

This shows that it is possible to solve $\max_{\pi\in \PiInd\brk{\cI}} \widehat{V}_{t,h}\rbr{\mut\circ_{t}\pi\circ_{t+1}\widehat{\pi}_{t+1:H}}$ with computational complexity
$\bigoh\rbr{N\S^k A}$. Since $\EndoPolicyMaximizer^\eps_{t,h}$ optimizes over all possible factor sets $\cI\in \fullc_{\leq k}$ where $\abr{\fullc_{\leq k}}= \bigoh\rbr{ d^{k} }$ for $H$ times the total computational complexity is
$
\bigoh\rbr{ d^{k}N \S^k A H }.
$

% \dfcomment{The discussion in this section doesn't make sense because we never talk about the number of trajectories $N$ or expressing $\wh{V}$ as an empirical distribution in the algorithm's description.}

\subsection{Application of $\EndoPolicyMaximizer$ within $\PSDPE$}

In this section we state and prove \pref{lem: central result for PSDPE}, which shows that the application of $\EndoPolicyMaximizer$ within $\PSDPE$ (\pref{line: PSDPE optimization step}) is admissible, in the sense that the preconditions required by the algorithm are satisfied.

% . To do so, we verify the conditions of~\pref{thm: correctness of endo policy maximizer} and show they holds for $\PSDPE$ once the number of trajectories per-round is sufficient.

% \dfcomment{Confusing because we are using $N$ for number of trajectories per layer below, but we

\begin{lemma}[Guarantees of $\EndoPolicyMaximizer$ for $\PSDPE$]\label{lem: central result for PSDPE}
  Let precision parameter $\eps\in(0,1)$ and failure probability $\delta\in (0,1)$ be given. Assume that the mixture policies $\mut\in \Pimix$ used in \pref{alg: psdp with exogenous noise} are endogenous for all $t$. Then, if $N = \bigom\prn[\Big]{\frac{AH^2k^3 S^k \log\rbr{\frac{dAH}{\delta}}}{\eps^2}}$ trajectories are used for each layer, we have that with probability at least $1-\delta$, for all $t$:
\begin{enumerate}
    \item $\widehat{\pi}\ind{t}$ is an endogenous policy.
    \item $\widehat{\pi}\ind{t}$ is near-optimal in the sense that
    \begin{align*}
    \max_{\pi'\in \PiInd\brk{\Ic}}V_{t,H}\rbr{\mut \circt \pi'\circt[t+1] \widehat{\pi}\ind{t+1,H}} -  V_{t,H}\rbr{\mut \circt \widehat{\pi}\ind{t}\circt[t+1] \widehat{\pi}\ind{t+1,H}} \leq 4\eps.
\end{align*}
\end{enumerate}
\end{lemma}
\begin{proof}[\pfref{lem: central result for PSDPE}]
Let $\GS\ind{t}$ denote the event in which
\begin{enumerate}
    \item $\widehat{\pi}\ind{t}$ is an endogenous policy.
    \item $\widehat{\pi}\ind{t}$ is near optimal:
    \begin{align*}
    \max_{\pi'\in \PiInd\brk{\Ic}}V_{t,H}\rbr{\mut \circt \pi'\circt[t+1] \widehat{\pi}\ind{t+1,H}} -  V_{t,H}\rbr{\mut \circt \widehat{\pi}\ind{t}\circt[t+1] \widehat{\pi}\ind{t+1,H}} \leq 4\eps.
\end{align*}
\end{enumerate}
We will prove that for any $\delta>0$, 
\begin{align}
    \PP\rbr{\GS\ind{t} \mid \cap_{t'=t+1}^H\GS\ind{t'}}\geq 1-\delta, \label{eq: PSDPE what we need to prove}.
\end{align}
as long at least $\bigom\prn[\Big]{\frac{AH^2k^3 S^k \log\rbr{\frac{dAH}{\delta}}}{\eps^2}}$ trajectories are used at layer $t$. Whenever~\eqref{eq: PSDPE what we need to prove} holds,~\pref{lem: sequentual good event} implies that
\begin{align}
    \PP\rbr{\cap_{t=1}^H\GS\ind{t'}}\geq 1-H\delta, \label{eq: PSDPE implications of what we need to prove}
\end{align}
and scaling $\delta\gets \delta/H$ concludes the proof.

We now prove that~\eqref{eq: PSDPE what we need to prove} holds. To do so, we apply~\pref{thm: correctness of endo policy maximizer} and verify that assumptions $(\mathrm{A}1)$ and $(\mathrm{A}2)$ required by it hold.
\begin{enumerate}[label=$(\mathrm{A}\arabic*)$]
    \item Conditioning on the event $\cap_{t'=t+1}^H\GS\ind{t'}$, we have that $\widehat{\pi}\ind{t+1,H}$ is an endogenous policy. In addition $\mut$ is an endogenous policy and the reward function is endogenous by assumption. Thus, the conditions of~\pref{lem:restriction_endo_rewards} are satisfied, and the restriction property holds:
    \[
    \max_{\pi \in \Pi\brk{\cI}} V_{t,h}\rbr{\mu \circt \pi \circt[t+1] \psi} = \max_{\pi \in \Pi\brk{\Iendo{\cI}}} V_{t,h}\rbr{\mu \circt \pi \circt[t+1] \psi}.
    \]
    \item The proof of this result uses similar arguments to~\pref{lem: sample complexity of close model} . Fix $\pi\in \Pi\brk*{\fullc_{\leq k}}$ and observe that $ \widehat{V}_{t,H}\rbr{\mut\circ_{t}\pi\circ_{t+1}\widehat{\pi}_{t+1:H}}$ is an unbiased estimator for $ V_{t,H}\rbr{\mut\circ_{t}\pi\circ_{t+1}\widehat{\pi}_{t+1:H}}$, and is bounded by $AH$. Using \pref{lem: bernstein's inequality} and following the same steps as in the proof of~\pref{lem: sample complexity of close model}, we have that with probability at least $1-\delta$,
    \begin{align*}
        &\abr{\widehat{V}_{t,H}\rbr{\mut\circ_{t}\pi\circ_{t+1}\widehat{\pi}_{t+1:H}}  - V_{t,H}\rbr{\mut\circ_{t}\pi\circ_{t+1}\widehat{\pi}_{t+1:H}}} \\
        &\leq O\rbr{\sqrt{\frac{AH^2 \log\rbr{\frac{1}{\delta}}}{N}} + \frac{AH\log\rbr{\frac{1}{\delta}}}{N}}.
    \end{align*}
    Taking a union bound over all $\pi\in \Pi\brk*{\fullc_{\leq k}}$ and using that $\abr{\Pi\brk{\fullc_{\leq k}}} \leq O\rbr{d^{k+1} A^{S^k}}$, we have that with probability at least $1-\delta$,
\begin{align*}
        &\abr{\widehat{V}_{t,H}\rbr{\mut\circ_{t}\pi\circ_{t+1}\widehat{\pi}_{t+1:H}}  - V_{t,H}\rbr{\mut\circ_{t}\pi\circ_{t+1}\widehat{\pi}_{t+1:H}}} \\
        &\leq O\rbr{\sqrt{\frac{AH^2kS^k \log\rbr{\frac{dA}{\delta}}}{N}} + \frac{AHkS^k\log\rbr{\frac{dA}{\delta}}}{N}}.
    \end{align*}
    Hence, setting $N = \Omega\prn[\big]{\frac{AH^2k^3 S^k \log\rbr{\frac{dA}{\delta}}}{\eps^2}}$ and using that $\eps^2\leq \eps$ for $\eps\in (0,1)$, we have that with probability at least $1-\delta$, for all $\pi\in \Pi\brk*{\fullc_{\leq k}}$,
    \begin{align*}
        &\abr{\widehat{V}_{t,H}\rbr{\mut\circ_{t}\pi\circ_{t+1}\widehat{\pi}_{t+1:H}}  - V_{t,H}\rbr{\mut\circ_{t}\pi\circ_{t+1}\widehat{\pi}_{t+1:H}}} \leq \frac{\eps}{12k}.
    \end{align*}
\end{enumerate}

\end{proof}

%%% Local Variables:
%%% mode: latex
%%% TeX-master: "paper"
%%% End:

\newpage
\part{Additional Details and Proofs for Main Results}\label{part:main_res}

% \arxiv{\section{Proof of \preftitle{thm: sample complexity of state refinment}
%   (Correctness of $\AlgName$)}\label{app:OSSR_alg} }
\section{$\AlgName$ Description and Proof of \preftitle{thm: sample complexity of state refinment}}\label{app:OSSR_alg} 
%
%Proof of \pref{thm: sample complexity of state refinment}
%  (Correctness of )}}
\label{app: learnin eps endogenous policy cover}
\newcommand{\cGstat}{\cG_{\mathrm{stat}}}

\begin{algorithm}[tp]
  \setstretch{1.2}
\caption{$\AlgName_h^{\eps,\delta}$: Optimization-Selection State Refinement}
\label{alg: dp-sr paper version}
\begin{algorithmic}[1]
  \State\textbf{require:}
  \begin{itemize}
  \item Timestep $h$, precision parameter $\eps>0$, failure probability $\delta\in (0,1)$.
  \item Policy covers $\cbr{\Psi\ind{t}}_{t=1}^{h-1}$ for steps $1,\ldots,h-1$.
  \item Upper bound $k\geq 0$ on the cardinality of $\Ic$.
  \end{itemize}
  \State\textbf{initialize:}
  \begin{itemize}
  \item Let $\cI\ind{h,h} \gets \emptyset$ and $\Psi\ind{h,h} \gets \emptyset$.
    \item Define $N = C A S^{4k} H^2 k^3 \log\rbr{\frac{dSAH}{\delta}}\epsilon^{-2}$ for sufficiently large constant $C>0$, and let $\epsnot\ldef \tfrac{\eps}{2 \S^k H}$.
    \end{itemize}
    % \begin{itemize}
    % \item Let $\cI\ind{h,h} \gets \emptyset$ and $\Psi\ind{h,h} \gets \emptyset$.
    % \item Define $N \ldef{} \Theta\rbr{A S^{4k} H^2 k^3 \log\rbr{\frac{dSAH}{\delta}}\cdot\epsilon^{-2}}$.
    % \end{itemize}
    
    \For{$t=h-1,h-2,..,1$}
  % \item[]\algspace {\algcommentbig{\textbf{Dataset generation and importance sampling estimation}}}
  % \item[]\algspace {\algcommentbig{\textbf{Estimate occupancy measures}}}
  \item[]\algspace \algcommentul{Estimate occupancy measures}\vspace{3pt}
    \State\label{line: data generation process}Collect dataset $\cbr{(s_{t,n}, a_{t,n},
      \psi\ind{t+1,h}_{n} , s_{h,n})}_{n=1}^N$ by drawing $N$ trajectories from the process:
    \begin{itemize}[itemindent=10pt]
    \item Execute $\mu\ind{t}\ldef\unif(\Psi\ind{t})$ up to layer $t$ (resulting in state $s_{t,n}$).
    \item Sample action $a_{t,n}\sim\unif(\cA)$ and play it, transitioning to $s_{t+1,n}$ in the process. %(resulting in state $s_{t+1,n}$)
    \item Sample $\psi\ind{t+1,h}_n\sim\unif(\Psi\ind{t+1,h})$ and execute it from layers $t+1$ to $h$ (resulting in $s_{h,n}$).
    % \item 
    \end{itemize}
    % \[
    %   \mut\circ_{t} \unf(\Acal)\circ_{t+1}  \unf\rbr{\Psi\ind{t+1,h}}.
    % \]
    \State\multiline{For each $\cI\in\fullcleq$, $\pi\in\Pi\brk{\AllAbstractions}$, and $\psi\ind{t+1,h}\in\Psi\ind{t+1,h}$, define
        \begin{align*}
~~\dhat_h\prn*{s\brk{\cI}\midsem\mut\circt\pi\circt[t+1]\psi\ind{t+1,h}}=\frac{1}{N}\sum_{n=1}^{N}\frac{\indic\crl*{a_{t,n}=\pi(s_{t,n}), \psi\ind{t+1,h}_n=\psi\ind{t+1,h},s_{h,n}\brk{\cI}=s\brk{\cI}}}{(1/\abs{\cA})\cdot(1/\abs{\Psi\ind{t+1,h}})}.
        \end{align*}}\label{line: importance sampling estimation}
      \State{}Let $
          \EmpOccupancy\ind{t,h} \ldef \crl*{
            \dhat_h\prn*{\cdot\midsem\mut\circt\pi\circt[t+1]\psi\ind{t+1,h}}\mid{}\pi\in\PiInd\prn*{\scrI_{\leq{}k}},
              \psi\ind{t+1,h}\in\Psi\ind{t+1,h})
          }
        $. \label{line: construction of D hat}
        % \item[]\algspace \algcommentbig{\textbf{Optimization phase}}
        \vspace{2pt}
    \item[]\algspace \algcommentul{Phase I: Optimization}\hfill(\pref{alg: approximate 1-step endogenous policy maximizer} in \pref{app: near optimal endogenous policy})\vspace{3pt}
    \item[]\algspace %\algcomment{For each $\cI\in \fullc_{\leq k}\rbr{\cI\ind{t+1,h}}$, find an endogenous set of policies $\Gamma\ind{t}\brk*{\cJ}$ that approximately reaches $\cS\brk*{\cJ}$.}
      \algcomment{Beginning from any state at layer $t$, $\pi\ind{t}_{s\brk*{\cI}}\circt[t+1]\psi_{s\brk[\big]{\cI\ind{t+1,h}}}\ind{t+1,h}$ maximizes probability that $s_h\brk{\cI}=s\brk*{\cI}$.}
      \State{}For each $\cI \in \scrI_{\leq{}k}(\cI\ind{t+1,h})$ and $s\sbr{\cI} \in \Scal\sbr{\cI}$, let
        % \State \label{line: OSSR policy optimization phase}$\pi\ind{t}_{s\brk{\cI}} \gets \EndoPolicyMaximizer_{t,h}^{\eps}\rbr{\cbr{\dhat_h\rbr{ s\brk{\cI} \midsem \mut\circt  \pi \circt[t+1] \psi\ind{t+1,h}_{s\brk{\cI^{t+1,h}}}}}_{\pi\in \PiInd\brk{\fullc_{\leq k}}}}$
        \[
          ~~~~\pi\ind{t}_{s\brk{\cI}} \gets \EndoPolicyMaximizer_{t,h}^{\epsnot}\rbr{\cbr{\dhat_h\rbr{ s\brk{\cI} \midsem \mut\circt  \pi \circt[t+1] \psi\ind{t+1,h}_{s\brk{\cI^{t+1,h}}}}}_{\pi\in \PiInd\brk{\fullc_{\leq k}}}}.
        \]\label{line: OSSR policy optimization phase}
        % \EndFor
        \State{}Let $\Gamma\ind{t}\brk{\cI} \ldef
            \crl[\big]{\pi\ind{t}_{s\brk{\cI}}\mid{}s\brk{\cI}\in \cS\brk{\cI}}$.
        \vspace{2pt}
      % \item[]\algspace {\algcommentbig{\textbf{Selection phase}}}
      \item[]\algspace\algcommentul{Phase II: Selection}\hfill(\pref{alg: sufficient exploration check} in \pref{app: no false detection of factors})\vspace{3pt}
        \item[]\algspace\algcomment{Find factor set $\cI\ind{t,h}\subseteq\Ic$ such that $\Gamma\ind{t}\brk*{\cI\ind{t,h}}$ has good coverage for all factors in $\fullc_{\leq k}\rbr{\cI\ind{t+1,h}}$.}
        \State{}$(\cI\ind{t,h}, \Gamma\ind{t}\brk{\cI\ind{t,h}}) \gets
        \FactorDetectionAlg_{t,h}^{\epsnot}\rbr{\cbr{\Gamma\ind{t}\sbr{\cI}}_{\cI\in
            \scrI_{\leq{}k}(\cI\ind{t+1,h})} \midsem\cI\ind{t+1,h},\Psi\ind{t+1,h},
          \EmpOccupancy\ind{t,h}}$. 
\label{line:ossr_selection}
        \vspace{2pt}
      % \item[]\algspace {\algcommentbig{\textbf{Policy composition}}}
\item[]\algspace\algcommentul{Policy composition}\colt{\hfill\algcomment{Recall that $\pi\ind{t}_{s\brk{\cI\ind{t,h}}} \in \Gamma\ind{t}\brk*{\cI\ind{t,h}}$ and $\psi\ind{t+1,h}_{s\brk{\cI\ind{t+1,h}}}\in \Psi\ind{t+1,h}$.}}\vspace{3pt}
  \State{}Let $\cI\ind{t,h} \gets \cIhat$, then for each $s\brk{\cI\ind{t,h}}\in\cS\brk{\cI\ind{t,h}}$ define \colt{$\psi\ind{t,h}_{s\brk{\cI\ind{t,h}}} \ldef \pi\ind{t}_{s\brk{\cI\ind{t,h}}} \circt \psi\ind{t+1,h}_{s\brk{\cI\ind{t+1,h}}}.$}
  \arxiv{$\psi\ind{t,h}_{s\brk{\cI\ind{t,h}}} \ldef \pi\ind{t}_{s\brk{\cI\ind{t,h}}} \circt \psi\ind{t+1,h}_{s\brk{\cI\ind{t+1,h}}}.$}\label{line: policy composition OSSR}
    % \colt{ 
    % \psi\ind{t,h}_{s\brk{\cI\ind{t,h}}} \ldef \pi\ind{t}_{s\brk{\cI\ind{t,h}}} \circt \psi\ind{t+1,h}_{s\brk{\cI\ind{t+1,h}}}.
    % }
        % \arxiv{\item[] \hfill\algcomment{Recall that $\pi\ind{t}_{s\brk{\cI\ind{t,h}}} \in \Gamma\ind{t}\brk*{\cI\ind{t,h}}$ and $\psi\ind{t+1,h}_{s\brk{\cI\ind{t+1,h}}}\in \Psi\ind{t+1,h}$.}}
  \State{}Let $\Psi\ind{t,h} \gets \cbr{\psi\ind{t,h}_{s\brk{\cI\ind{t,h}}} : s\brk{\cI\ind{t,h}}\in \cS\brk{\cI\ind{t,h}}}$.
  % Let $\cI\ind{t,h} \gets \cI$ and $\Psi\ind{t,h} \gets \cbr{\psi\ind{t,h}_{s\brk{\cI\ind{t,h}}} : s\brk{\cI\ind{t,h}}\in \cS\brk{\cI\ind{t,h}}}$ where
  %       \begin{align*}
  %        \psi\ind{t,h}_{s\brk{\cI\ind{t,h}}} \ldef{} \pi\ind{t}_{s\brk{\cI\ind{t,h}}} \circ \psi\ind{t+1,h}_{s\brk{\cI\ind{t+1,h}}},\ \text{for}\ \pi\ind{t}_{s\brk{\cI\ind{t,h}}} \in \Gamma\ind{t}\brk*{\cI\ind{t,h}},\ \psi\ind{t+1,h}_{s\brk{\cI\ind{t+1,h}}}\in \Psi\ind{t+1,h}.
  %       \end{align*}
        \EndFor
        % \State Let 
        % and $\cI\ind{h}\ldef{}\cI\ind{1,h}$.
        \State{}\textbf{return} $\Psi\ind{h}\ldef{}\Psi\ind{1,h}$. \hfill\algcomment{Policy cover for timestep $h$.}
        % $(\Psi\ind{h}, \cI\ind{h})$
        % Policy cover $\Psi\ind{h} \ldef
    % \Psi\ind{1,h}$ and endogenous factor $\cI\ind{h}\ldef\cI\ind{1,h}$.
\end{algorithmic}
\end{algorithm}

In this section we present and analyze the full $\AlgName_{h}^{\eps,\delta}$ algorithm (\pref{alg: dp-sr paper version}). The algorithm may be thought of as a sample-based version of the $\ExactAlgName$ algorithm described in \pref{sec: DP-SR exact}. While $\ExactAlgName$ assumes exact access to state occupancy measures, $\AlgName_{h}^{\eps,\delta}$ estimates the occupancy measures in a data-driven fashion, which introduces the need to account for statistical errors.

This section is organized as follows. First, in \pref{app:OSSR_overview} we give a high-level overview of the algorithm design principles behind $\AlgName_{h}^{\eps,\delta}$. Then, in~\pref{app:thm_proof_ossr}, we prove the main result concerning its performance, \pref{thm: sample complexity of state refinment}. \pref{sec:ossr_sub1,app: application of endoPolicyOptimizer to OSSR} contain proofs for supporting results used in the proof of \pref{thm: sample complexity of state refinment}.

\subsection{\pcalg: Algorithm Overview}\label{app:OSSR_overview}

% \dfcomment{Clean this up / make it coherent.}

% Our main algorithm, $\AlgName^{\eps,\delta}_h$, is depicted in~\pref{alg: dp-sr paper version}, that adapts $\ExactAlgName$ to the sample-based case.

The $\AlgName^{\eps,\delta}_h$ algorithm follows the same template as \pcalgexact: For each $h\in\brk{H}$, given policy covers $\Psi\ind{1},\ldots,\Psi\ind{h-1}$, the algorithm builds a policy cover $\Psi\ind{h}$ for layer $h$ in a backwards fashion using dynamic programming. There are two differences from the exact algorithm. First, we only have sample access to the underlying \framework, the algorithm estimates the relevant occupancy measures for each backward step using Monte Carlo rollouts\arxiv{ and importance weighting}. Second, the optimization and selection phases from \pcalgexact are replaced by error-tolerant variants given by the subroutines $\EndoPolicyMaximizer$ and $\FactorDetectionAlg$ (\pref{alg: approximate 1-step endogenous policy maximizer} in \pref{app: near optimal endogenous policy} and \pref{alg: sufficient exploration check} in \pref{app: no false detection of factors}, respectively). %We now expand on both procedures.

  \paragraph{State occupancy estimation.}
  In order to apply dynamic programming in the same fashion as \pcalgexact, 
  each backward step $1\leq{}t\leq{}h-1$ of \pcalgfull proceeds by building estimates for the layer-$h$ occupancies $d_h(s\brk{\cI}\midsem\mu\ind{t}\circt\pi\circt[t+1]\psi\ind{t+1,h})$ for all $\cI\in\fullcleq$, $\pi\in\Pi\brk{\fullcleq}$ and $\psi\ind{t+1,h}\in\Psi\ind{t+1,h}$. This is accomplished through Monte Carlo: We gather trajectories by running $\mu\ind{t}$ up to layer $t$, sampling $a_t\sim\unif(\cA)$ uniformly, then sampling $\psi\ind{t+1,h}\sim\unif(\Psi\ind{t+1,h})$ and using it to roll out from layer $t+1$ to $h$. We then build estimates by importance weighting the empirical frequencies. We appeal to uniform convergence to ensure that the estimated occupancies are uniformly close for all $\cI\in\fullcleq$ and $\pi\in \PiInd\brk*{\fullc_{\leq k}}$; this argument critically uses that $\abs{\Psi\ind{t+1,h}}\leq{}S^{k}$ and $\log\abs*{\PiInd\brk*{\fullc_{\leq k}}}\leq O\rbr{kS^k\log\rbr{dA}}$, as well as the fact that we only require convergence for factors of size at most $k$.
  % , the variance of the estimated occupancy measure for any given policy $\pi\in\Pi\brk{\fullc_{\leq{}k}}$ scales with $\poly(S^{k},A,H)$. 
  % which leads only to an additional $\log\abs*{\PiInd\brk*{\fullc_{\leq k}}}\leq O\rbr{kS^k\log\rbr{dA}}$ factor. This 

  \paragraph{Error-tolerant backward state refinement.} Given the estimated state occupancy measures above, each  backward step $1\leq{}t\leq{}h-1$ of $\AlgName^{\eps,\delta}_h$ follows the general optimization-selection template used in \pcalgexact. For the optimization step (\pref{line: OSSR policy optimization phase}), it applies the subroutine $\EndoPolicyMaximizer^\eps_{t,h}$ (\pref{alg: approximate 1-step endogenous policy maximizer} in \pref{app: near optimal endogenous policy}), which finds a collection of endogenous ``one-step'' policy covers $\prn{\Gamma\ind{t}\brk*{\cI}}_{\cI\in \fullc_{\leq k} \rbr{\cI\ind{t+1,h}}}$, which have the property that for all $\cI\in \fullc_{\leq k} \rbr{\cI\ind{t+1,h}}$ and $s\in\cS$, the $t\to h$ policy $\pi\ind{t}_{s\sbr{\cI}}\circ \psi\ind{t+1,h}_{s\sbr{\cI\ind{t+1,h}}}$ (approximately) maximizes the probability that $s_h\brk{\cI}=s\brk{\cI}$. Then, at selection step (\pref{line:ossr_selection}), $\AlgName^{\eps,\delta}_h$ applies the subroutine $\FactorDetectionAlg^\eps_{t,h}$ (\pref{alg: sufficient exploration check} in \pref{app: no false detection of factors}), which selects a single factor set $\cI\ind{t,h}\subseteq\Ic$ such that---by choosing $\Psi\ind{t,h}$ to be the composition of $\Gamma\ind{t}\brk*{\cI\ind{t,h}}$ and $\Psi\ind{t+1,h}$---we obtain an (approximate) $t\to{}h$ policy cover.
% to obtain $$.
% Lastly, the backward update step is similar to the one of $\ExactAlgName.$

Full descriptions and proofs of correctness for $\EndoPolicyMaximizer^\eps_{t,h}$ and $\FactorDetectionAlg^\eps_{t,h}$ are given in \pref{app: near optimal endogenous policy} and~\pref{app: no false detection of factors}. Briefly, both subroutines are based on approximate versions of the constraints used in the optimization and selection phase for \ExactAlgName (\pref{line: exact DP-SR state1} and~\pref{line: exact DP-SR state2} of \pref{alg: exact dynamics state refinement}), but ensuring endogeneity of the resulting factors is more challenging due to approximation errors, and it no longer suffices to simply search for the factor set with minimum cardinality. Instead, we search for factor sets that satisfy approximate versions of \pref{line: exact DP-SR state1} and~\pref{line: exact DP-SR state2} with an \emph{additive regularization term} based on cardinality. We show that as long as this penalty is carefully chosen as a function of the statistical error in the occupancy estimates, the resulting factor sets will be endogenous while inducing sufficient amount of exploration (with high probability).

In \pref{app:abstract_search}, we provide a general template for designing error-tolerant algorithms that search for endogenous factors using the approach described; both $\EndoPolicyMaximizer^\eps_{t,h}$ and $\FactorDetectionAlg^\eps_{t,h}$ are special cases of this template.

\subsection{Proof of \preftitle{thm: sample complexity of state refinment}}\label{app:thm_proof_ossr}

We now restate and prove \pref{thm: sample complexity of state refinment}, which shows that \pcalgfull learns an endogenous $\eps$-optimal policy cover with sample complexity depending only logarithmically on the number of factors $d$.

\DPSRSampleComplexity*
\begin{proof}[\pfref{thm: sample complexity of state refinment}] We begin by defining a success event for \mainalg. 
\begin{definition}[Success of~$\AlgName$ at the layer $h$]\label{def:ossr_success}
  $\GS\ind{h}$ is defined as the event in which the following properties hold:
\begin{enumerate}
    \item $\Psi\ind{h}$ is an endogenous $\eta/2$-approximate policy cover for layer $h$.
    \item $\cI\ind{h}$ contains only endogenous factors.
\end{enumerate}
\end{definition}
In addition, we define $\cG\ind{<h}=\cap_{h'=1}^{h-1}\GS\ind{h'}$. The following intermediate result---proven in the sequel (\pref{sec:ossr_sub1})---serves as our starting point.% the proof of \pref{thm: sample complexity of state refinment}.
\begin{restatable}[Success of State Refinement]{theorem}{SuccessOfBackwardStateRefinement}\label{thm: success of backward state refinement}
  Fix $h\in [H]$ and condition on $\cG\ind{<h}$. Then, for any $\eps>0$ (recalling that $\epsnot\ldef \tfrac{\eps}{2 \S^k H}$), by setting
\begin{align*}
    N = \Theta\prn[\Big]{A S^{2k} k^3 \log\rbr{\frac{dSAH}{\delta}}\cdot\epsnot^{-2}},
\end{align*}
$\pcalg_{h}^{\eps,\delta}$ guarantees that with probability at least $1-\delta$, for all $t\leq{}h$,
% for all $t\in [h]$ with probability at least $1-\delta$ it holds that
\begin{enumerate}
    \item $\cI\ind{t,h}\subseteq \Ic$, and $\Psi\ind{t,h}$ contains only endogenous policies.
    \item For all $s\in\cS$,
    \begin{align}
        \max_{\pi\in \PiInd\brk{\Ic}}d_h\rbr{s\brk{\Ic} ; \mut\circt \pi\circt[t+1]\psi\ind{t+1,h}_{s\brk{\cI\ind{t+1,h}}}} - d_h\rbr{s\brk{\Ic} ; \mut\circt \psi\ind{t,h}_{s\brk{\cI\ind{t,h}}}} \leq   \epsnot, \label{eq: consequence of fobastate main thm rel 1}
    \end{align}
    where we recall that $\psi\ind{t,h}_{s\brk{\cI\ind{t,h}}}\in \Psi\ind{t,h}$ and $\psi\ind{t+1,h}_{s\brk{\cI\ind{t+1,h}}}\in \Psi\ind{t+1,h}$.
    % $s\brk{\Ic} = \rbr{s\brk{\cI\ind{t,h}},s\brk{\Ic \setminus{} \cI\ind{t,h}}} =  \rbr{s\brk{\cI\ind{t+1,h}},s\brk{\Ic \setminus{} \cI\ind{t+1,h}}}$, and. % Then, for all $s\brk{\Ic}$ and all $t\in [h-1]$,    
\end{enumerate}
\end{restatable}

% \pref{thm: success of backward state refinement} directly implies that $\Psi\ind{h}$ is a set of endogenous policies, that $\abr{\Psi\ind{h}}\leq \S^k$ and that  $\cI\ind{h}$ contains only endogenous factors.
% \YEcomment{verify we indeed define a policy cover such that $\Psi$ as cardinality at most $S^k$.}

We now show that conditioned on the event in \pref{thm: success of backward state refinement}, the set $\Psi\ind{h}$ is an endogenous, $\epsilon$-approximate policy cover (as long as $\epsnot$ is chosen to be sufficiently small). In particular, we will show that for all $s\brk{\Ic}\in \Scal\brk{\Ic}$ there exists a policy $\psi\in\Psi\ind{h}$ such that
\begin{align}
  % \max_{\pi\in \PiIndNS\brk{\Ic}}d_h(s\brk{\Ic} ; \pi) \leq d_h(s\brk{\Ic} ; \psi) + \eta/2. \label{eq: main result what we need to prove}
      \max_{\pi\in \PiIndNS\brk{\Ic}}d_h(s\brk{\Ic} ; \pi) \leq d_h(s\brk{\Ic} ; \psi) + \eps. \label{eq: main result what we need to prove}
\end{align}
Fix $s\brk{\Ic}\in \cS\brk{\Ic}$. From first part of~\pref{thm: success of backward state refinement}, we have that $\cI\ind{1,h}\subseteq \Ic$, so we can write
$     
s\brk{\Ic} = \rbr{s\brk{\cI\ind{1,h}}, s\brk{\Ic\setminus{}\cI\ind{1,h}}}.
$
We will show that the policy $\psi\ind{1,h}_{s\brk{\cI\ind{1,h}}}\in \Psi\ind{h}=\Psi\ind{1,h}$ maximizes the probability of reaching $s\brk{\Ic}\in \cS\brk{\Ic}$ in the sense of \eqref{eq: main result what we need to prove}.

Define a endogenous ``reward function'' $R_{s\brk{\Ic}}$, with \[R_{s\brk{\Ic},h}(s_h\sbr{\Ic})\ldef \indic\cbr{s_h\sbr{\Ic} \ldef s\sbr{\Ic}}\] and $R_{s\brk{\Ic},t}(\cdot)\ldef0$ for $t\neq{}h$. Letting $r_{s\brk{\Ic},t}\ldef{}R_{s\brk{\Ic},t}(s_t\sbr{\Ic})$, we can write
\begin{align}
    d_h(s\sbr{\Ic};\pi) \ldef{} \EE_\pi\sbr{\sum_{t=1}^{h} r_{s\brk{\Ic},t} }. \label{eq:reward_endo_mdp_indicator}
\end{align}
% \dfcomment{where is the $\En\brk*{\cdot\mid{}\pi}$ notation defined?}
That is, we can view the state occupancy $d_h(s\sbr{\Ic};\pi)$ as the state value function for the \framework $\Mcal \ldef \rbr{\Scal,\Acal,T, R_{s\brk{\Ic}},H, d_1}$. Let $\pi\in \PiIndNS\brk{\Ic}$ be an endogenous policy. We let $\Mcal_{\mathrm{en}} = \rbr{\Scal,\Acal,\Tc, R_{s\brk{\Ic}},H, \dc}$ denote the endogenous component of this MDP, and let $Q^\pi_{t,\mathrm{en}}(s\brk{\Ic},a)$ denote the associated state-action value function for $\cM_{\en}$.

To proceed, we use the representation above within the performance difference lemma (\pref{lem:pd_for_endo_policies}) to bound the suboptimality of $\psi\ind{1,h}_{s\brk{\cI\ind{1,h}}}$ by a sum of "per-step" errors for each of the backward steps. In particular for any pair of endogenous policies $\pi,\psi\in \PiIndNS\brk{\Ic}$, \pref{lem:pd_for_endo_policies} implies that
\begingroup
\allowdisplaybreaks
\begin{align}
    &d_h(s\sbr{\Ic};\pi) - d_h(s\sbr{\Ic} ;\psi) \nonumber \\
    &\stackrel{\mathrm{(a)}}{=} \sum_{t=1}^h \EE_{\pi} \sbr{Q^{\psi}_{t,\mathrm{en}}(s_t\sbr{\Ic},\pi_{t}(s_t\sbr{\Ic}) - Q^{\psi}_{t,\mathrm{en}}(s_t\sbr{\Ic},\psi_{t}(s_t\sbr{\Ic}))}  \nonumber\\
    &\leq  \sum_{t=1}^h \EE_{s\brk{\Ic}\sim d_{t}\rbr{\cdot \midsem \pi}} \sbr{\max_{a}Q^{\psi}_{t,\mathrm{en}}(s\sbr{\Ic},a) - Q^{\psi}_{t,\mathrm{en}}(s\sbr{\Ic},\psi_{t}(s\sbr{\Ic}))}  \nonumber\\
    % &=\sum_{t=1}^h \sum_{s\sbr{\Ic}\in \cS\brk{\Ic}} d_{t}(s\sbr{\Ic}\midsem\pi)\rbr{\max_{a}Q^{\psi}_{t,\mathrm{en}}(s\sbr{\Ic},a) - Q^{\psi}_{t,\mathrm{en}}(s\sbr{\Ic},\psi_{t}(s\sbr{\Ic}))}. \nonumber \\
    &\leq \sum_{t=1}^h \sum_{s\sbr{\Ic}\in \cS\brk{\Ic}} \max_{\pi'\in \PiIndNS\brk{\Ic}}d_{t}(s\sbr{\Ic}\midsem\pi')\rbr{\max_{a}Q^{\psi}_{t,\mathrm{en}}(s\sbr{\Ic},a) - Q^{\psi}_{t,\mathrm{en}}(s\sbr{\Ic},\psi_{t}(s\sbr{\Ic}))}. \nonumber \\
    % &= \sum_{t=1}^h \sum_{s\sbr{\Ic}:\ \max_{\pi'\in \PiIndNS\brk{\Ic}}d_{t}(s\sbr{\Ic}\midsem\pi')\geq \eta} \max_{\pi'\in \PiIndNS\brk{\Ic}}d_{t}(s\sbr{\Ic}\midsem\pi')\rbr{\max_{a}Q^{\psi}_{t,\mathrm{en}}(s\sbr{\Ic},a) - Q^{\psi}_{t,\mathrm{en}}(s\sbr{\Ic},\psi_{t}(s\sbr{\Ic}))}. \nonumber \\
    &\stackrel{\mathrm{(b)}}{\leq} 2S^{k} \sum_{t=1}^h \sum_{s\sbr{\Ic}\in \Scal\brk{\Ic}} d_t(s\sbr{\Ic}\midsem \mu\ind{t})\rbr{\max_{a}Q^{\psi}_{t,\mathrm{en}}(s\sbr{\Ic},a) - Q^{\psi}_{t,\mathrm{en}}(s\sbr{\Ic},\psi_{t}(s\sbr{\Ic}))} \nonumber\\
    &= 2S^{k} \sum_{t=1}^h \EE_{s\brk{\Ic}\sim d_t\rbr{\cdot \midsem \mut}}\sbr{\max_{a}Q^{\psi}_{t,\mathrm{en}}(s\sbr{\Ic},a) - Q^{\psi}_{t,\mathrm{en}}(s\sbr{\Ic},\psi_{t}(s\sbr{\Ic}))} \nonumber\\
    % &\stackrel{\mathrm{(c)}}{=}  2S^{k} \sum_{t=1}^h \max_{\pi'\in \PiInd\brk{\Ic}}\EE_{s\brk{\Ic}\sim d_t\rbr{\cdot \midsem \mut}}\sbr{Q^{\psi}_{t,\mathrm{en}}(s\sbr{\Ic},\pi'(s\sbr{\Ic}))} -\EE_{s\brk{\Ic}\sim d_t\rbr{\cdot \midsem \mut}}{Q^{\psi}_{t,\mathrm{en}}(s\sbr{\Ic},\psi_{t}(s\sbr{\Ic}))}\nonumber\\
    & \stackrel{\mathrm{(c)}}{=} 2S^{k} \sum_{t=1}^h \rbr{\max_{\pi'\in \PiInd\brk{\Ic}} d_h(s_{h}\sbr{\Ic} ; \mu\ind{t}\circt \pi'\circt[t+1]\psi) - d_h(s_{h}\sbr{\Ic} ; \mu\ind{t}\circt \psi  )}. \label{eq: central theorem relation 1}
\end{align}
\endgroup
We justify the steps above as follows:
\begin{itemize}
  \item The equality $\mathrm{(a)}$ follows from \pref{lem:pd_for_endo_policies}).
% , and the inequality $\mathrm{(b)}$ holds because $\pi\in \PiIndNS\brk{\Ic}$ and 
% \begin{align}
%     \max_{a}Q^{\psi}_{t,\mathrm{en}}(s\sbr{\Ic},a) - Q^{\psi}_{t,\mathrm{en}}(s\sbr{\Ic},\psi_{t}(s\sbr{\Ic}))\geq 0\label{eq:Q_delta_is_positive}
    %   \end{align}
  \item 
Relation $\mathrm{(b)}$ holds because $$\frac{\max_{\pi\in \PiInd\brk{\Ic}}d_t\rbr{\cdot \midsem \pi}}{d_t\rbr{\cdot \midsem \mut}}\leq 2S^k$$ which is a consequence of~\pref{lem: policy cover bound importance sampling ratio}. In particular, we use that (i) $\cbr{\Psi\ind{t}}_{t=1}^{h-1}$ are endogenous $\eta/2$-approximate policy covers, (ii) either $\max_{\pi\in \PiInd\brk{\Ic}}d_t\rbr{s\brk{\Ic}\midsem \pi}\geq \eta$  or $\max_{\pi\in \PiInd\brk{\Ic}}d_t\rbr{s\brk{\Ic}\midsem \pi}=0$ for all $s\brk{\Ic}\in \Scal\brk{\Ic}$ by the reachability assumption, and (iii) nonnegativity:
\[
\max_{a}Q^{\psi}_{t,\mathrm{en}}(s\sbr{\Ic},a) - Q^{\psi}_{t,\mathrm{en}}(s\sbr{\Ic},\psi_{t}(s\sbr{\Ic}))\geq 0.
\]
\item 
  Relation $\mathrm{(c)}$ holds by the skolemization principle~(\pref{lem:skolemization}) and the tower rule for conditional probabilities. %(\eqref{eq:reward_endo_mdp_indicator}).
\end{itemize}
% We apply \eqref{eq: central theorem relation 1} with.
Recall that the event defined in~\pref{thm: success of backward state refinement} (\eqref{eq: consequence of fobastate main thm rel 1}) implies that for all $t\leq{}h$,
\begin{align*}
    & \max_{\pi'\in \PiInd\brk{\Ic}} d_h(s_{h}\sbr{\Ic} ; \mu\ind{t}\circt \pi'\circt[t+1]\psi\ind{t+1,h}_{s\sbr{\cI\ind{t+1,h}}} - d_h(s_{h}\sbr{\Ic} ; \mu\ind{t}\circt \psi\ind{t,h}_{s\sbr{\cI\ind{t,h}}})  \leq \epsnot.
\end{align*}
Plugging this bound into~\eqref{eq: central theorem relation 1} with $\psi \gets  \psi\ind{1,h}_{s\sbr{\cI\ind{1,h}}}$, we have that for all endogenous policies $\pi$,
\begin{align*}
    d_h(s\sbr{\Ic};\pi) - d_h(s\sbr{\Ic} ; \psi\ind{1,h}_{s\sbr{\cI\ind{1,h}}}) \leq 2 \S^k H \epsnot.
\end{align*}
% \dfcomment{This is problematic. You can only rescale the $\eps$ coming from statistical error. You can't just rescale the error from $\Psi\ind{1}...\Psi\ind{h-1}$ without exponential blowup.}
By using that $\epsnot\ldef \eps/ 2 \S^k H $ and taking the maximum with respect to $\pi\in \PiInd\brk{\Ic}$, we conclude that for all $s\sbr{\Ic} = \rbr{s\sbr{\cI\ind{1,h}}, s\sbr{\Ic\setminus{}\cI\ind{1,h}}}$, the policy $\psi\ind{1,h}_{s\sbr{\cI\ind{1,h}}}$ satisfies
\begin{align}
    \max_{\pi\in \PiInd\brk{\Ic}}d_h(s\sbr{\Ic};\pi) - d_h(s\sbr{\Ic} ; \psi\ind{1,h}_{s\sbr{\cI\ind{1,h}}})\leq \eps. \label{eq: central final result relation }
\end{align}
This establishes that the set $\Psi\ind{h}$ is an endogenous $\eps$-approximate policy cover. With this choice for $\epsnot$, the total sample complexity is $\bigoh\rbr{A S^{4k} H^2 k^3 \log\rbr{\frac{dSAH}{\delta}}\cdot\eps^{-2}}$. Finally, we note that as a consequence of \pref{thm: success of backward state refinement}, we have $\cI\ind{1,h}\subseteq\Ic$ as desired. We have $\abr{\Psi\ind{h}}\leq S^k$ by construction.
% This concludes the first part of the theorem.

% Next, by construction of $\Psi\ind{1,h}$ it is given by $$\Psi\ind{h}\ldef \Psi\ind{1,h}= \cbr{\psi\ind{1,h}_{s\brk{\cI\ind{1,h}}} : s\brk{\cI\ind{1,h}}\in \cS\brk{\cI\ind{1,h}}}.$$

%This concludes the second part of the theorem.
\end{proof}

\begin{comment}
\subsection{Sample Complexity of State Occupancy Estimation}
\label{app: epsilon close model sample complexity}

\dfcomment{Can we ditch this subsection? It seems like we can move everything to the prelims.}

\begin{definition}[$\epsilon$-approximate set of occupancy measures with respect to $\rbr{\Pi, \fullc,h}$] \label{def: approximate set of occupancy measures}
Let 
$
\EmpOccupancy = \crl*{\dhat_h^{\pi} \mid{} \pi\in\Pi},
$
be a set of occupancy measures over timestep $h$. We say that $\EmpOccupancy$ is $\epsilon$-approximate with respect to $\rbr{\Pi, \fullc,h}$ if for all $\pi\in \Pi, \cI\in \fullc, s\sbr{\cI}\in \Scal\sbr{\cI}$ it holds that
\begin{align*}
    \abr{\widehat{d}_{h}\rbr{ s_h\sbr{\cI} = s\sbr{\cI} \midsem\pi} - d_{h}\rbr{s_h\sbr{\cI} = s\sbr{\cI} \midsem \pi}} \leq \eps.
\end{align*}
\end{definition}

We recall the following lemma from \pref{app: cocentration results}. %
It is based on standard Bernstein concentration result and a union bound.

\SampleComplexityEstimator*
\end{comment}

% \newpage
\subsection{Proof of \preftitle{thm: success of backward state refinement} (Success of State Refinement Step)}
\label{sec:ossr_sub1}

In this section we prove~\pref{thm: success of backward state refinement}, a supporting result used in the proof of~\pref{thm: sample complexity of state refinment}.  The result shows for each step $t$, the optimization and selection phases in $\AlgName^{\eps,\delta}_{h}$ lead to a set of endogenous $t\to{}h$ policies~$\Psi\ind{t,h}$, as long as certain preconditions are satisfied.
      %       $ that also satisfies near-optimality conditions specified below.

\SuccessOfBackwardStateRefinement*

% We make the following definition of an event which will be found to be useful.  

% \begin{definition}[Success of no false detection of endogenous factors]\label{def: no false detection of endo factors of previous rounds}
% Let $\Gendo\ind{t+1,h}$ be the event in which the following holds for the output of $\AlgName$ at the $h^{\th}$ timestep and the $t+1$ iteration.
% \begin{enumerate}[label = ($E_{\arabic*}$)]
%     % \item The returned policy cover of $\AlgName$ at the $t^{\th}$ timestep $\Psi\ind{t}$ where $t\leq h-1$ contains only endogenous policies, and, thus $\mu\ind{t} = \unf\rbr{\Psi\ind{t}}$ is an endogenous policy.
%     \item The set $\cI\ind{t+1,h}$ contains only endogenous factors.
%     \item The set $\Psi\ind{t+1,h}$ contains endogenous policies.
% \end{enumerate}
% \end{definition}
% That is, $\Gendo\ind{t+1,h}$ captures two scenarios. First,  at the $h^{\th}$ iteration of~\pref{alg: main} and the $t+1$ iteration of~$\AlgName$ all recovered policies and factors are endogenous. Second, the set of state frequencies $\EmpOccupancy\ind{t,h}$ is a good approximation with respect to  $\rbr{\PiInd\brk{\fullc_{\leq k}}, \fullc_{\leq{}k}\rbr{\cI\ind{t+1,h}},h}$ (see~\pref{def: approximate set of occupancy measures}).

% We are now ready to prove~\pref{thm: success of backward state refinement}.

% {\todo{$\GS$ here possibly needs to be changed}}
\begin{proof}[\pfref{thm: success of backward state refinement}]
The event $\GS\ind{<h}$ (\pref{def:ossr_success}) holds by assumption, which implies that the policy sets $\Psi\ind{t}$ for $t\in [h-1]$ contain only endogenous policies. As a result, 
\begin{align}
    \mu\ind{t} \ldef \unf\rbr{\Psi\ind{t}}. \label{eq: mut is endo policy statement}
\end{align}
 is an endogenous mixture policy. To proceed, we define some intermediate success events which will be used throughout the proof. First, for $t\leq{}h$ define %an event $\GS_1\ind{t,h}$ via
\[
\GS_1\ind{t,h} \ldef \crl*{\text{$\Psi\ind{t,h}$ contains only endogenous policies, and $\cI\ind{t,h}\subseteq\Ic$}}.
\]
Observe when $\GS_1\ind{t,h}$ holds, we can express all states $s\brk{\Ic}\in\cS\brk{\Ic}$ as
    \begin{align*}
        s\brk{\Ic} = \rbr{s\brk{\cI\ind{t,h}},s\brk{\Ic \setminus{} \cI\ind{t,h}}} = \rbr{s\brk{\cI\ind{t+1,h}},s\brk{\Ic \setminus{} \cI\ind{t+1,h}}},
    \end{align*}
    since $\cI\ind{t+1,h} \subseteq \cI\ind{t,h} \subseteq \Ic$. Next, we define an event $\GS_{2}\ind{t,h}$ via 
        \begin{align*}
 &\GS_{2}\ind{t,h}\ldef\\
 &\crl*{\forall{}s\brk{\Ic}\in\cS\brk{\Ic}: \!  \text{$\max_{\pi\in \PiInd\brk{\Ic}}\! d_h(s\brk{\Ic} ; \mut\circt \pi\circt[t+1]\psi\ind{t+1,h}_{s\brk{\cI\ind{t+1,h}}}) - d_h(s\brk{\Ic} ; \mut\circt\psi\ind{t,h}_{s\brk{\cI\ind{t,h}}}) \leq \epsnot$}},
        \end{align*}
        where we recall that $\psi\ind{t+1,h}_{s\brk{\cI\ind{t+1,h}}}\in \Psi\ind{t,h}$ and $\psi\ind{t,h}_{s\brk{\cI\ind{t,h}}}\in \Psi\ind{t,h}$. Finally, let $\GS\ind{t,h} \ldef \GS\ind{t,h}_1 \cap \GS\ind{t,h}_2$.
%\begin{enumerate}%[label = ($\GS_{\arabic*}\ind{t,h}$)]
%    \item Event $\GS_1\ind{t,h}$: $\crl*{\text{$\Psi\ind{t,h}$ contains only endogenous policies and $\cI\ind{t,h}\subseteq\Ic$}}$.
%    \item Event $\GS_2\ind{t,h}$: Let $s\brk{\Ic}\in \Scal\brk{\Ic}$, and observe it can be written as
%    \begin{align*}
%        s\sbr{\Ic} = \rbr{s\sbr{\cI\ind{t,h}},s\sbr{\Ic \setminus{} \cI\ind{t,h}}} = \rbr{s\sbr{\cI\ind{t+1,h}},s\sbr{\Ic \setminus{} \cI\ind{t+1,h}}},
%    \end{align*}
%    since $\cI\ind{t+1,h} \subseteq \cI\ind{t,h} \subseteq \Ic$. Let $\psi\ind{t+1,h}_{s\sbr{\cI\ind{t+1,h}}}\in \Psi\ind{t,h},\psi\ind{t,h}_{s\sbr{\cI\ind{t,h}}}\in \Psi\ind{t,h}$. Then, for all $s\sbr{\Ic}$ and all $t\in [h-1]$,
%    \begin{align}
%        \max_{\pi\in \PiInd\brk{\Ic}}d_h(s\sbr{\Ic} ; \mut\circt \pi\circt[t+1]\psi\ind{t+1,h}_{s\sbr{\cI\ind{t+1,h}}}) - d_h(s\sbr{\Ic} ; \mut\circt\psi\ind{t,h}_{s\sbr{\cI\ind{t,h}}}) \leq \epsilon. \label{eq: consequence of fobastate main thm rel 1 proof}
%    \end{align}
%\end{enumerate}
We will prove that for all $t\leq{}h$,
\begin{align}
 \PP\rbr{\GS\ind{t,h} \mid \cap_{t'=t+1}^{h} \GS\ind{t',h}, \GS\ind{<h}}\geq 1-\delta/H. \label{eq: what we need to show central lemma}
\end{align}
Taking a union bound (\pref{lem: sequentual good event}), this implies that $\PP\rbr{\cap_{t'=1}^{h} \GS\ind{t',h} \mid \GS\ind{<h}}\geq 1-\delta$, which establishes \pref{thm: success of backward state refinement}.

\paragraphp{Proving \eqref{eq: what we need to show central lemma}.} Let $t<h$ be fixed, and condition on $\cap_{t'=t+1}^{h} \GS\ind{t',h}$ and $\GS\ind{<h}$. We will show that whenever these events hold and the estimated occupancy measures have sufficiently high accuracy, $\GS\ind{t,h}$ holds. Formally, recalling \pref{def: approximate set of occupancy measures}, define an event
\begin{align}
  \cGstat\ind{t,h} = \crl*{\text{$\EmpOccupancy$ is $\frac{\epsnot}{12k}$-approximate with respect to $\rbr{\mu\ind{t}\circt\PiInd\brk{\fullc_{\leq k}} \circt[t+1]\Psi\ind{t+1,h}, \fullc_{\leq{}k}\rbr{\cI\ind{t+1,h}},h}$}}.\label{eq: what we need to show central lemma rel 2}
\end{align}
Our goal is to show that conditioned on $\cap_{t'=t+1}^{h} \GS\ind{t',h}$ and $\GS\ind{<h}$, $\cGstat\ind{t,h}\implies\cG\ind{t,h}$, so that
\begin{align*}
  \PP\rbr{\GS\ind{t,h} \mid \cap_{t'=t+1}^{h} \GS\ind{t',h}, \GS\ind{<h}} 
  \geq \PP\rbr{\cGstat\ind{t,h} \mid \cap_{t'=t+1}^{h}\GS\ind{t',h}, \GS\ind{<h}}
  \stackrel{\mathrm{(a)}}{\geq} 1-\delta.
\end{align*}
Here $\mathrm{(a)}$ is a consequence of \pref{lem: sample complexity of close model}, which asserts that by setting
\begin{align}
  N = \bigom\prn[\Big]{A S^{2k} k^3 \log\rbr{\frac{d SA}{\delta}}\cdot\epsnot^{-2}},  \label{eq:num_samples_sr_thm}
\end{align}
the estimated state occupancies $\EmpOccupancy$ produced in \pref{line: importance sampling estimation} of $\AlgName^{\eps,\delta}_h$ are  $\epsnot/12k$-approximate with respect to   $\rbr{\mu\ind{t}\circt\PiInd\brk{\fullc_{\leq k}} \circt[t+1]\Psi\ind{t+1,h}, \fullc_{\leq{}k}\rbr{\cI\ind{t+1,h}},h}$, in the sense of \pref{def: approximate set of occupancy measures}. We formally verify that the preconditions required to apply~\pref{lem: sample complexity of close model} are satisfied at the end of the proof for completeness.
% \dfcomment{For arxiv: Add a proposition/lemma statement asserting that the conditions hold.}

% which uses standard concentration inequalities and a uniform convergence argument.

% See that conditioning on $\cap_{t'=1}^{t-1} \GS_{t'}$, it holds that $(E1)$ $\Psi\ind{t+1,h}$ is an endogenous set of policies, $(E2)$ $\cI\ind{t+1,h}$ contains only endogenous factors, and $(E3)$ $\EmpOccupancy$ is $\epsnot$-approximate with respect to $\rbr{\PiInd\brk{\fullc_{\leq k}}, \fullc_{\leq{}k}\rbr{\cI\ind{t+1,h}},h}$
% We now prove that \eqref{eq: what we need to show central lemma rel 2} holds and conclude the proof. That is, conditioning on $\cap _{t'=t+1}^h \GS\ind{t',h}$ and $\cap_{h'=1}^{h-1}\GS\ind{h'}$ having access to $\EmpOccupancy$ is $\epsnot/12k$-approximate with respect to $\rbr{\PiInd\brk{\fullc_{\leq k}}, \fullc_{\leq{}k}\rbr{\cI\ind{t+1,h}},h}$ implies  $\GS\ind{t,h}$.

We now prove that conditioned on $\cap_{t'=t+1}^{h} \GS\ind{t',h}$ and $\GS\ind{<h}$, $\cGstat\ind{t,h}\implies\cG\ind{t,h}$. This relies on two claims: Success of $\EndoPolicyMaximizer$ and success of $\FactorDetectionAlg$.

% \begin{itemize}
\paragraphp{Success of $\EndoPolicyMaximizer^{\epsnot}_{t,h}$.} We appeal to \pref{lem: correctness of endo policy cover}, verifying that the assumptions it requires, $(\mathrm{A}1)$ and $(\mathrm{A}2)$, are satisfied (conditioned on $\cap_{t'=t+1}^{h} \GS\ind{t',h}$ and $\GS\ind{<h}$).
        \begin{enumerate}[label=$(\mathrm{A}\arabic*)$]
            \item $\mu\ind{t}$ is an endogenous policy when $\cG\ind{<h}$ holds (see \eqref{eq: mut is endo policy statement}) and $\Psi\ind{t+1,h}$ contains only endogenous policies whenever $\GS\ind{t+1,h}$ holds.
            % \item $\cI\ind{t+1,h}\subseteq \Ic$ when $\GS\ind{t+1,h}$ holds.
            \item $\EmpOccupancy$ is $\epsnot/12k$-approximate with respect to $\rbr{\PiInd\brk{\fullc_{\leq k}}, \fullc_{\leq{}k}\rbr{\cI\ind{t+1,h}},h}$ whenever $\cGstat\ind{t,h}$ holds.
        \end{enumerate}
        Thus,~\pref{lem: correctness of endo policy cover} implies that for all $\cI\in \fullc_{\leq k}\rbr{\cI\ind{t,h}}$ and   $s\brk*{\cI}\in \cS\brk*{\cI}$, the respective invocation of the sub-routine $\EndoPolicyMaximizer^{\epsnot}_{t,h}$  outputs a policy $\pi\ind{t}_{s\brk*{\cI}}\in \Gamma\ind{t}\brk*{\cI}$ that is (i) endogenous, and (ii) near-optimal in the following one-step sense:
        \begin{align}
          \label{eq:endo_optimality}
        \max_{\pi\in \PiInd\brk{\fullc_{\leq k}}} d_h\rbr{ s\brk{\cI} \midsem \mut\circt  \pi \circt[t+1] \psi\ind{t+1,h}_{s\brk{\cI^{t+1,h}}}} \leq d_h\rbr{ s\brk{\cI} \midsem \mut\circt  \pi\ind{t}_{s\brk{\cI}} \circt[t+1] \psi\ind{t+1,h}_{s\brk{\cI^{t+1,h}}}} +4\epsnot.
        \end{align}
        % where we recall that we can write $s\brk*{\cI} = \rbr{s\brk*{\cI\ind{t+1,h}} ,s\brk*{\cI\setminus{} \cI\ind{t+1,h}}}$.
    \paragraphp{Success of $\FactorDetectionAlg^{\epsnot}_{t,h}$.} We appeal to \pref{thm:  no false detection of controllable factors}, verifying that the assumptions $(\mathrm{A}1)$-$(\mathrm{A}3)$ required by it are satisfied.% and conclude that $\FactorDetectionAlg_{t,h}^\eps$ succeeds.
    \begin{enumerate}[label=$(\mathrm{A}\arabic*)$]
    \item $\mu\ind{t}$ is endogenous whenever $\cG\ind{<h}$ holds. Whenever $\GS\ind{t+1,h}$ holds, we are guaranteed that $\Psi\ind{t+1,h}$ contains only endogenous policies, so that $\psi\ind{t+1,h}_{s\brk{\cI^{t+1,h}}}\in \Psi\ind{t+1,h}$ is endogenous in particular.
    \item $\EmpOccupancy$ is $\epsnot/12k$-approximate with respect to $\rbr{\PiInd\brk{\fullc_{\leq k}}, \fullc_{\leq{}k}\rbr{\cI\ind{t+1,h}},h}$ by $\cGstat\ind{t,h}$.
    \item Due to the success of $\EndoPolicyMaximizer_{t,h}^{\epsnot}$ (verified above), the condition in \eqref{eq:endo_optimality}  is satisfied.
    \end{enumerate}
    % Assuming that  $(1)$ $\mu\ind{t}$ is an endogenous policy (see \eqref{eq: mut is endo policy statement}) when $\GS\ind{h}$ holds, $(2)$ $\Psi\ind{t+1,h}$ contains only endogenous policies when $\GS\ind{t+1,h}$ holds, and $(3)$ $\EmpOccupancy$ is $\epsnot/12k$-approximate with respect to $\rbr{\PiInd\brk{\fullc_{\leq k}}, \fullc_{\leq{}k}\rbr{\cI\ind{t+1,h}},h}$, the conditions of \pref{thm:  no false detection of controllable factors} are satisfied. The fact that $\cI\ind{t,h}\subseteq \Ic$ is a consequence of the first statement of \pref{thm:  no false detection of controllable factors}. The fact that $\GS_2^{t,h}$ holds is a consequence of the second statement of~\pref{thm:  no false detection of controllable factors}.
    Hence, by~\pref{thm:  no false detection of controllable factors}, $\FactorDetectionAlg_{t,h}^{\epsnot}$ returns a tuple $(\cI\ind{t,h},\Psi\ind{t,h}\brk{\cI\ind{t,h}})$ such that
    \begin{enumerate}
    \item $\cI\ind{t,h}\subseteq \Ic$.
    \item For all $s\in\cS$,
      % For any $s\sbr{\Ic}\in \Scal\sbr{\Ic}$, let 
    % \begin{align*}
    %     s\sbr{\Ic} = \rbr{s\sbr{\cI\ind{t,h}},s\sbr{\Ic \setminus{} \cI\ind{t,h}}} = \rbr{s\sbr{\cI\ind{t+1,h}},,s\sbr{\Ic \setminus{} \cI\ind{t+1,h}} },
    % \end{align*}
    % valid since $\cI\ind{t+1,h},\cI\ind{t,h}\subseteq \Ic$.
    \[
       \hspace{-0.7cm} \max_{\pi\in \PiInd\brk{\Ic}}d_h\rbr{s\sbr{\Ic} ; \mut\circt \pi\circt[t+1]\psi\ind{t+1,h}_{s\sbr{\cI\ind{t+1,h}}}} - d_h\rbr{s\sbr{\Ic} ;  \mut\circt \pi\ind{t}_{s\brk{\cI\ind{t,h}}}\circt[t+1]\psi\ind{t+1,h}_{s\sbr{\cI\ind{t+1,h}}}} \leq 16\epsnot,
    \]
    where we recall that $\psi\ind{t+1,h}_{s\sbr{\cI\ind{t+1,h}}}\in \Psi\ind{t+1,h}$ and $ \pi\ind{t}_{s\brk{\cI\ind{t,h}}}\in \Gamma\ind{t}\brk*{\cI\ind{t,h}}$.
    \end{enumerate}
    % \end{itemize}
    \paragraphp{Wrapping up.} Scaling $\epsnot\gets \epsnot/16$ and $\delta\gets \delta/H$, and recalling that  $\psi\ind{t,h}_{s\sbr{\cI\ind{t,h}}}\in \Psi\ind{t,h}$ is given by\loose
\begin{align*}
    \psi\ind{t,h}_{s\sbr{\cI\ind{t,h}}} \ldef \pi\ind{t}_{s\brk{\cI\ind{t,h}}}\circt[t+1]\psi\ind{t+1,h}_{s\sbr{\cI\ind{t+1,h}}},
\end{align*}
we have that for all $t<h$, $$\PP\rbr{\GS\ind{t,h} \mid \cap_{t'=t+1}^{h} \GS\ind{t',h}, \cap_{h'=1}^{h-1}\GS\ind{h'}}\geq 1-\delta/H,$$ proving the result. %\dfcomment{For arxiv: Be more precise about rescaling above lol.}
% given
% $
% N = \Theta\rbr{A S^{2k} k^3 \log\rbr{\frac{d SA}{\delta}}\cdot\epsnot^{-2}},
% $ proving~\eqref{eq: what we need to show central lemma}, and the theorem.
 
\paragraphp{Verifying conditions of \pref{lem: sample complexity of close model}.} 
We conclude by verifying that the four conditions required by~\pref{lem: sample complexity of close model} hold, conditioned on $\cap_{t'=t+1}^{h} \GS\ind{t',h}$ and $\GS\ind{<h}$; this justifies the application in the prequel.
% To verify them we simply bound the cardinality of $\fullc_{\leq k}(\cI\ind{t+1,h}), \Psi\ind{t+1,h},\Pi\brk*{\fullc_{\leq k}(\cI\ind{t+1,h})}$ and $\cS\brk*{\cI\ind{t+1,h}}$.
\begin{enumerate}
\item By construction, $\Psi\ind{t+1,h} = \crl[\big]{\psi\ind{t+1,h}_{s\brk{\cI\ind{t,h}}}\mid s\sbr{\cI\ind{t+1,h}}\in\Scal\sbr{\cI\ind{t+1,h}}}$. Thus, $\abr{\Psi\ind{t+1,h}} = \abr{\Scal\sbr{\cI\ind{t,h}}} \leq \S^k$, since %conditioning on $\GS\ind{t+1,h}$ implies that $\cI\ind{t+1,h}$ contains only endogenous factors,
  $\abr{\cI\ind{t+1,h}}\leq k$.
\item We have $\abr{\PiInd\brk{\fullc_{\leq k}}}\leq \bigoh\prn*{d^kA^{\S^k}}$, since the number of factor sets of size at most $k$ is
    \begin{align}
      \sum_{k'=0}^k{d \choose k'} \leq \prn*{\frac{ed}{k}}^{k} \leq O\rbr{d^k}, \label{eq: upper bond on sum d choose k}
    \end{align}
    % \dfcomment{The best bound I know here is $\sum_{k'=0}^k{d \choose k'}\leq{}(ed/k)^k$, which is technically $d^{O(k)}$, not $O(d^k)$}.
   and for any factor set $\cI$ with $\abr{\cI}\leq k$ we have $\abr{\Pi\brk{\cI}}\leq A^{\S^k}$.
\item $\abr{\fullc_{\leq{}k}\rbr{\cI\ind{t+1,h}}} \leq \abr{\fullc_{\leq{}k}}\leq O\rbr{d^k}$ by \eqref{eq: upper bond on sum d choose k}, 
\item For any fixed set $\cI$ with $\abs{\cI}\leq{}k$, we have $\abr{\Scal\sbr{\cI}}\leq \S^k$.
\end{enumerate}
\end{proof}

% \subsection{Application of $\EndoPolicyMaximizer$ in $\AlgName^{\eps,\delta}_{h}$}
\subsection{Application of $\EndoPolicyMaximizer$ in $\AlgName$}\label{app: application of endoPolicyOptimizer to OSSR}

% \dfcomment{This should probably get moved into the OSSR section of the appendix.}

The main guarantee for the $\EndoPolicyMaximizer^\eps_{t,h}$ subroutine (\pref{thm: correctness of endo policy maximizer})
implies that the policy $\pi\ind{t}_{s\brk{\cI}}$ returned in \pref{line: OSSR policy optimization phase} of $\AlgName$ is endogenous, as well as near-optimal in the following this sense:
\begin{align*}
    \max_{\pi\in \PiInd\brk{\fullc_{\leq k}}} d_h\rbr{ s\brk{\cI} \midsem \mut\circt  \pi \circt[t+1] \psi\ind{t+1,h}_{s\brk{\cI^{t+1,h}}}} \leq d_h\rbr{ s\brk{\cI} \midsem \mut\circt  \pi\ind{t}_{s\brk{\cI}} \circt[t+1] \psi\ind{t+1,h}_{s\brk{\cI^{t+1,h}}}} + O(\eps).
\end{align*}
In this subsection we state and prove \pref{lem: correctness of endo policy cover}, which shows that the preconditions $(\mathrm{A}1)$ and $(\mathrm{A}2)$ required to apply \pref{thm: correctness of endo policy maximizer} are satisfied, so that the claim above indeed holds.

\begin{lemma}\label{lem: correctness of endo policy cover}
Fix $h\in [H]$ and $t\leq{}h$. Suppose that the following conditions hold:
\begin{enumerate}[label=$(\mathrm{C}\arabic*)$]
    \item $\mu\ind{t}\in \Pimix\brk{\Ic}$ is endogenous and $\Psi\ind{t+1,h}$ contains only endogenous policies.
    % \item $\cI\ind{t+1}\subseteq \Ic.$
    \item The collection $\EmpOccupancy$ of occupancy measures is $\epsilon/12k$-approximate with respect to\\ $\rbr{\mu\ind{t}\circ\Pi\brk{\fullc_{\leq{}k}} \circ \Psi\ind{t+1,h}, \fullc_{\leq{}k}\rbr{\cI\ind{t+1,h}},h}.$
    \end{enumerate}
      Then assumptions $(\mathrm{A}1)$ and $(\mathrm{A}2)$ of \pref{thm: correctness of endo policy maximizer} are satisfied when \epm is invoked within $\AlgName$, and for all $\cI\in \fullc_{\leq{}k}\rbr{\cI\ind{t+1,h}}$:
\begin{enumerate}
    \item The set $\Gamma\ind{t}\sbr{\cI} = \crl*{\pi\ind{t}_{s\brk{\cI}} \mid s\brk{\cI} \in \cS\brk{\cI}}$ contains only endogenous policies.
    \item For all $s\sbr{\cI}\in \cS\brk*{\cI}$, the policy $\pi\ind{t}_{s\sbr{\cI}}\in \Gamma\sbr{\Ical}$ satisfies
    \begin{align*}
      % \hspace{-0.8cm}
      &\max_{\pi\in \PiInduced\brk{\fullc_{\leq{}k}}} d_h\rbr{s\sbr{\cI} \midsem \mu\ind{t}\circt \pi \circt[t+1] \psi\ind{t+1,h}_{s\sbr{\cI\ind{t+1,h}}}} \\
      &\leq d_h\rbr{s\sbr{\cI} \midsem \mu\ind{t}\circt \pi\ind{t}_{s\sbr{\cI}} \circt[t+1] \psi\ind{t+1,h}_{s\sbr{\cI\ind{t+1,h}}}} + 4\epsilon.
    \end{align*}
  \end{enumerate}
\end{lemma}
\begin{proof}[\pfref{lem: correctness of endo policy cover}]%
  % We show that the two assumptions of~\pref{thm: correctness of endo policy maximizer}, that are needed to establish the correctness of $\EndoPolicyMaximizer$, hold. Then, the result follows by applying~\pref{thm: correctness of endo policy maximizer}.
Toward proving the result, we begin with a basic observation. Fix $\cI\in \fullc_{\leq{}k}\rbr{\cI\ind{t+1,h}}$ and $s\sbr{\cI}\in \cS\sbr{\cI}$. Define an MDP $\rbr{\cS,\Acal, T,R_{s\brk{\cI}},h}$ where $R_{s\brk{\cI},h} = \indic\cbr{s_h\brk{\cI} = s\brk{\cI}}$ and $R_{s\brk{\cI},h'}=0$ for all $h'\neq h$. Observe that the occupancy measure for $s\brk{\cI}$ at layer $h$ is equivalent to the $(t,h)$ value function in this MDP:
  % \dfcomment{We should discuss this superscript notation because it's pretty nasty.}
\begin{align}
    V_{t,h}\rbr{\mut\circt \pi \circt[t+1] \psi\ind{t+1,h}_{s\sbr{\cI\ind{t+1,h}}}} = d_h\rbr{s\sbr{\cI} \midsem \mut\circt \pi \circt[t+1] \psi\ind{t+1,h}_{s\sbr{\cI\ind{t+1,h}}}}. \label{eq: reduction value state frequency}
\end{align}
We now show that assumptions $(\mathrm{A}1)$ and $(\mathrm{A}2)$ of \pref{thm: correctness of endo policy maximizer} hold when the theorem is invoked with this value function, from which the result will follow.
% hold when the th
% , from which the result follows by applying~\pref{thm: correctness of endo policy maximizer}. 
% We will invoke \pref{thm: correctness of endo policy maximizer} with this value function.
% See that the value function is defined over the MDP $\rbr{\cS,\Acal, T,r_{s\brk{\cI}},h}$ where $r_{h,s\brk{\cI}} = \indic\cbr{s_h\brk{\cI} = s\brk{\cI}}$ and zero for all $h'\neq 0.$

\paragraphp{Verifying assumption $(\mathrm{A}1)$ of~\pref{thm: correctness of endo policy maximizer}.} The policies $\mut$ and $\psi\ind{t+1,h}_{s\sbr{\cI\ind{t+1,h}}} \in \Psi\ind{t+1,h}$ are endogenous by condition $(\mathrm{C}1)$. Hence, the assumptions of the restriction lemma (\pref{lem: invariance of reduced policy cover}) are satisfied, which gives
% This implies that,   
% Thus,~\pref{lem: endognous set is the maximizer} is --------, by which it holds that  
% \begin{align*}
%   &\max_{\pi\in \PiInduced\sbr{\fullc_{\leq{}k}\rbr{\cI\ind{t+1,h}}}} d_h\rbr{s\sbr{\cI} \midsem \mut\circt \pi \circt[t+1] \psi\ind{t+1,h}_{s\sbr{\cI\ind{t+1,h}}}} =  \max_{\pi\in \PiInd\sbr{\Ic}} d_h\rbr{s\sbr{\cI} \midsem \mut\circt \pi \circt[t+1] \psi\ind{t+1,h}_{s\sbr{\cI\ind{t+1,h}}}}\\
%   \iff &\max_{\pi\in \PiInduced\sbr{\fullc_{\leq{}k}\rbr{\cI\ind{t+1,h}}}} V_{t,h}\rbr{\mut\circt \pi \circt[t+1] \psi\ind{t+1,h}_{s\sbr{\cI\ind{t+1,h}}}} =  \max_{\pi\in \PiInd\sbr{\Ic}} V_{t,h}\rbr{\mut\circt \pi \circt[t+1] \psi\ind{t+1,h}_{s\sbr{\cI\ind{t+1,h}}}} \tag{By definition}
% \end{align*}
% Thus, condition $(\mathrm{A}1)$ of~\pref{lem: endognous set is the maximizer} is satisfied. Furthermore, $(1),(2)$ and $(3)$ also implies that $(\mathrm{A}2)$ of~\pref{thm: correctness of endo policy maximizer} holds, i.e., by~\pref{lem: invariance of reduced policy cover} it holds that 
\begin{align*}
   &\max_{\pi\in \PiInd\sbr{\cI}} d_h\rbr{s\sbr{\cI} \midsem \mut\circt \pi \circt[t+1] \psi\ind{t+1,h}_{s\sbr{\cI\ind{t+1,h}}}} =  \max_{\pi\in \PiInd\sbr{\Iendo{\cI}}} d_h\rbr{s\sbr{\cI} \midsem \mut\circt \pi \circt[t+1] \psi\ind{t+1,h}_{s\sbr{\cI\ind{t+1,h}}}}\\
   &\iff \max_{\pi\in \PiInd\sbr{\cI}} V_{t,h}\rbr{\mut\circt \pi \circt[t+1] \psi\ind{t+1,h}_{s\sbr{\cI\ind{t+1,h}}}} =  \max_{\pi\in \PiInd\sbr{\Iendo{\cI}}} V_{t,h}\rbr{\mut\circt \pi \circt[t+1] \psi\ind{t+1,h}_{s\sbr{\cI\ind{t+1,h}}}}.
     % \tag{By definition; cf. \eqref{eq: reduction value state frequency}}
\end{align*}

% \paragraphp{Verifying $(\mathrm{A}2)$ of~\pref{thm: correctness of endo policy maximizer}.} $(\mathrm{A}2)$ of~\pref{thm: correctness of endo policy maximizer} holds by $(\mathrm{A}2)$.

\paragraphp{Verifying assumption $(\mathrm{A}2)$ of~\pref{thm: correctness of endo policy maximizer}.} By condition $(\mathrm{C}2)$, we have that $\EmpOccupancy$ is $\epsilon/12k$-approximate with respect to $\rbr{\mu\ind{t}\circ\PiInd\brk{\fullc_{\leq k}} \circ \Psi\ind{t+1,h}, \fullc_{\leq{}k}\rbr{\cI\ind{t+1,h}},h}$, and hence
\begin{align*}
    &\abr{\widehat{d}_{h}\rbr{ s\sbr{\cI} \midsem \mut\circt \pi\circt[t+1] \psi\ind{t+1,h}_{s\sbr{\cI\ind{t+1,h}}}} - d_{h}\rbr{s\sbr{\cI} \midsem \mut\circt \pi\circt[t+1] \psi\ind{t+1,h}_{s\sbr{\cI\ind{t+1,h}}}}} \leq \eps/12k\\
  &\iff  \abr{\widehat{V}_{t,h}\rbr{\mut\circt \pi \circt[t+1] \psi\ind{t+1,h}_{s\sbr{\cI\ind{t+1,h}}}} - V_{t,h}\rbr{\mut\circt \pi \circt[t+1] \psi\ind{t+1,h}_{s\sbr{\cI\ind{t+1,h}}}}} \leq \eps/12k.
\end{align*}
% where $\mathrm{(a)}$ holds by 

\end{proof}

%%% Local Variables:
%%% mode: latex
%%% TeX-master: "paper"
%%% End:

\newpage

% \section{RL in the Presence of Exogenous Information: \RLwithExo}
\section{Proof of \preftitle{thm:main} (Correctness of \RLwithExo)}
\label{app: RL in the presence of exogenous information}

In this section we formally prove \pref{thm:main}, which shows that $\RLwithExo$ (\pref{alg: main}) learns an $\eps$-optimal policy for a general \kEMDP. 
% \colt{
% \dfcomment{Update discussion.}
  % \paragraph{Overview of analysis.}
  The correctness of \mainalg is essentially a direct corollary of the results derived for $\AlgName$ and $\PSDP$ in~\pref{app: learnin eps endogenous policy cover} and~\pref{app: psdp in the presence of exogenous information}. The high probability guarantee for $\AlgName$ (\pref{thm: sample complexity of state refinment}) implies that iteratively applying $\AlgName^{\eta/2,\delta}_h$ results in an endogenous $\eta/2$-approximate policy covers for every layer $h\in [H]$. Conditioning on this event, $\PSDPE$ is guaranteed to find an $\eps$-optimal policy with high probability (\pref{thm:ex_PSDP_sample_comp}).

\rlwithexo*

\begin{proof}[\pfref{thm:main}] We first show that $\AlgName$ results in a near-optimal (endogenous) policy cover, then show that the application of \PSDPE is successful.
  \paragraphp{Application of \pcalg.} Let $\GS\ind{h}$ denote the event in which $\AlgName_h^{\eta/2,\delta}(\crl{\Psi\ind{t}}_{t=1}^{h-1})$ returns an endogenous $\eta/2$-approximate policy cover $\Psi\ind{h}$ with $\abs{\Psi\ind{h}}\leq{}S^{k}$, and let $\cG\ind{<h}\ldef{}\cap_{h'=1}^{h-1}\GS\ind{h}$. \pref{thm: sample complexity of state refinment} states that for all $h\geq{}2$, if we condition on $\cG\ind{<h}$, then given
$
    N = \bigoh\rbr{\frac{A S^{4k} H^2 k^3 \log\rbr{\frac{dSAH}{\delta}}}{\eta^2}}
$
samples, $\AlgName_h^{\eta/2,\delta}$ ensures that $\cG\ind{h}$ holds probability at least $1-\delta$ . Furthermore, $\cG\ind{1}$ holds trivially for $h=1$. By \pref{lem: sequentual good event}, this implies that $\PP\rbr{\cap_{h=1}^{H} \GS\ind{h}}\geq 1-H\delta $. Scaling $\delta\gets \delta/2H$, we conclude that given 
\[
    N_{\AlgName} = \bigoh\rbr{\frac{A S^{4k} H^2 k^3 \log\rbr{\frac{dSAH}{\delta}}}{\eta^2}}
\]
samples across all applications of $\AlgName_h^{\eta/2,\delta}$, the collection $\cbr{\Psi\ind{h}}_{h=1}^H$ is a set of endogenous $\eta/2$-approximate policy covers with probability at least $1-\delta/2$. We denote this event by $\GS_{\AlgName}$, so that $\PP\rbr{\GS_{\AlgName}}\geq 1-\delta/2$.
% \arxiv{\dfcomment{For arxiv: Rescaling of $\delta$ should be reflected in the algo box.}}

\paragraphp{Application of \PSDP.} Conditioned on the event $\GS_{\AlgName}$, the conditions of~\pref{thm:ex_PSDP_sample_comp} hold, so that the application of $\PSDPE$ is admissible. As a result, given
\[
  N_{\PSDPE}% = \bigoh\rbr{\frac{A S^{k} H^2 k^3 \log\rbr{\frac{dSAH}{\delta}}}{\epsnot^2}}
  = \bigoh\rbr{\frac{A S^{3k} H^4 k^3 \log\rbr{\frac{dSAH}{\delta}}}{\epsilon^2}}
\]
samples, $\PSDPE$ finds an endogenous $\eps$-optimal policy. We denote this event by $\GS_{\PSDPE}$, so that $\PP\rbr{\GS_{\PSDPE} \mid \GS_{\AlgName} } \geq 1-\delta/2$.% by~\pref{thm:ex_PSDP_sample_comp}.

\paragraphp{Concluding the proof.} $\RLwithExo$ returns an endogenous $\eps$-optimal policy when $\GS_{\AlgName}$ and $\GS_{\PSDPE}$ hold, and by the union bound
$\PP\rbr{\GS_{\AlgName}\cap{}\GS_{\PSDPE}} \geq 1-\delta$.
% since $\PP\rbr{\GS_{\PSDPE}}\geq 1-\delta$ and $\PP\rbr{\GS_{\PSDPE} \mid \GS_{\AlgName}}\geq 1-\delta$.
The total number of samples is %by $\RLwithExo$ is 
\[
  N = N_{\AlgName} + N_{\PSDPE} \leq \bigoh\rbr{\frac{A S^{4k}H^2 k^3 \log\rbr{\frac{dSAH}{\delta}}}{\eta^2} + \frac{A S^{3k}H^4 k^3 \log\rbr{\frac{dSAH}{\delta}}}{\epsilon^2}}.
\]
\end{proof}

\subsection{Computational Complexity of $\RLwithExo$}
The \RLwithExo procedure can be implemented with $O(d^kNS^k AH)$ runtime. In \pref{sec:comp_comp_pspde}, we show that \PSDPE can be implemented in runtime $O(d^kNS^k AH)$. Similarly, $\AlgName_h^{\eps,\delta}$ can be implemented with runtime $O(d^kNS^k A)$. The most computationally demanding aspect of \AlgName is optimizing the function $ \widehat{V}_{t,H}\rbr{\mut\circ_{t}\pi\circ_{t+1}\widehat{\pi}_{t+1:H}}$ over the policy class $\Pi\brk{\fullc_{\leq k}}$. As shown in \pref{sec:comp_comp_pspde}, this procedure can be implemented with runtime $O(d^kNS^k A)$, which is repeated for $H$ times in \RLwithExo.

%%% Local Variables:
%%% mode: latex
%%% TeX-master: "paper"
%%% End:

\end{document}